%% file: Parameter_Selection_VSNR.tex
\documentclass[final]{article}

\usepackage{amsthm,amstext,latexsym,amssymb,mathrsfs}
\usepackage{cite}
\usepackage[cmex10]{amsmath}

\usepackage{multicol}
\usepackage{multirow}

\usepackage{dsfont}

\usepackage{slashbox}

\usepackage{color}
\usepackage{array}
\usepackage{mdwmath}
\usepackage{mdwtab}
\usepackage{url}
\usepackage[ruled,vlined]{algorithm2e}

\usepackage{float}
\usepackage{graphicx}

\newtheorem{theorem}{Theorem}

\newtheorem{example}{Example}
\newtheorem{proposition}{Proposition}
\newtheorem{lemma}{Lemma}

\def\iinfty{{\boldsymbol{\infty}}}
\def\lambdabar{\overline{\lambda}}

\newcommand{\Ran}{\mathrm{Ran}}
\newcommand{\Ker}{\mathrm{Ker}}
\newcommand{\diag}{\mathrm{diag}}

\newcommand{\R}{\mathbb{R}}
\newcommand{\one}{\mathds{1}}

\newcommand{\C}{\mathbb{C}}
\newcommand{\Z}{\mathbb{Z}}

\newcommand{\E}{\mathbb{E}}

\newcommand{\argmin}{\mathop{\mathrm{arg\,min}}}
\newcommand{\argmmin}[2]{\mathop{\begin{array}[t]{lr} \argmin  &  #1  \\
                                        \begin{array}[t]{l} #2 \end{array}\end{array}}}
\newcommand{\argmax}{\mathop{\mathrm{arg\,max}}}
\newcommand{\Argmin}{\mathop{\mathrm{Arg\,min}}}

\def\pp{{\bf p}}
\def\PPsi{{\boldsymbol{\Psi}}}
\def\llambda{{\boldsymbol{\lambda}}}
\def\qq{{\bf q}}
\def\hh{{\bf h}}
\def\pp{{\bf p}}

\def\FF{{\cal F}}
\def\xx{{\bf x}}

\def\yy{{\bf y}}

\title{ \textbf{{\sc Processing stationary noise: model and parameter selection in variational methods.}}}

\author{J\'er\^ome Fehrenbach \thanks{IMT-UMR5219, Universit\'e de Toulouse, CNRS, Toulouse, France ({\tt jerome.fehrenbach@math.univ-toulouse.fr})} \and Pierre Weiss \thanks{ITAV-USR3505, Universit\'e de Toulouse, CNRS, Toulouse, France ({\tt pierre.weiss@itav-recherche.fr})}}

\begin{document}
\maketitle

\begin{abstract}
Additive or multiplicative stationary noise recently became an important issue in applied fields such as microscopy or satellite imaging. Relatively few works address the design of dedicated denoising methods compared to the usual white noise setting. We recently proposed a variational algorithm to tackle this issue. In this paper, we analyze this problem from a statistical point of view and provide deterministic properties of the solutions of the associated variational problems. 
In the first part of this work, we demonstrate that in many practical problems, the noise can be assimilated to a colored Gaussian noise. We provide a quantitative measure of the distance between a stationary process and the corresponding Gaussian process.
In the second part, we focus on the Gaussian setting and analyze denoising methods which consist of minimizing the sum of a total variation term and an $l^2$ data fidelity term. While the constrained formulation of this problem allows to easily tune the parameters, the Lagrangian formulation can be solved more efficiently since the problem is strongly convex. Our second contribution consists in providing analytical values of the regularization parameter in order to approximately satisfy Morozov's discrepancy principle.
\end{abstract}


\graphicspath{{Images/pdf/}}
\DeclareGraphicsExtensions{.pdf}

\pagestyle{myheadings}
\thispagestyle{plain}
\markboth{Parameter and model selection}{for the removal of stationary noise.}

\input{Sec_Introduction}
\input{Sec_Notation}
\input{Sec_DescriptionOfVSNR}
\input{Sec_StationaryProcesses}

\input{Sec_ParameterEstimation}

\input{Sec_Appendix}

\section*{Conclusion} 
\label{sec:conclusion}

This paper focussed on the problem of stationary noise removal using variational methods. In the first part, we showed that assuming the noise to be Gaussian is reasonable under conditions that are met in many applications such a destriping. In the second part we thus concentrated on variational problems that consist of minimizing $l^1-l^2$ functionals. We derived upper and lower bounds on the $l^2$-norm of the solutions of these functionals and showed that they can be used for simplifying the task of parameter selection. We also provided a numerical trick that allows to drastically reduce the computing times for cases where the noise is described as a sum of stationary processes. Overall this work allows to strongly reduce the computing times, to ease the parameter selection and to make our algorithms robust to different conditions.

As a perspective, let us notice that the lower bound proposed in proposition \eqref{eq:propminor} is coarse and it would be interesting to obtain tighter results highlighting why the upper bound is near tight in practice. We also plan to study the problem of deterministic parameter selection in a more general setting such as $l^p-l^q$ functionals.

\section*{Acknowledgments} 
This work was partially supported by ANR SPH-IM-3D (ANR-12-BSV5-0008).

\end{document}

%% file: Sec_Introduction.tex
\section{Introduction}

In a recent paper \cite{vsnr}, a variational method that decomposes an image into the sum of a piecewise smooth component and a set of stationary processes was proposed. This algorithm has a large number of applications such as deconvolution or denoising when \textit{structured patterns} degrade the image contents. A typical example of application that received a considerable attention lately is destriping \cite{munch2009stripe,leischner2010formalin,carfantan2010statistical,shu2011destripe,vsnr}. It was also shown to generalize the negative norm models \cite{Meyer,Vese,Osher,Aujol} in the discrete setting \cite{vsnr2}. Figures  \ref{fig:blonde}, \ref{fig:SPIM}, \ref{fig:comparisons} show examples of applications of this algorithm in an additive noise setting and Figure \ref{fig:lena} shows an example with a multiplicative noise model.

This algorithm is based on the hypothesis that the observed image $u_0$ can be written as:
\begin{equation}
 u_0 = u + \sum_{i=1}^m b_i
\end{equation}
where $u$ denotes the original image and $(b_i)_{i\in \{1,\cdots,m\}}$ denotes a set of realizations of independent stochastic processes $B_i$. 
These processes are further assumed to be stationary and read $B_i=\psi_i\ast \Lambda_i$ where $\psi_i$ denotes a known kernel and $\Lambda_i$ are i.i.d. random vectors. The decomposition algorithm can then be deduced from a Bayesian approach, leading to large scale convex optimization problems of size $m\times n$ where $n$ is the number of pixels/voxels in the image.

This method is now used routinely in the context of microscopy imaging. Its main weakness for a broader use lies in the difficulty to set its parameters adequately. One basically needs to input the filters $\psi_i$ and the marginals of each random vectors $\Lambda_i$, which is uneasy even for imaging specialists. Our aim in this paper is to provide a set of mathematically founded rules to simplify the parameter selection and reduce computing times. We do not tackle the problem of finding the filters $\psi_i$ (which is a problem similar to blind deconvolution), but focus on the choice of the marginals of $\Lambda_i$. 

The outline of the paper is as follows.
Notation are described in section \ref{sec:notation}.
In section \ref{sec:decompositionalgorithm}, we review the main principles motivating the decomposition algorithm.
In section \ref{sec:marginalselection}, we show that - from a statistical point of view and for many applications - assuming that $\lambda_i$ is a Gaussian process is nearly equivalent to selecting other marginals. This has the double advantage of simplifying the analysis of the model properties and reducing the computational complexity.
In section \ref{sec:parameterselection}, we show that when $b_i$ are drawn from Gaussian processes, parameter selection can be performed in a \textit{deterministic} way, by analyzing the primal-dual optimality conditions. We also show that the proposed ideas allows to reduce the problem dimension from $m\times n$ to $n$ variables, thus dividing the storage cost and computing times by a factor roughly equal to $m$. 
The appendix \ref{sec:appendix} contains the proofs of the results stated in section \ref{sec:parameterselection}. 
%Moreover it allows to avoid the need for scaling the different linear transforms appearing in the convex problem. This makes the algorithm robust to the filters choice $\psi_i$. 
%Overall, this work allows to highly simplify the parameter selection and to reduce the computing times.

\begin{figure}[htb]
\begin{minipage}[b]{1.0\linewidth}
  \centering
  \centerline{
    \includegraphics[width=3cm]{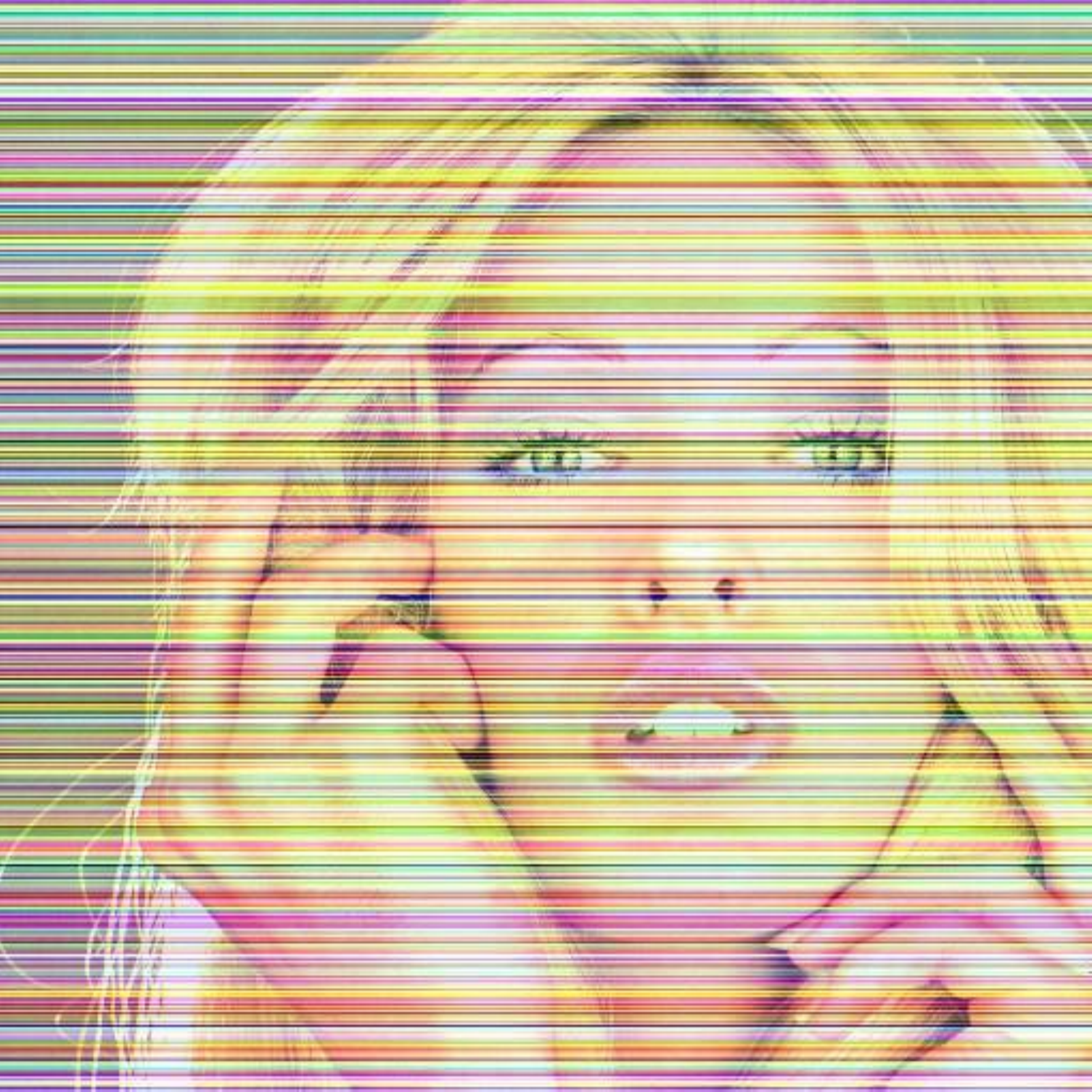}
    \includegraphics[width=3cm]{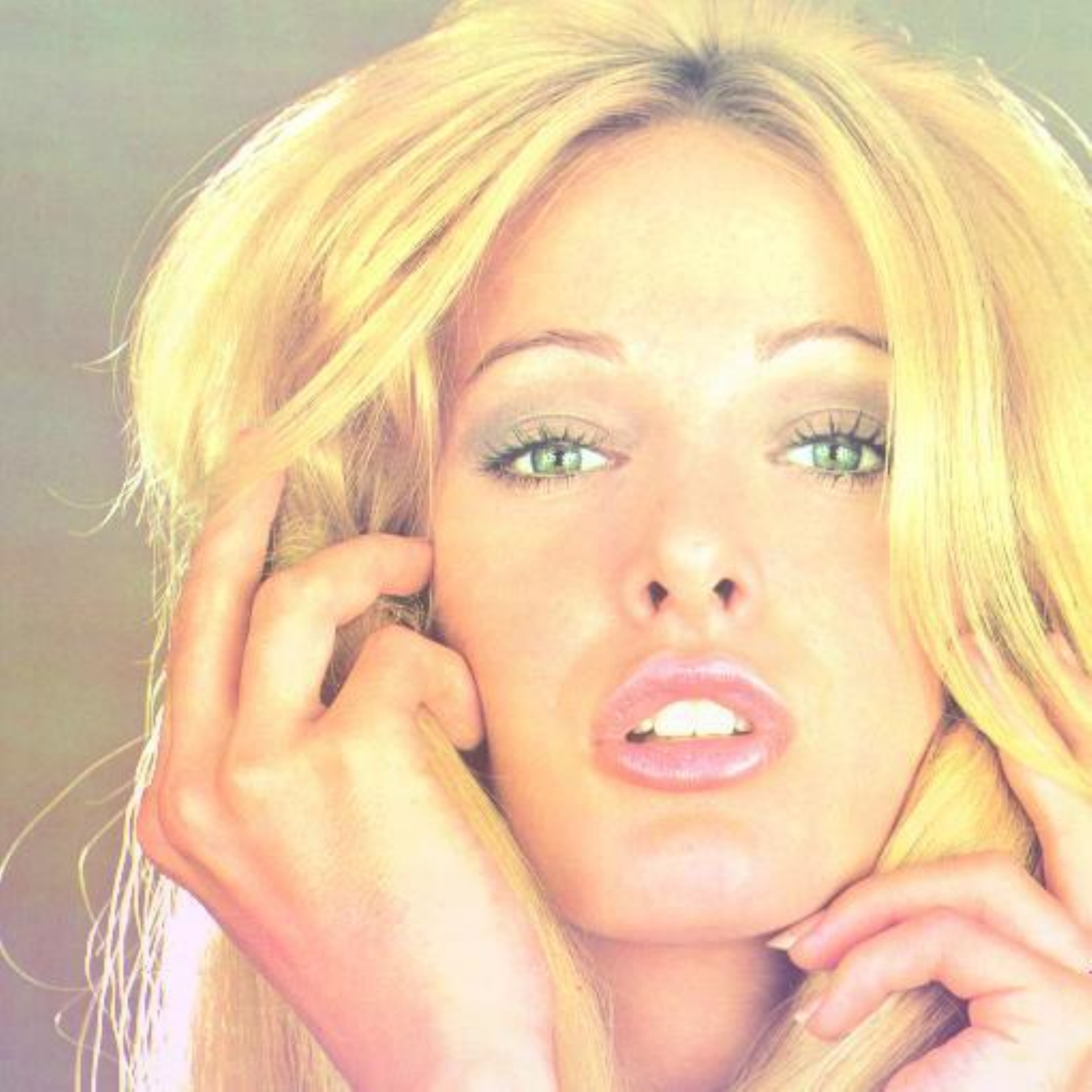}
    \includegraphics[width=3cm]{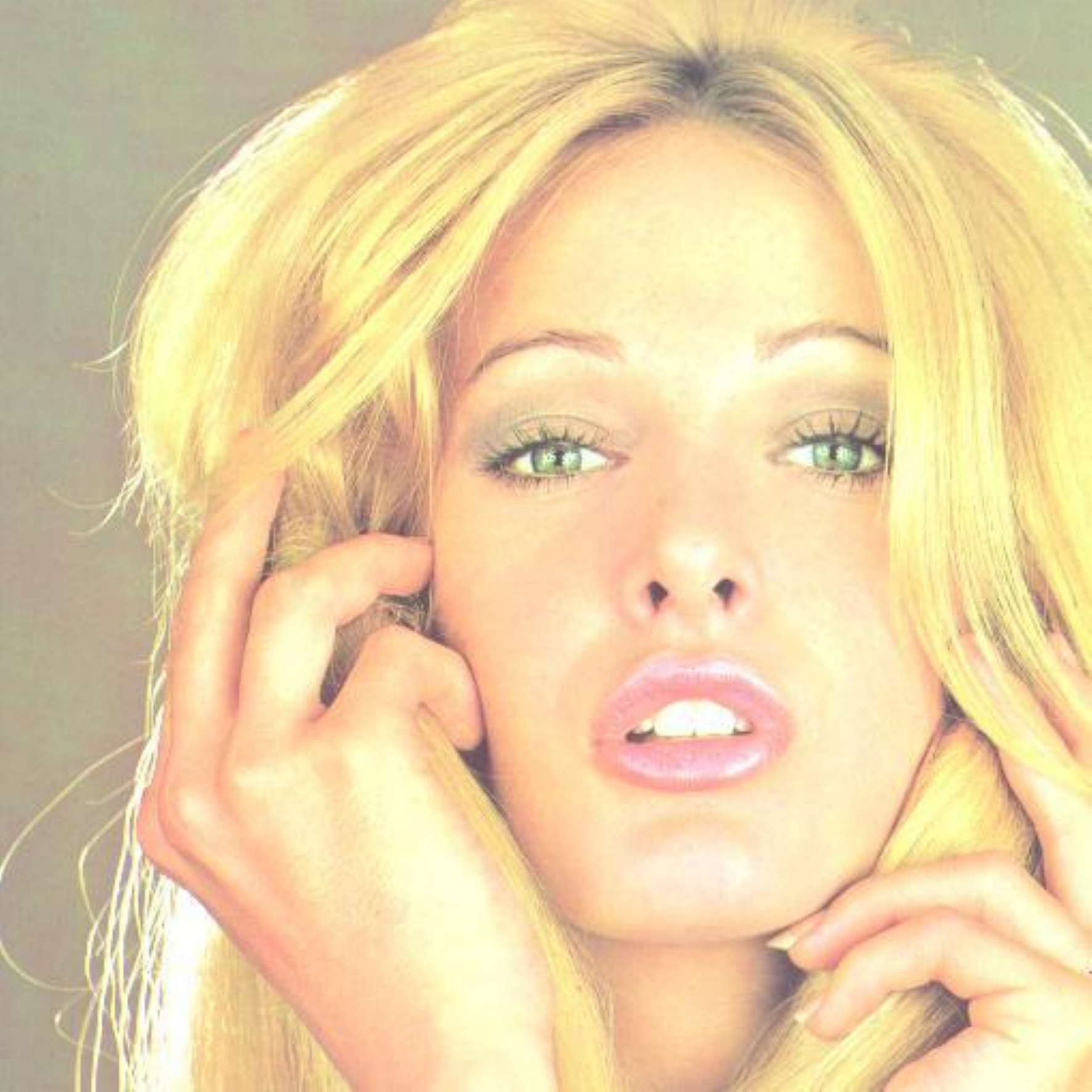}}
  \centerline{
    \includegraphics[width=3cm]{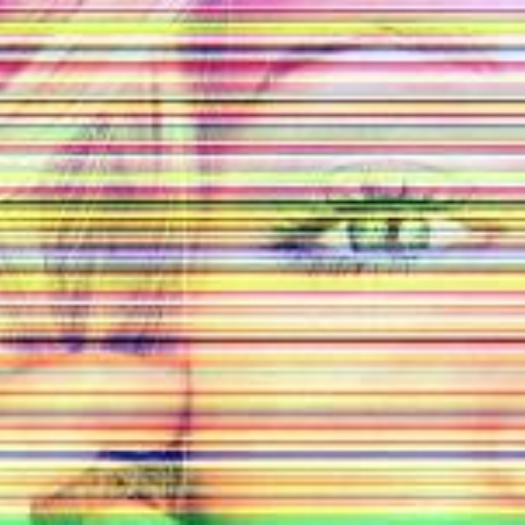}
    \includegraphics[width=3cm]{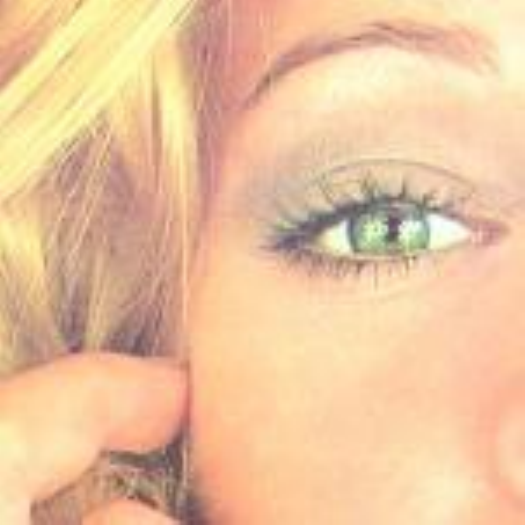}
    \includegraphics[width=3cm]{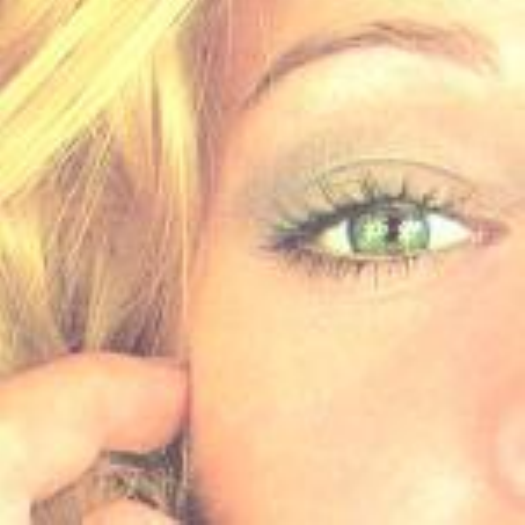}}
\end{minipage}
\caption{Top: full size images - Bottom: zoom on a small part. From left to right: Noisy image (16,5dB), denoised using the method proposed in \cite{vsnr} (PSNR=32,3dB), original image.}
\label{fig:blonde}
\end{figure}

\begin{figure}[htb]
\begin{minipage}[b]{1.0\linewidth}
  \centering
  \centerline{
    \includegraphics[width=5cm]{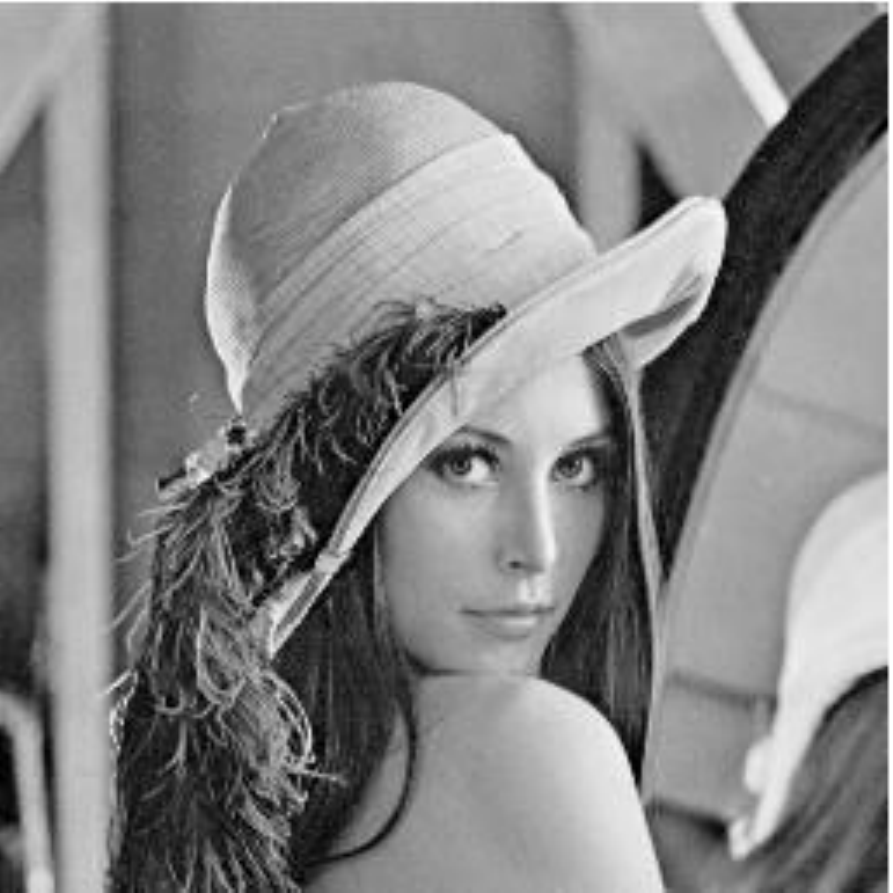} \ 
    \includegraphics[width=5cm]{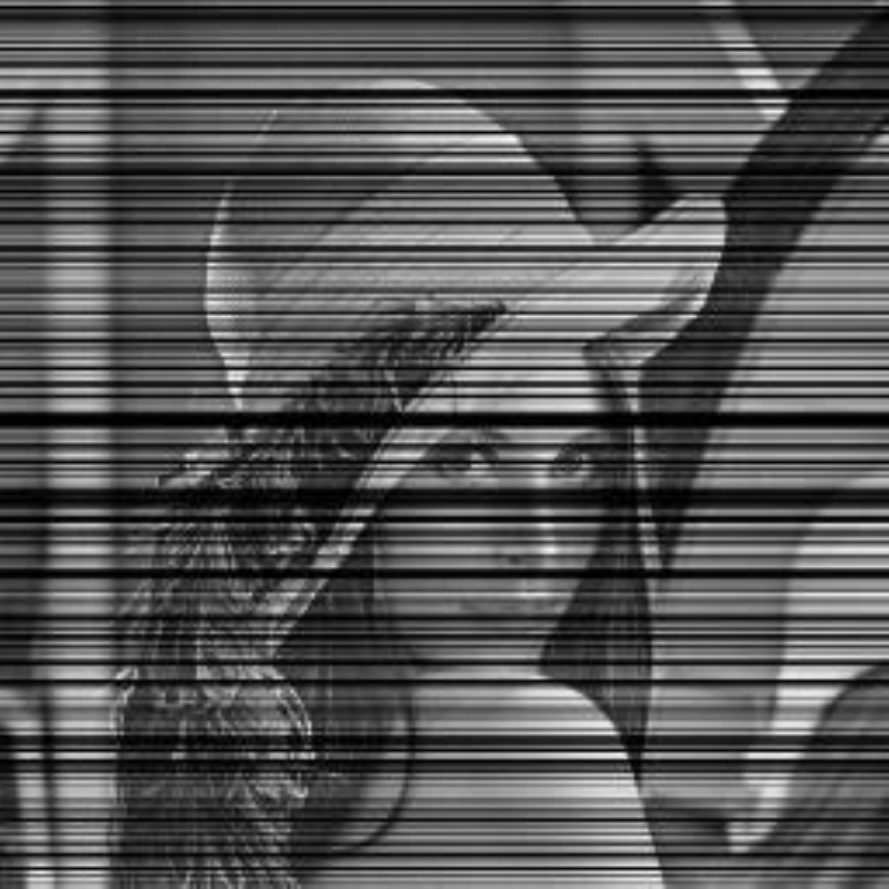}
  }
  \centering{
    \includegraphics[width=5cm]{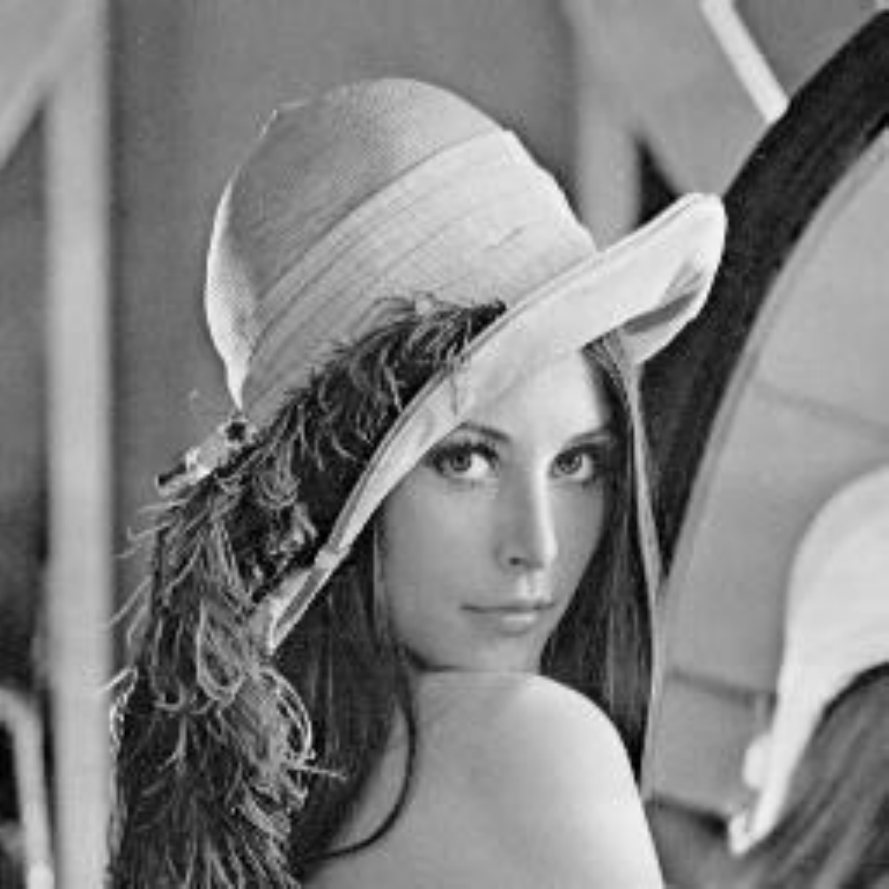} \
    \includegraphics[height=5cm]{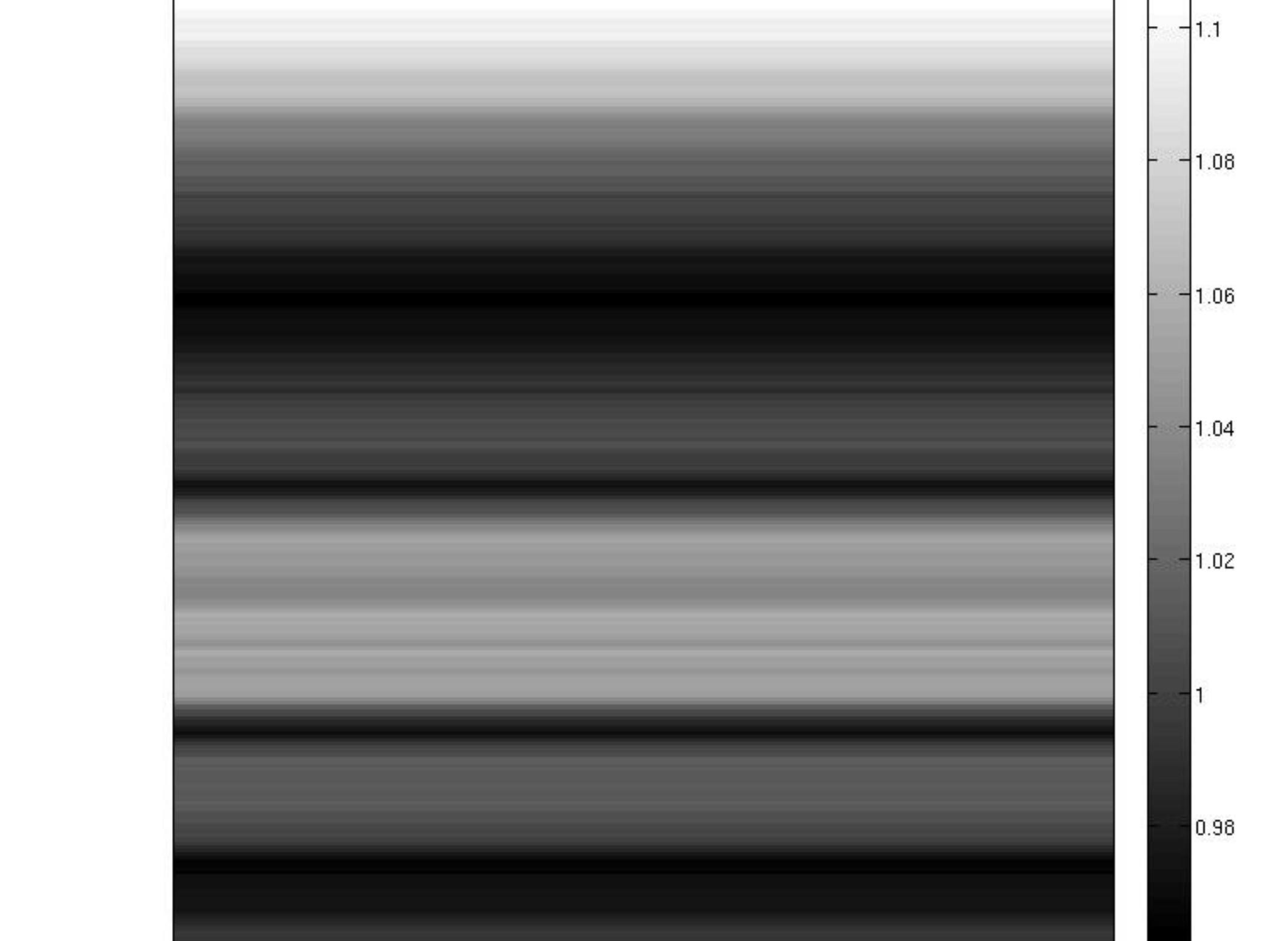}
  }
\end{minipage}
\caption{An example involving a multiplicative noise model. From left to right. Original image - Noisy image. It is obtained by multiplying each line of the original image by a uniform random variable in $[0.1,1]$. SNR=10.6dB - Denoised using the method proposed in \cite{vsnr} on the logarithm of the noisy image. SNR=29.1dB - Ratio between the original image and the denoised image. The multiplicative factor is retrieved accurately.}
\label{fig:lena}
\end{figure}

\begin{figure}[htb]
\begin{minipage}[b]{1.0\linewidth}
  \centering
  \centerline{
    \includegraphics[width=5cm,angle=90]{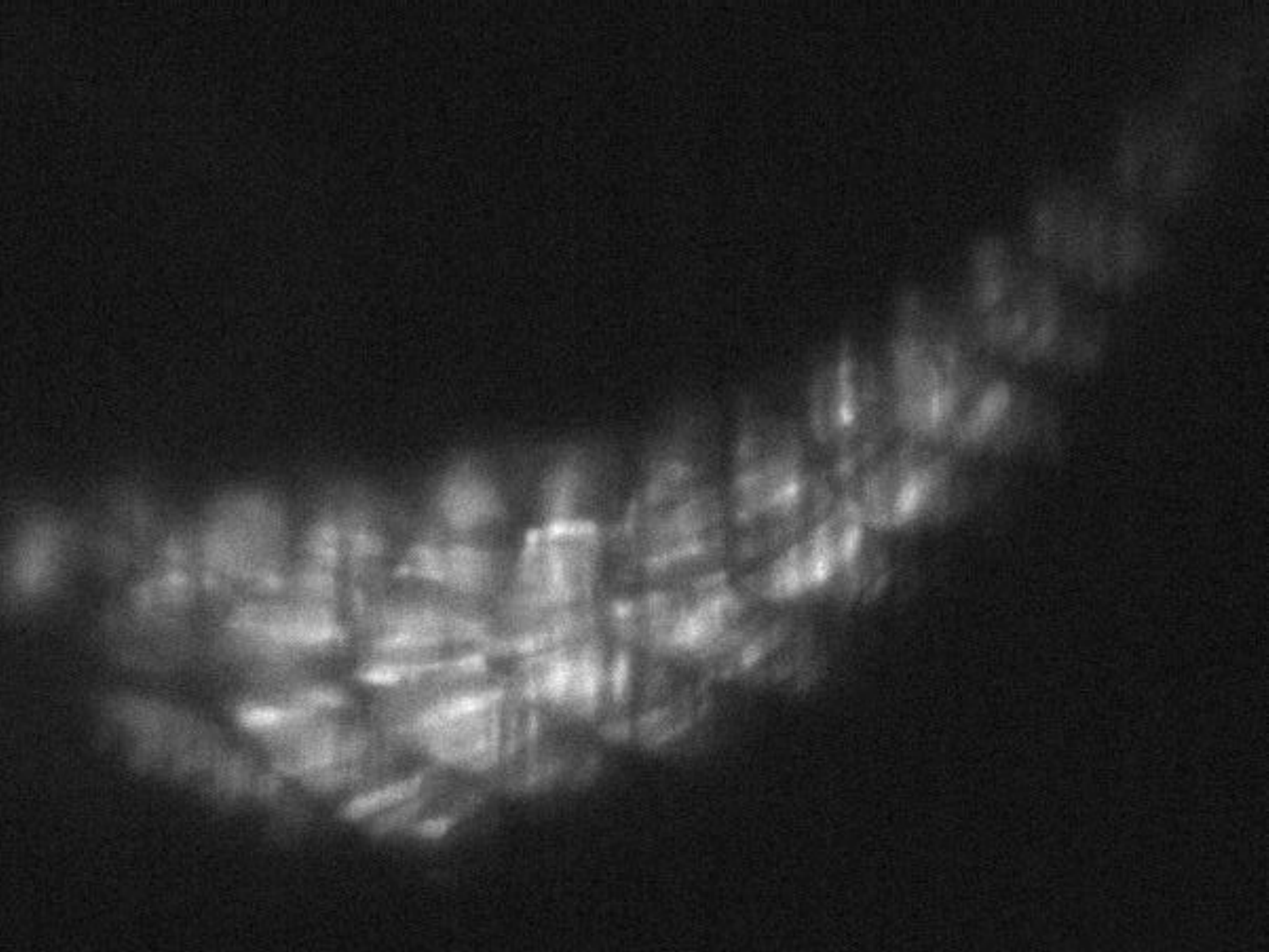}\ 
    \includegraphics[width=5cm,angle=90]{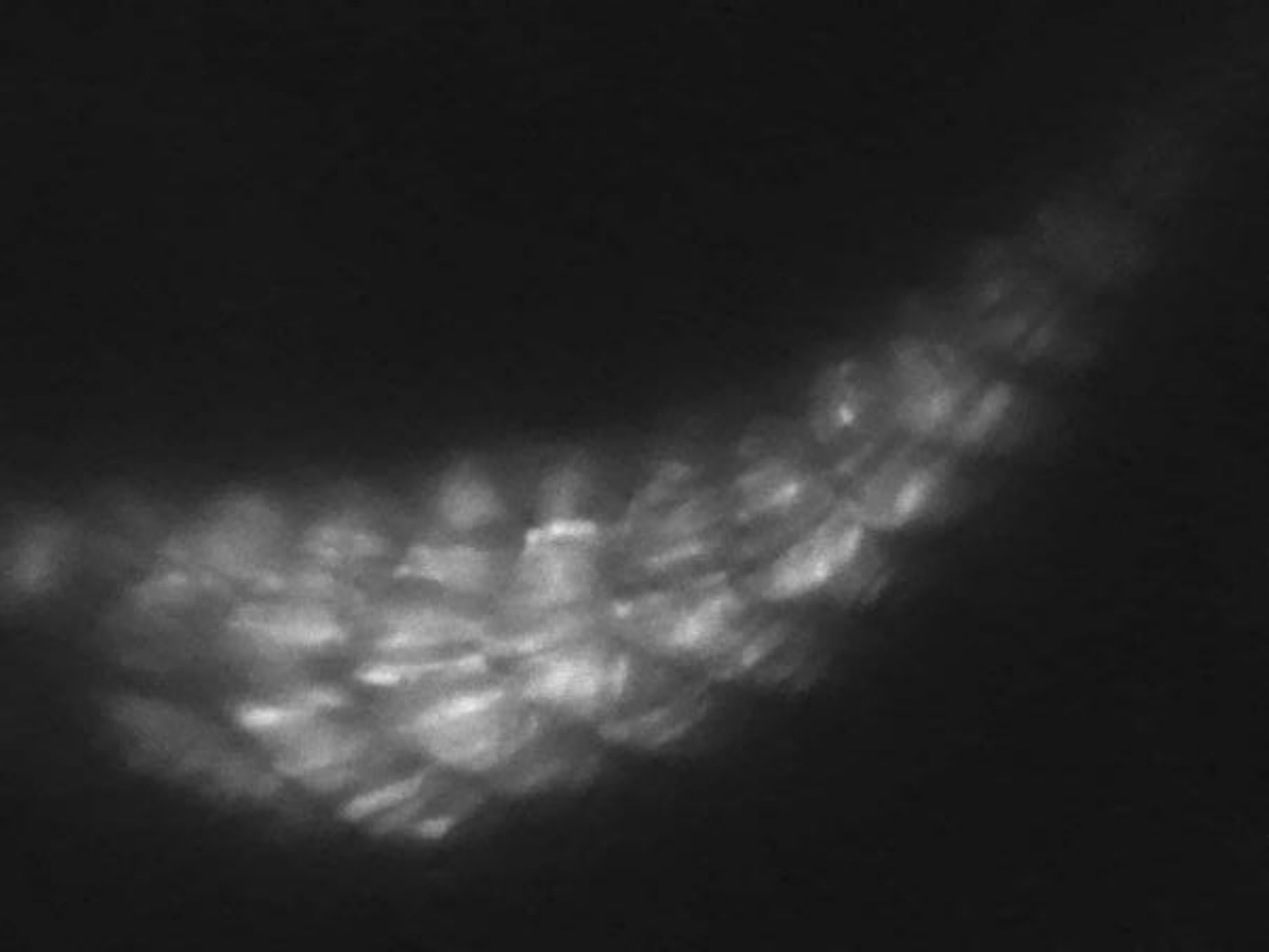}}
\end{minipage}
\caption{Left: SPIM image of a zebrafish embryo Tg.SMYH1:GFP Slow myosin Chain I specific fibers. Right: denoised image using VSNR. (Image credit: Julie Batut).}
\label{fig:SPIM}
\end{figure}

%% file: Sec_Notation.tex
\section{Notation and context} 

\subsection{Notation}
\label{sec:notation}

We consider discrete $d$-dimensional images $u\in \R^n$, where $n=n_1\cdot n_2 \cdots n_d$ denotes the pixels number.
The pixels locations belong to the set $\Omega=\{1,\cdots,  n_1\}\times \cdots \times  \{1,\cdots, n_d\}$.
The pixel value of $u$ at location $\xx\in \Omega$ is denoted $u(\xx)= u(x_1, \cdots x_d)$.
Let $u\in \R^n$ denote an image. The image $u^{mean}\in \R^n$ has all its components equal to the mean of $u$. 
The standard $l^p$-norms on $\R^n$ are denoted $\|\cdot\|_p$.
Discrete vector fields $\qq = \begin{pmatrix}
                               q_1 \\ \vdots \\ q_d
                              \end{pmatrix}
\in \R^{n\times d}$ are denoted by bold symbols. The isotropic $l^p$-norms on vector fields are denoted $\|\cdot\|_\pp$ and defined by:
\begin{equation*}
 \|\qq\|_\pp = \|\sqrt{q_1^2+\cdots + q_d^2}\|_p.
\end{equation*}

The discrete partial derivative in direction $k$ is defined by
\begin{equation*}
 \partial_k u(\cdot,x_k,\cdot) = 
\left\{
\begin{array}{ll}
 u(\cdot,x_k+1,\cdot)-u(\cdot,x_k,\cdot)  & \textrm{if \ } 1\leq x_k < n_k\\
 u(\cdot,1,\cdot)-u(\cdot,n_k,\cdot)      & \textrm{if \ } x_k=n_k.
\end{array}\right.
\end{equation*}
Using periodic boundary conditions allows to rewrite partial derivatives as circular convolutions: $\partial_k u = d_k \star u$ where $d_k$ is a finite difference filter. The discrete gradient operator in $d$-dimension is defined by:
\begin{equation*}
 \nabla = 
\left( \begin{array}{c}
\partial_1 \\ \partial_2 \\ \vdots \\ \partial_d
\end{array} \right).
\end{equation*}
The discrete isotropic total variation of $u\in \R^n$ is defined by $TV(u)=\|\nabla u\|_{\boldsymbol 1}$.
Let $\|\cdot\|_\alpha$ and $\|\cdot\|_\beta$ denote norms on $\R^n$ and $\R^m$ respectively and $A:\R^n\rightarrow \R^m$ denote a linear operator.
The subordinate operator norm $\|A\|_{\alpha \rightarrow \beta}$ is defined as follows:
\begin{equation}
\label{eq:opnorm}
 \|A\|_{\alpha \rightarrow \beta} = \max_{\|x\|_\alpha\leq 1} \|Ax\|_\beta.
\end{equation}

Let $u$ and $v$ be two $d$-dimensional images.
The pointwise product between $u$ and $v$ is denoted $u\odot v$ and the pointwise division is denoted $u\oslash v$.
The conjugate of a number or a vector $a$ is denoted $\bar a$.
The transconjugate of a matrix $\mathcal{M}\in \C^{m\times n}$ is denoted $\mathcal{M}^*$.
The canonical basis of $\R^n$ is denoted $(e_i)_{i\in \{1,\cdots,n\}}$.
The discrete Fourier basis of $\C^n$ is denoted $(f_i)_{i\in \{1,\cdots,n\}}$.
We use the convention that $\|f_i\|_\infty=1, \ \forall i$ so that $\|f_i\|_2=\sqrt{n}$ (see e.g. \cite{Mallat}).
In all the paper $\FF=\begin{pmatrix}
                         f_1^* \\ \vdots \\ f_n^*
                        \end{pmatrix}$ 
denotes the $d$-dimensional discrete Fourier transform matrix. 
The inverse Fourier transform is denoted $\FF^{-1}$ and satisfies $\FF^{-1}=\frac{\FF^*}{n}$.
The discrete Fourier transform of $u$ is denoted $\FF u$ or $\hat u$.
It satisfies $\|\hat u\|_2 = \sqrt{n}\|u\|_2$.
The discrete symmetric of $u$ is denoted $\tilde u$ and defined by $\tilde u=\FF^{-1}\bar{\hat{u}}$.
The convolution product bewteen $u$ and $\psi$ is denoted $u\star \psi$ and defined for any $\xx\in X$ by:
\begin{equation}
 u\star \psi (\xx) = \sum_{\yy\in \Omega} u(\yy) \psi(\xx - \yy)
\end{equation}
where periodic boundary conditions are used. It satifies 
\begin{equation}
 u\star \psi = \FF^{-1}\left( \hat u  \odot \hat \psi \right).
\end{equation}
Since the discrete convolution is a linear operator, it can be represented by a matrix. 
The convolution matrix associated to a kernel $\psi$ is denoted in capital letters $\Psi$:
\begin{equation}
 \Psi u = u\star\psi.
\end{equation}
The transpose of a convolution operator with $\psi$ is a convolution operator with the symmetrized kernel: $\Psi^T u = \tilde \psi \star u$.

% The indicator function of a convex and closed set $C\subseteq \R^n$ is denoted $\chi$ and defined as:
% \begin{equation*}
%  \chi_C(x)= 
% \left\{
% \begin{array}{ll}
% 0 & \textrm{if \ } x\in C \\
% +\infty & \textrm{otherwise}.
% \end{array}
% \right.
% \end{equation*}
% The proximal operator or resolvent of a convex, closed function $F:\R^n\rightarrow \R\cup \{+\infty\}$ is defined for all $x_0\in \R^n$ by:
% \begin{equation*}
% (I + \partial F^*)^{-1}(x_0) = \argmin_{x\in \R^n}  F(x) + \frac{1}{2}\|x-x_0\|_2^2.
% \end{equation*}

%% file: Sec_DescriptionOfVSNR.tex
\subsection{Decomposition algorithm}
\label{sec:decompositionalgorithm}

The VSNR algorithm (Variational Stationary Noise Removal) is described in \cite{vsnr}. The starting point of our algorithm is the following image formation model:
\begin{equation}
\label{eq:imageformation}
u_0= u +\sum_{i=1}^m \lambda_i \star \psi_i 
\end{equation}
where $u_0\in \R^n$ is the observed image and $u \in \R^n$ is the image to recover. Each $\psi_i\in \R^n$ is a known filter and each $\lambda_i \in \R^n$ is the realization of a random vector with i.i.d. entries. We assume that its density reads $\pp(\lambda_i)\propto \exp(-\phi_i(\lambda_i))$ where $\phi_i$ is a \textit{separable} function of kind 
\begin{equation}
\label{eq:separable}
\phi_i(\lambda_i)=\sum_{\xx\in\Omega} \varphi_i(\lambda_i(\xx)).
\end{equation} 
with $\varphi_i:\R\rightarrow \R\cup\{+\infty\}$ (typical examples are $l^p$ to the $p$ norms). 
Note that hypothesis  \eqref{eq:separable} is a simple consequence of the i.i.d. hypothesis. 

Our aim is to recover both the stationary components $b_i=\lambda_i\star \psi_i$ and the image $u$.
Assuming that the noise $\displaystyle b=\sum_{i=1}^m b_i$ and the image are drawn from independent random vectors, 
the maximum a posteriori (MAP) approach leads to the following optimization problem: 
\begin{equation*}
 \textrm{Find\ } (\llambda^*,u^*)\in \argmax_{\llambda \in \R^{n\times m}, u\in \R^n} \pp(\llambda,u| u_0).
\end{equation*}
Bayes' rule allows to reformulate this problem as:
\begin{equation*}
\textrm{Find\ } \llambda^* \in \argmin_{\llambda \in \R^{n\times m}, u\in \R^n} -\log\pp(u_0|\llambda,u) - \log\pp(\llambda)-\log\pp(u),
\label{eq:MAP*}
\end{equation*}
where $u=u_0-\sum_{i=1}^m \lambda_i\star \psi_i$. 
Since we assumed independence of the $\lambda_i$s,
$$
-\log\pp(\llambda)=\sum_{i=1}^m - \log\pp(\lambda_i).
$$
If we further assume that $\pp(u)\propto\exp(-\|\nabla u\|_{\boldsymbol 1})$, the optimization problem we aim at solving finally writes:
\begin{equation}
\textrm{Find } \ \ \llambda \in \Argmin_{\llambda \in \R^{n\times m}} \left|\left|\nabla \left(u_0-\sum_{i=1}^m\lambda_i\star \psi_i\right)\right|\right|_{\boldsymbol 1}+\sum_{i=1}^m \phi_i(\lambda_i).
\label{eq:OptExact}
\end{equation}

This problem is convex and can be solved efficiently using first order algorithms such as Chambolle-Pock's primal-dual method \cite{CP,vsnr2}.
The filters $\psi_i$ and the functions $\phi_i$ are user defined and should be selected using prior knowledge on the noise properties.
Unfortunately, the choice of $\phi_i$ proved to be very complicated in applications. 
Even for the special case $\phi_i(\cdot) = \frac{\alpha_i}{2} \|\cdot\|_2^2$, $\alpha_i$ is currently obtained by trial and error and interesting values vary in the range $[10^{-8},10^{10}]$ depending on the filters $\psi_i$ and the noise level. It is thus essential to restrict the range of these parameters in order to ease the task of end-users. 

Problem \eqref{eq:OptExact} is a very large scale problem since typical 3D images contain from $10^8$ to $10^9$ voxels. Most automatized parameter selection methods such as generalized cross validation  \cite{golub1979generalized} or generalized SURE \cite{2012-vaiter-acha} require to solve several instances of \eqref{eq:OptExact}. This leads to excessive computational times in our setting. In this paper, we propose to estimate the parameters $\alpha_i$ according to Morozov principle \cite{morozov1966solution}. Contrarily to recent contributions \cite{van2011sparse,AravkinSIOPT1012} which find solutions of the constrained problems by iteratively solving the unconstrained problem \eqref{eq:OptExact}, our aim is to obtain an analytical approximate value of $\alpha_i$. This approach is motivated by the fact that in denoising applications, the users usually have a crude idea of the noise level, so 
that it makes no sense to reach exactly a given noise level. Note that the constrained problem could be solved directly by using methods such as the ADMM \cite{ng2010solving,teuber2013minimization}. However, when $\phi_i(\cdot) = \frac{\alpha_i}{2} \|\cdot\|_2^2$, the Lagrangian formulation is strongly convex, while the constrained one is not, and efficient methods that converge in $O\left( \frac{1}{k^2}\right)$ can be devised in the strongly convex setting \cite{weiss2009efficient,CP}.

%\subsection{Links with existing models}

%Before analyzing our model properties, we first recall two results linking model \eqref{eq:OptExact} with the deconvolution problem and the negative norms models, showing that the proposed results could also be applied to these problems.

%\paragraph{Link with deconvolution}

%In deconvolution, the observed image $u_0$ writes:
%\begin{equation}
% u_0 = h\star u + b
%\label{eq:convolution}
%\end{equation}
%where $b$ is a white noise and $h$ is a given filter. 
%Assuming that $\hat h$ does not vanish, this equation can be rewritten:
%\begin{equation*}
% \hat u_0 \oslash \hat h = \hat u + \hat b \oslash \hat h.
%\end{equation*}
%By denoting $\tilde u_0 = \FF^{-1}\left(\hat u_0 \oslash \hat h \right)$ and $\tilde b= \FF^{-1}\left(\hat b \oslash \hat h\right)$, model \ref{eq:convolution} rewrites:
%\begin{equation*}
% \tilde u_0 = u +\tilde b
%\end{equation*}
%where $\tilde b$ is a stationary process. Deconvolution thus naturally fits formalism \ref{eq:imageformation} with $\psi= \FF^{-1}\left(\one \oslash \hat h \right)$ where $\one\in \R^n$ has all its components equal to $1$.

%\paragraph{Link with negative norms}

%In \cite{vsnr2}, we showed that if $\psi_i$ correspond to finite difference filters, problem \ref{eq:OptExact} corresponds to the negative norm models proposed e.g. in \cite{Meyer,Osher,Vese,Aujol}. (*voir si on d\'eveloppe...*)

%% file: Sec_StationaryProcesses.tex
\section{Effectiveness of the Gaussian model in the non Gaussian setting} 
\label{sec:marginalselection}

In this section we analyze the statistical properties of random processes that can be written as $\Lambda\ast \psi$ where $\Lambda$ is a white noise process.
Our main result is that the stationary noise $b_i=\lambda_i\star \psi_i$ can be assimilated to a \textit{Gaussian} colored noise for many applications of interest even if $\Lambda$ is non Gaussian.
 The heuristic reason is that if convolutions kernels with a large support are considered, then many pixels have a significant contribution to one pixel of the estimated noise component. Therefore, a central limit theorem implies that the sum of these contributions can be assimilated to a sum of Gaussian processes.

%The objective of the present section is to provide mathematical grounds to this observation.

\subsection{Distance of stationary processes to the Gaussian distribution}
\label{sec:distancestat}

Our results are simple consequences of the Berry-Esseen theorem \cite{billingsley2009convergence} that quantifies the distance between a sum of independent random variables and a Gaussian.

% \begin{theorem}[Lindeberg-Feller]
% \label{thm:LindebergFeller}
% Let $X_1,X_2,\cdots,X_n$ be independent random variables of variance $\sigma_i^2$ and mean $0$. Let $s_n=\sqrt{\sigma_1^2+\sigma_2^2+\cdots+\sigma_n^2}$ and
% $$
% S_n=\frac{X_1+X_2+\cdots+X_n}{s_n}.
% $$
% If for all $\epsilon>0$:
% \begin{equation}
% \label{eq:LindebergCondition}
% \lim_{n\rightarrow +\infty} \frac{1}{s_n^2} \sum_{i=1}^n \E(X_i^2 \one_{|X_i|>\epsilon s_n})=0
% \end{equation}
% then $S_n \stackrel{(\mathcal{D})}{\rightarrow} \mathcal{N}(0,1)$.
% \end{theorem}

\begin{theorem}[Berry-Esseen]
\label{thm:BerryEsseen}
Let $X_1,X_2,\cdots,X_n$ be independent centered random variables of finite variance $\sigma_i^2$ and finite third order moment $\rho_i=\E(|X_i|^3)$.
\begin{equation*}
\textrm{Let\ \ } \ \ S_n=\frac{X_1+X_2+\cdots+X_n}{\sqrt{\sigma_1^2+\sigma_2^2+\cdots+\sigma_n^2}}.
\end{equation*}
Let $F_n$ denote the cumulative distribution functions (cdf) of $S_n$.
Let $\Phi$ denote the cdf of the standard normal distribution. Then
\begin{equation}
\|F_n-\Phi\|_\infty\leq C_0 \frac{\sum_{i=1}^n \rho_i}{(\sum_{i=1}^n \sigma_i^2)^{3/2}}
\label{eq:Berry} 
\end{equation}
where $C_0\leq 0.56$.
\end{theorem}

In our problem, we consider random vectors of kind:
\begin{equation}
\label{eq:defB}
B=\psi\star \Lambda = \Psi \Lambda,
\end{equation}
so that
\begin{equation*}
B(\xx)=\sum_{\yy\in \Omega} \Lambda(\xx-\yy)\psi(\yy), 
\end{equation*}
where $(\Lambda(\xx))_{\xx\in \Omega}$ are i.i.d. random variables. Let us further assume that they are of finite second and third order moment \footnote{This hypothesis is not completely necessary, but simplifies the exposition.}. Denote $\sigma^2=\E(\Lambda(\yy)^2)<+\infty$ and $\rho=\E(|\Lambda(\yy)|^3)< +\infty$.
The mean of $B$ is $\mathbb{E}(B)=0$ since convolution operators preserve the set of vectors with zero mean. 
Moreover its covariance matrix is $\mathop{Cov}(B)=\sigma^2 \Psi^T\Psi$ with $\Psi^T\Psi=\FF^{-1} \mathop{diag}(|\hat \psi|^2)\FF$ whatever the distribution of $\Lambda$. Since Gaussian processes are completely described by their mean and covariance matrix, it suffices to prove that any coordinate $B(\xx)$ is close to a Gaussian for the whole process $B$ to be near Gaussian. The following results state that $B$ is close to a Gaussian random vector whatever the law of $\Lambda$ if the filter $\psi$ satisfies a geometrical criterion discussed later.

\begin{proposition}
 \label{prop:NonAsymptot}
Let $G$ denote the cdf of $\frac{B(\xx)}{s}$ where $s=\|\psi\|_2$. 
This cdf is independent of $\xx$, moreover:
\begin{equation}
\label{eq:boubou}
 \|G-\Phi\|_\infty\leq 0.56 \frac{\rho}{\sigma^3} \frac{\|\psi\|_3^3}{\|\psi\|_2^3}.
\end{equation}
\end{proposition}
\begin{proof}
The independence w.r.t. $\xx$ is a direct consequence of the stationarity of $B$.
Bound \eqref{eq:boubou} is a direct consequence of Berry-Esseen theorem \ref{thm:BerryEsseen}. 
It suffices to notice that $\E(|\Lambda(\xx-\yy)\psi(\yy)|^2)=\psi(\yy)^2\sigma^2$,  $\E(|\Lambda(\xx-\yy)\psi(\yy)|^3)=|\psi(\yy)|^3\rho$ for any $(\xx,\yy)\in \Omega^2$ and to apply theorem \ref{thm:BerryEsseen}. 
\end{proof}

Thus, if  $\frac{\|\psi\|_3^3}{\|\psi\|_2^3}$ is sufficiently small, the distribution of $B$ will be near Gaussian. 
The following result clarifies this condition in an asymptotic regime. 
\begin{proposition}
\label{prop:AymptotOne}
Let $\psi:\R_+^d\rightarrow \R$ denote a function.
Let $\Omega_n=[1,n]^d\cap \Z^d$ denote a Euclidean grid. Let $\displaystyle s_n=\sqrt{\sum_{\xx\in \Omega_n}\psi^2(\xx)}$.
If $\Lambda(\xx)$ is of finite second and third order moment and the sequence $(\psi(\xx))_{\xx\in \Z^d}$ is uniformly bounded $|\psi(\xx)|\leq M<+\infty, \ \forall \xx\in \Z^d$ and has infinite variance $\displaystyle\lim_{n\rightarrow +\infty} s_n=+\infty$, then for all $\xx\in \Omega_n$:
\begin{equation}
\frac{B(\xx)}{s_n} \stackrel{(\mathcal{D})}{\rightarrow} \mathcal{N}(0,\sigma^2).
\end{equation}
\end{proposition}
\begin{proof}
Let us denote:
\begin{equation}
\label{eq:fn}
 f(n)=\frac{\sum_{\xx \in \Omega_n} |\psi(\xx)^3|}{\left( \sum_{\xx \in \Omega_n} \psi(\xx)^2\right)^{3/2}}.
\end{equation}
We have 
\begin{align*}
\sum_{\xx \in \Omega_n} |\psi(\xx)|^3 & \leq \sum_{\xx \in \Omega_n} \|\psi\|_\infty \psi(\xx)^2 \\
				     & \leq M s_n^2.
\end{align*}
Thus:
\begin{align*}
 f(n)\leq\frac{M s_n^2}{s_n^3}=\frac{M}{s_n}.
\end{align*}
The right-hand side in (\ref{eq:Berry}) is $f(n)$ and goes to $0$ as $n\rightarrow +\infty$. 
Lindeberg-Feller theorem could also be used in this context and allow to avoid moment conditions. 
%This refinement seems useless in practical applications.
\end{proof}

\subsection{Examples}

We present different examples of kernels where the Theorem \ref{thm:BerryEsseen} applies.

\begin{example}
We first consider a kernel that is the indicator function of a "large" set, namely $\psi(\xx)=1$ if $\xx\in I$, and $\# I=N$.
Then the upper bound in Equation \eqref{eq:Berry} is $C_0/\sqrt{N}$. It becomes small when $N$ becomes large.
\end{example}

\begin{example}
Let us study the case of  kernels with a (slow enough) power decay: $\psi(\xx)=|\xx|^\alpha$, for some $-d/2<\alpha<0$. In this case, the quantity $s_n$ tends to infinity since it is asymptotic to 
$$\int_{[1,n]^d} |\xx|^{2\alpha} d\xx \sim K \int_{r=1}^nr^{d-1}r^{2\alpha}dr\sim Kn^{d+2\alpha}$$
for some constant $K$. Therefore Proposition \ref{prop:AymptotOne} applies. This result is still valid for $\alpha\ge0$.
\end{example}

\begin{example} We treat the case of an anisotropic Gaussian filter $\psi$ with axes aligned with the coordinate axes. In this case the variance is finite and proposition \ref{prop:AymptotOne} does not apply. However we can give an explicit value of the upper bound in \eqref{eq:boubou}, which ensures that the process is close from a Gaussian. Let us assume that 
$$\psi(\xx)=Ke^{-\sum_{i=1}^d x_i^2/2\sigma_i^2},$$
where $K$ is a normalizing constant and $\xx=(x_1,x_2,\ldots,x_d)\in\Z^d$.  %In other words, $\psi$ is an anisotropic Gaussian, with principal directions parallel to the coordinate axes.
We provide in this case  an upper bound for $f(n)$ in terms of $(\sigma_i)$.
For the sake of simplicity we assume that $K=1$.% (it is not restrictive because $f(n)$ is homogeneous).
\begin{align*}
\sum_{\xx \in \Z^d} |\psi(\xx)^3|&=\sum_{(n_1,\dots,n_d)\in\Z^d}e^{-3\sum_{1\le i\le d}n_i^2/2\sigma_i^2}\\
&=\prod_{i=1}^d\left(\sum_{n\in\Z}e^{-3n^2/2\sigma_i^2}\right)\\
&=\prod_{i=1}^d\left(1+2\sum_{n>0}e^{-3n^2/2\sigma_i^2}\right)\end{align*}
and similarly
\begin{align*}
\sum_{\xx \in \Z^d} |\psi(\xx)^2|&=\prod_{i=1}^d\left(1+2\sum_{n>0}e^{-n^2/\sigma_i^2}\right)
\end{align*}
We use the following inequalities
$$\dfrac{1}{2}\sqrt{\dfrac{\pi}{\alpha}}-1\le\int_{1}^{+\infty}e^{-\alpha t^2}dt\le\sum_{n>0}e^{-\alpha n^2}\le \int_0^{+\infty}e^{-\alpha t^2}dt=\dfrac{1}{2}\sqrt{\dfrac{\pi}{\alpha}}$$
to obtain
$$\max\left(1,\sqrt{\dfrac{\pi}{\alpha}}-1\right)\le1+2\sum_{n>0}e^{-\alpha n^2}\le 1+\sqrt{\dfrac{\pi}{\alpha}}.$$

It follows that
$$\lim_{n\to\infty}f(n)\le\prod_{i=1}^d\dfrac{\left(1+\sigma_i\sqrt{{2\pi}/{3}}\right)}{\max\left(1,\sigma_i\sqrt{\pi}-1\right)^{3/2}}=\prod_{i=1}^d g(\sigma_i).$$
Note that $\displaystyle g(\sigma) = \mathop{O}_{+\infty}(\frac{1}{\sqrt{\sigma}})$. In other words if the Gaussian kernel has sufficiently large variances, the constant in the upper bound of \eqref{eq:Berry} is small.
\end{example}

\begin{example}\label{ex:illustrate}
In this example, we illustrate the theorem on a practical setting.
Let us assume that $\Lambda(\xx)$ is a Bernoulli-uniform random variable in order to model sparse processes. 
With this model $\Lambda(\xx)=0$ with probability $1-\gamma$ and takes a random value distributed uniformly in $[-1,1]$ with probability $\gamma$. 
Simple calculation leads to $\sigma^2=\frac{\gamma}{3}$ and $\rho=\frac{\gamma}{4}$ so that equation \eqref{eq:boubou} gives:
\begin{equation}
 \|G-\Phi\|_\infty\leq \frac{0.73}{\sqrt{\gamma}} \frac{\|\psi\|_3^3}{\|\psi\|_2^3}.
\end{equation}

Let us define a 2D anisotropic Gaussian filter as:
\begin{equation}
\label{eq:Gaussian}
\psi(x_1,x_2)=C\exp\left(-\frac{x_1^2}{2\sigma_1^2}-\frac{x_2^2}{2\sigma_2^2}\right)
\end{equation}
where $C$ is a normalization constant. This filter is used frequently in the microscopy experiments we perform and is thus of particular interest. 
Figure \ref{fig:stationnaryprocesses} shows practical realizations of stationary processes defined as $\Lambda\star \psi$. 
Note that as $\sigma_1$ or $\gamma$ increase, the texture gets similar to the Gaussian process on the last row. 
Table \ref{tab:practicalbounds} quantifies the proximity of the non Gaussian process to the Gaussian one using proposition \ref{prop:NonAsymptot}.
The processes can hardly be distiguished from a perceptual point of view when the right hand-side in \eqref{eq:boubou} is less than $0.4$.

 \begin{figure}[!htbp]
  \centering
 \begin{tabular}{|c|c|c|c|c|c|}
  \hline
	\backslashbox{$\gamma$}{$\sigma_1$} &\textbf{2}&\textbf{8}&\textbf{32}&\textbf{64}&\textbf{128}\\\hline 
	\textbf{0.001}&
  \includegraphics[width=0.15\textwidth]{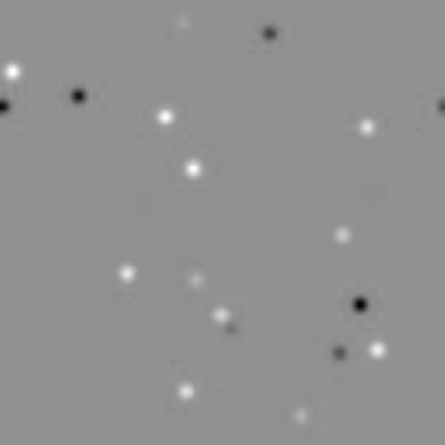}&
  \includegraphics[width=0.15\textwidth]{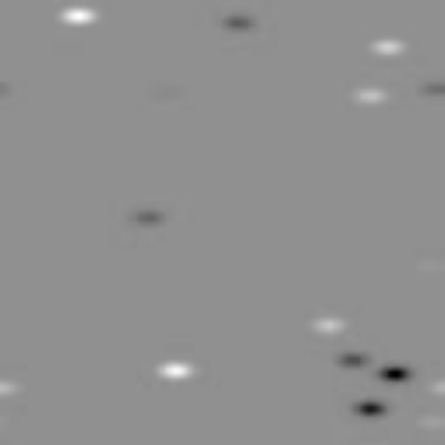}&
  \includegraphics[width=0.15\textwidth]{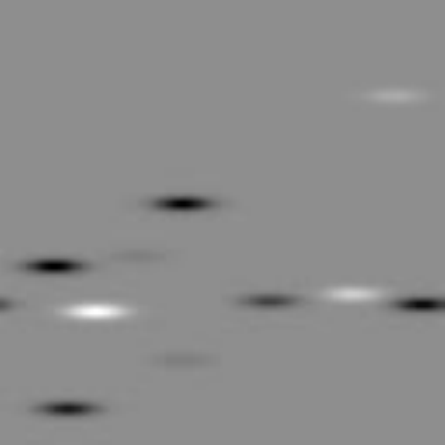}&
  \includegraphics[width=0.15\textwidth]{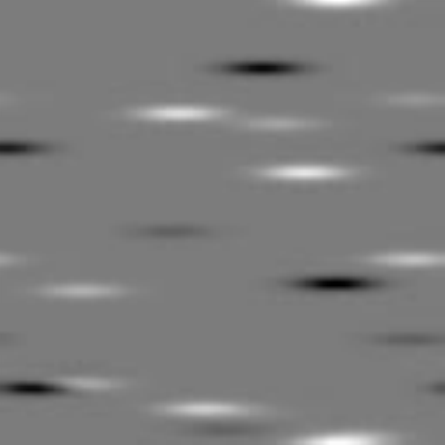}&
  \includegraphics[width=0.15\textwidth]{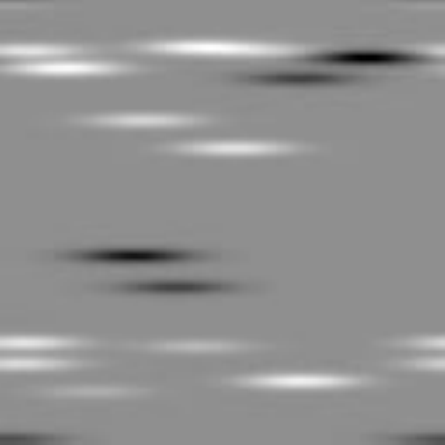}\\
  \hline 
  \textbf{0.01} &
  \includegraphics[width=0.15\textwidth]{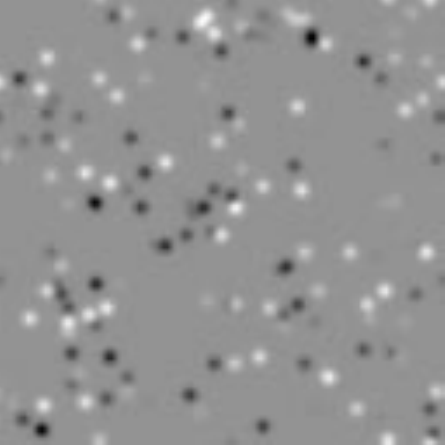}&
  \includegraphics[width=0.15\textwidth]{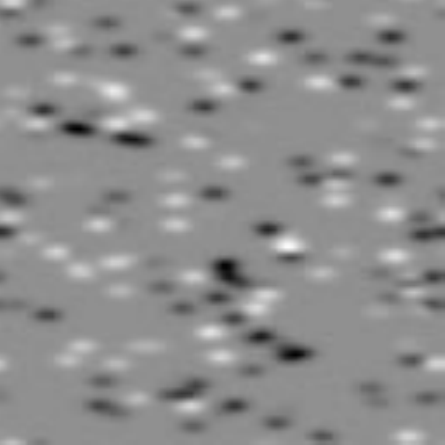}&
  \includegraphics[width=0.15\textwidth]{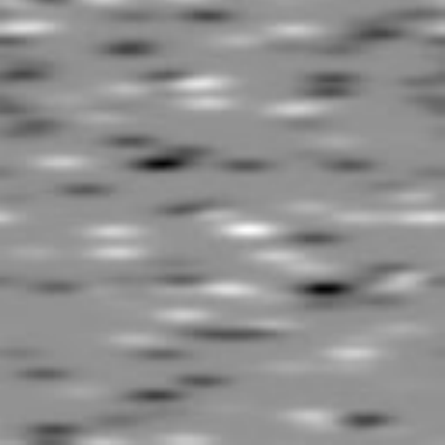}&
  \includegraphics[width=0.15\textwidth]{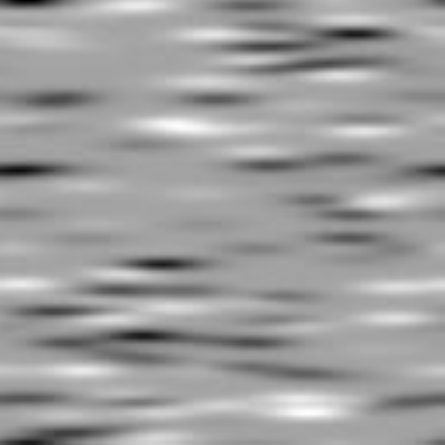}&
  \includegraphics[width=0.15\textwidth]{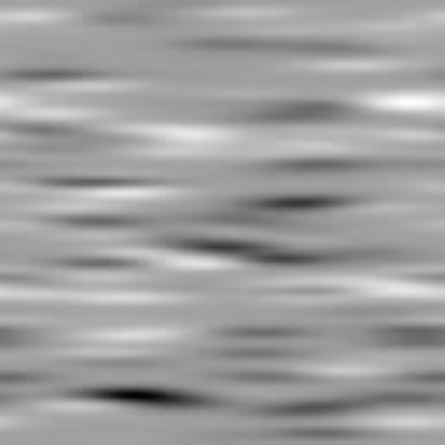}\\
  \hline 
  \textbf{0.05}&
  \includegraphics[width=0.15\textwidth]{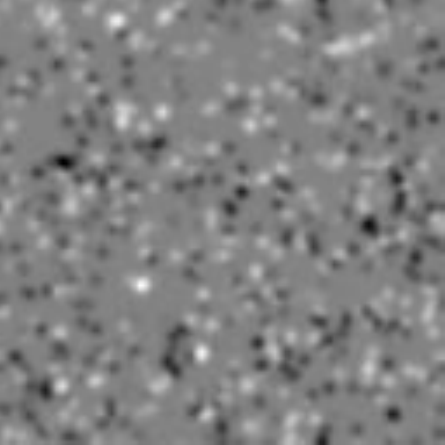}&
  \includegraphics[width=0.15\textwidth]{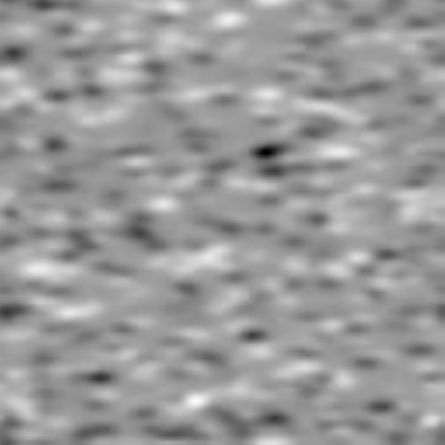}&
  \includegraphics[width=0.15\textwidth]{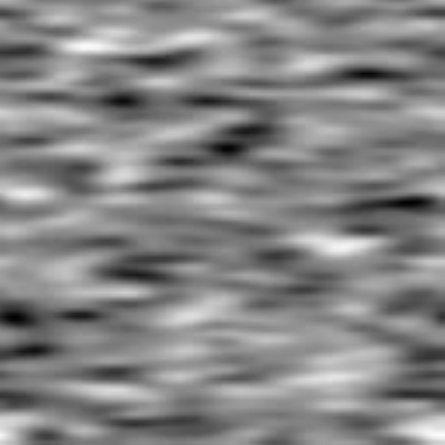}&
  \includegraphics[width=0.15\textwidth]{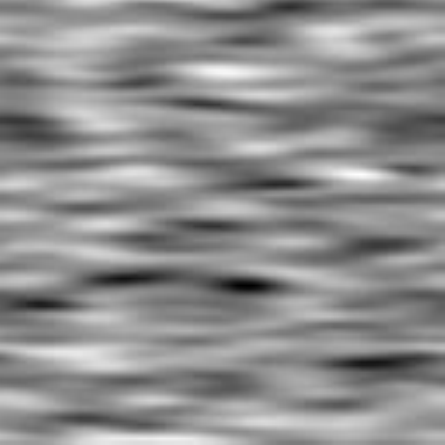}&
  \includegraphics[width=0.15\textwidth]{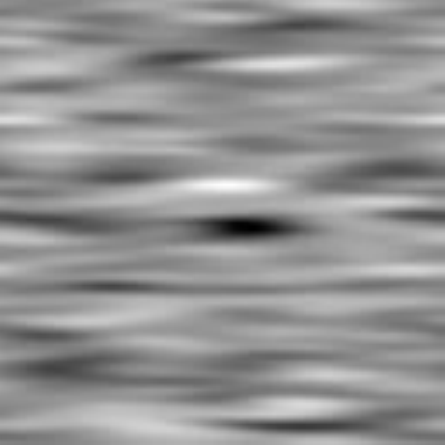}\\
  \hline 
  \textbf{0.1}&
  \includegraphics[width=0.15\textwidth]{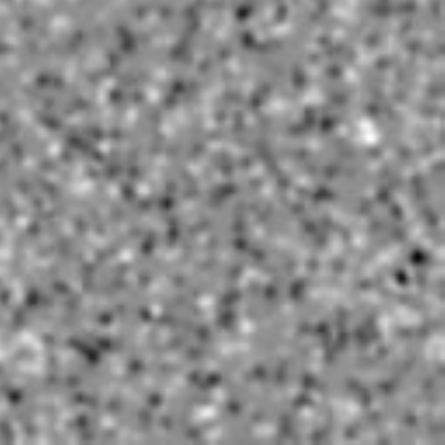}&
  \includegraphics[width=0.15\textwidth]{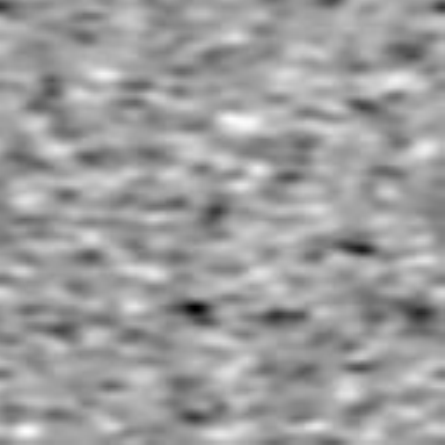}&
  \includegraphics[width=0.15\textwidth]{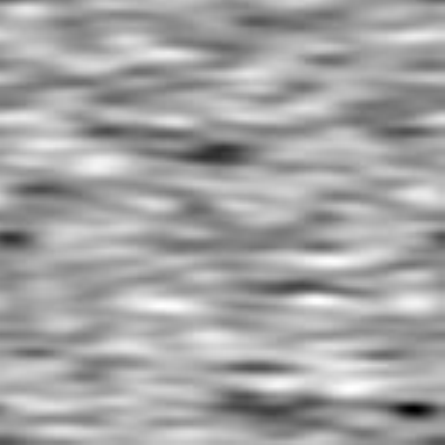}&
  \includegraphics[width=0.15\textwidth]{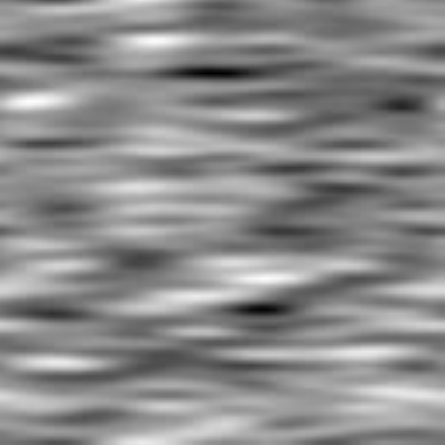}&
  \includegraphics[width=0.15\textwidth]{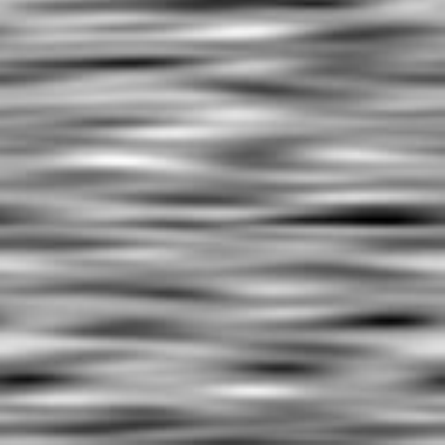}\\
  \hline 
  \textbf{0.5}&
  \includegraphics[width=0.15\textwidth]{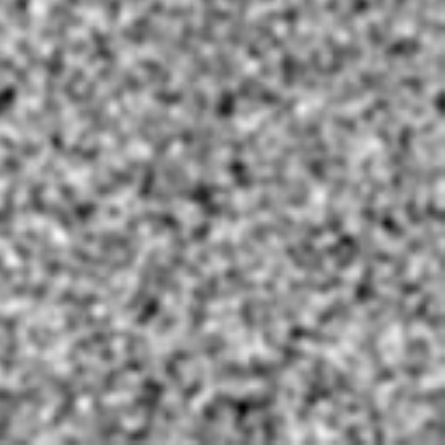}&
  \includegraphics[width=0.15\textwidth]{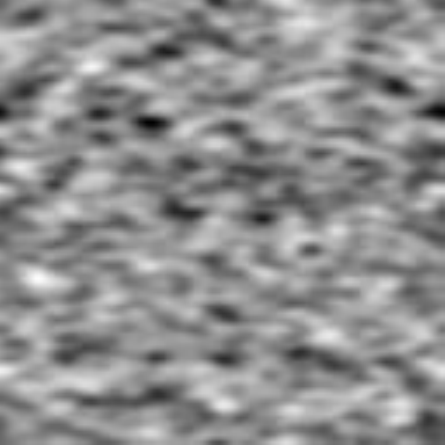}&
  \includegraphics[width=0.15\textwidth]{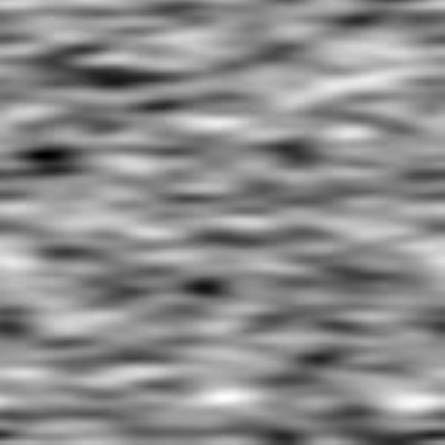}&
  \includegraphics[width=0.15\textwidth]{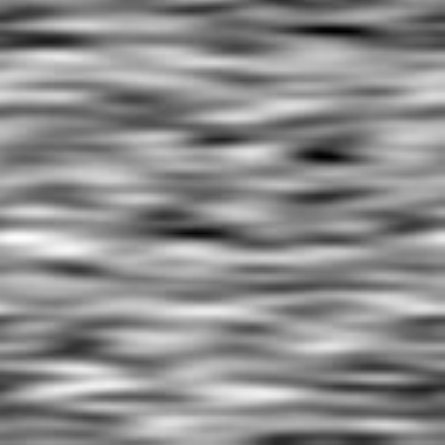}&
  \includegraphics[width=0.15\textwidth]{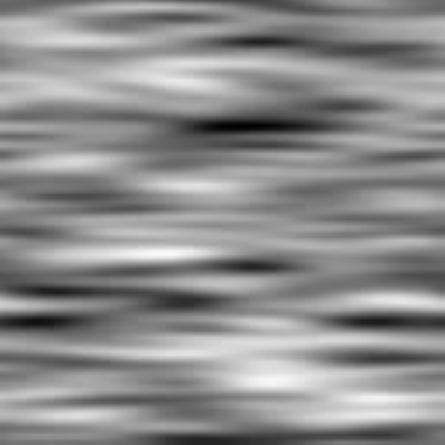}\\
  \hline 
  \textbf{1}&
  \includegraphics[width=0.15\textwidth]{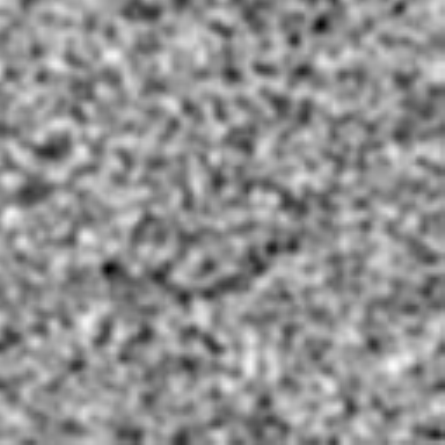}&
  \includegraphics[width=0.15\textwidth]{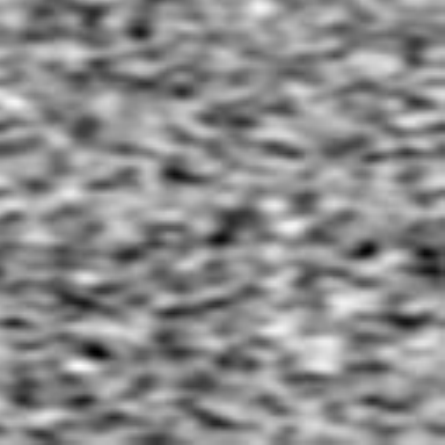}&
  \includegraphics[width=0.15\textwidth]{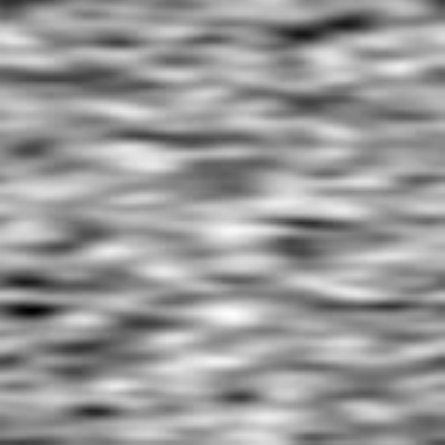}&
  \includegraphics[width=0.15\textwidth]{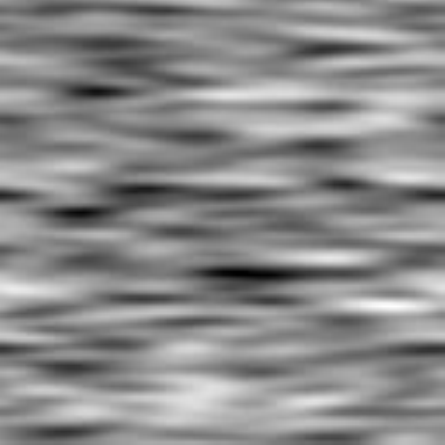}&
  \includegraphics[width=0.15\textwidth]{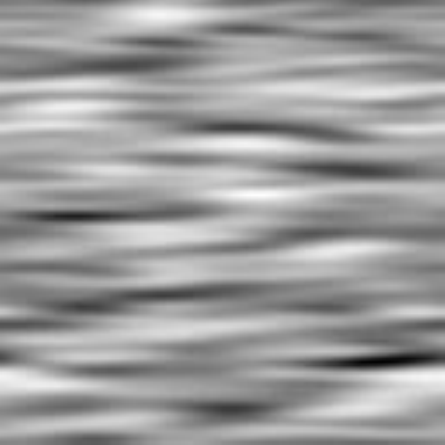}\\
  \hline\hline 
  \rotatebox{90}{Gaussian process}&
  \includegraphics[width=0.15\textwidth]{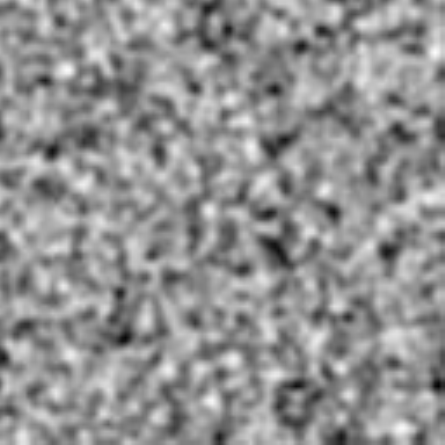}&
  \includegraphics[width=0.15\textwidth]{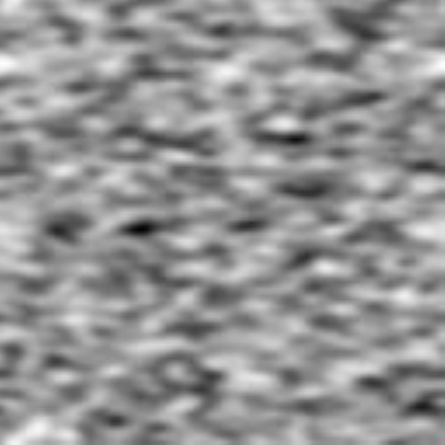}&
  \includegraphics[width=0.15\textwidth]{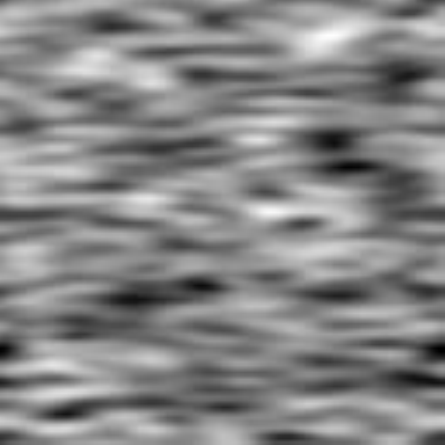}&
  \includegraphics[width=0.15\textwidth]{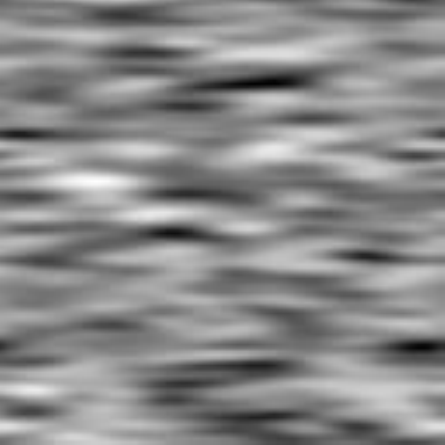}&
  \includegraphics[width=0.15\textwidth]{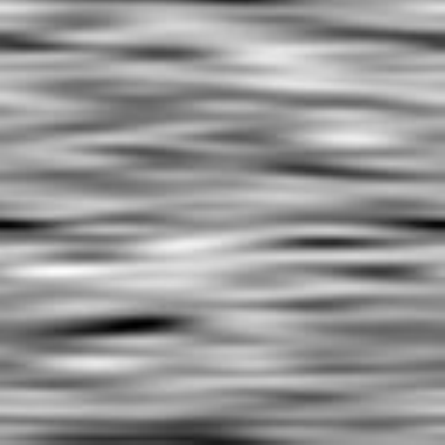}\\
  \hline 
  \end{tabular}
\caption{The first six rows show stationnary processes obtained by convolving an anisotropic Gaussian filter with Bernoulli uniform processes for different values of $\gamma$ and different values of $\sigma_1$. The value of $\sigma_2=2$.
The last row shows a Gaussian process obtained by convolving Gaussian white noise with the same Gaussian filter.}\label{fig:stationnaryprocesses}
\end{figure}

\begin{table}
\centering
\begin{tabular}{|l|c|c|c|c|c|}
\hline
\backslashbox{$\gamma$}{$\sigma_1$} &\textbf{2}&\textbf{8}&\textbf{32}&\textbf{64}&\textbf{128}\\\hline
\textbf{0.001}&1.00&1.00&1.00&1.00&1.00\\\hline
\textbf{0.01}&1.00&1.00&0.98&0.82&0.69\\\hline
\textbf{0.05}&0.88&0.62&0.44&0.37&0.31\\\hline
\textbf{0.1}&0.62&0.44&0.31&0.26&0.22\\\hline
\textbf{0.5}&0.28&0.20&0.14&0.12&0.10\\\hline
\textbf{1}&0.20&0.14&0.10&0.08&0.07\\\hline
\end{tabular} 
\caption{Values of bound \eqref{eq:boubou} with respect to $\gamma$ and $\sigma_1$. \label{tab:practicalbounds}}
\end{table}
\end{example}

\subsection{Numerical validation}

In the previous paragraphs we showed that in many situations, stationary random processes $B$ of kind 
\begin{equation}
B=\Lambda\star \psi
\end{equation}
where $\Lambda$ denotes a white noise process can be assimilated to a coloured Gaussian noise. 
A Bayesian approach thus indicates that problem \eqref{eq:OptExact} can be replaced by the following approximation:
\begin{equation}
\label{eq:pb}
 \textrm{Find \ } \llambda(\alpha)= \argmin_{\llambda\in \R^{n\times m}} \|\nabla (u_0-\sum_{i=1}^m\psi_i \star\lambda_i) \|_{\boldsymbol 1} + \sum_{i=1}^m\frac{\alpha_i}{2} \|\lambda_i\|_2^2
\end{equation}
for a particular choice of $\alpha_i$ discussed later. This new problem has an attractive feature compared to \eqref{eq:OptExact}: it is strongly convex in $\llambda$, which implies uniqueness of the minimizer and the existence of efficient minimization algorithms. Unfortunately, it is well known that Bayesian approaches can substantially deviate from the prior models that underly the MAP estimator \cite{nikolova2007model}. The aim of this paragraph is to validate the proposed approximation experimentally. We consider a problem of stationary noise removal.

We generate stationary processes from the models described in Example \ref{ex:illustrate} and Figure \ref{fig:stationnaryprocesses} for different values of $\gamma$. Bernoulli-uniform processes are generated from functions $\phi_i$ that are nonconvex ($l^0$-norms) and in this case, problem \eqref{eq:OptExact} is a hard combinatorial problem. 
We denoise the image using either a standard $l^1$-norm relaxation:
\begin{equation}
\label{eq:pbl1}
 \textrm{Find \ } \lambda \in \Argmin_{\lambda\in \R^n} \|\nabla (u_0- \psi \star\lambda) \|_{\boldsymbol 1} + \alpha \|\lambda\|_1,
\end{equation}
or the $l^2$-norm approximation suggested by the previous theorems: 
\begin{equation}
\label{eq:pbl2}
 \textrm{Find \ } \lambda \in \Argmin_{\lambda\in \R^n} \|\nabla (u_0- \psi \star\lambda) \|_{\boldsymbol 1} + \frac{\alpha}{2} \|\lambda\|_2^2.
\end{equation}
The optimal parameter $\alpha$ is estimated by dichotomy in order to maximize the SNR of the denoised image. 
As can be seen in Figure \ref{fig:denoisingstationary} the $l^1$-norm approximation provides better results for very sparse Bernoulli processes and the $l^2$ approximation provides similar or better results when the Bernoulli process gets denser. This  confirms the results presented in section \ref{sec:distancestat}. 

\begin{figure}[!htbp]
  \centering
 \begin{tabular}{|c|c|c|c|}
  \hline
  & 6.02dB & 27.09dB & 16.87dB \\
  \textbf{0.001}&
  \includegraphics[width=0.15\textwidth]{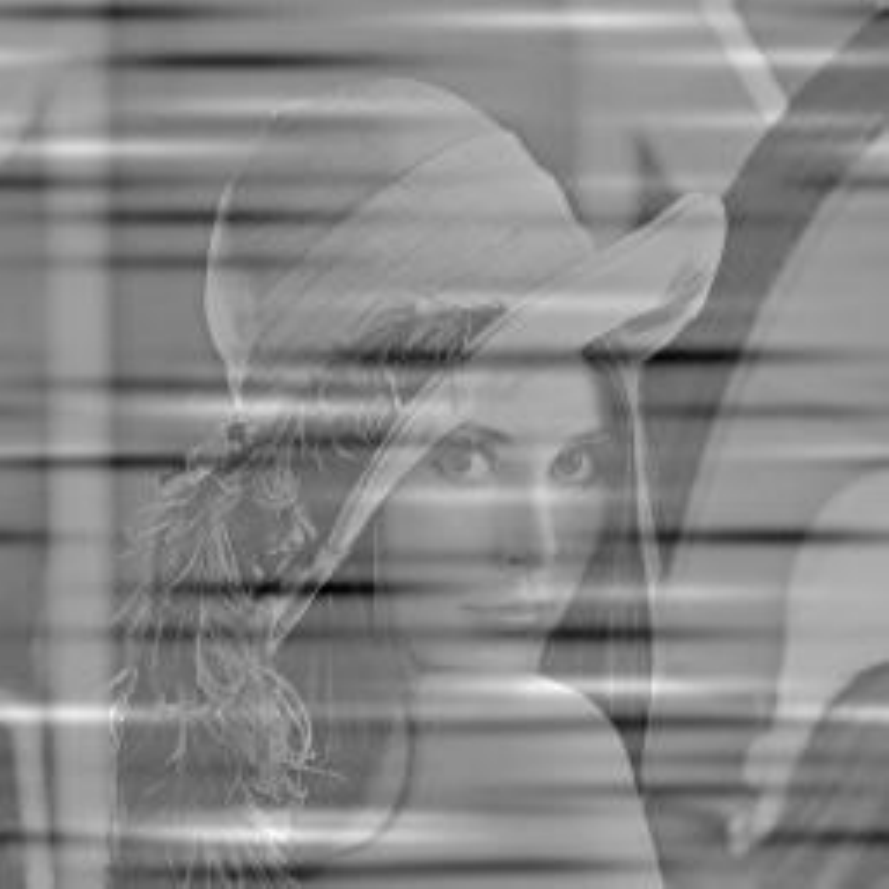}&
  \includegraphics[width=0.15\textwidth]{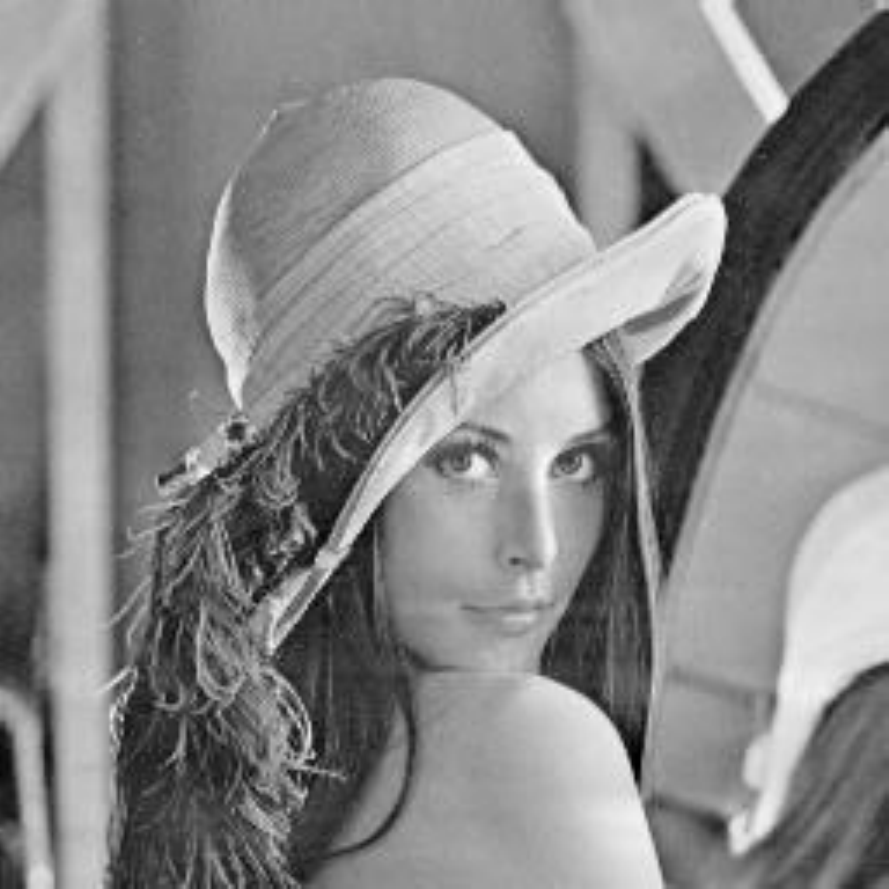}&
  \includegraphics[width=0.15\textwidth]{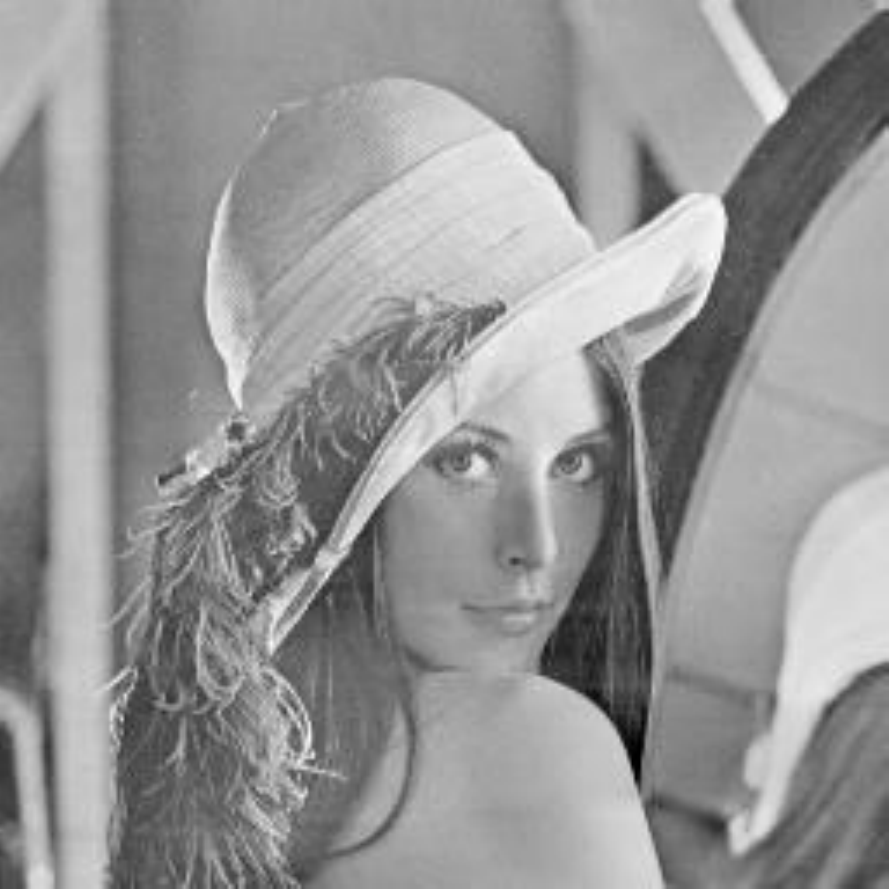}\\
  \hline
  & 6.02dB & 16.41dB & 15.87dB \\
  \textbf{0.01} &
  \includegraphics[width=0.15\textwidth]{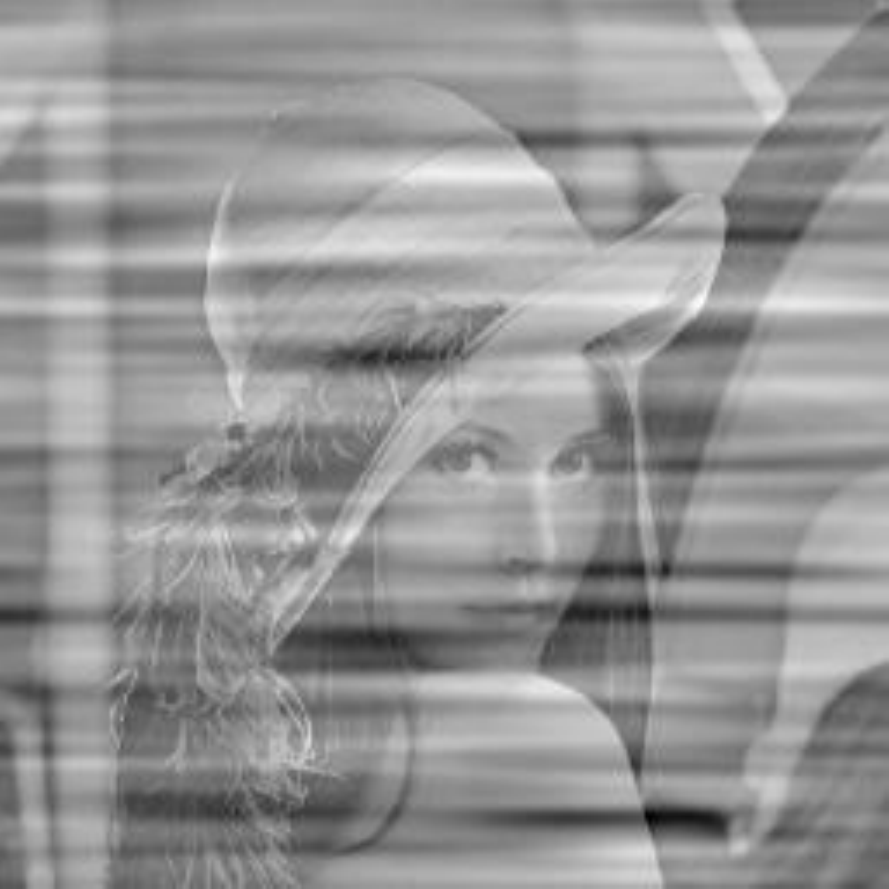}&
  \includegraphics[width=0.15\textwidth]{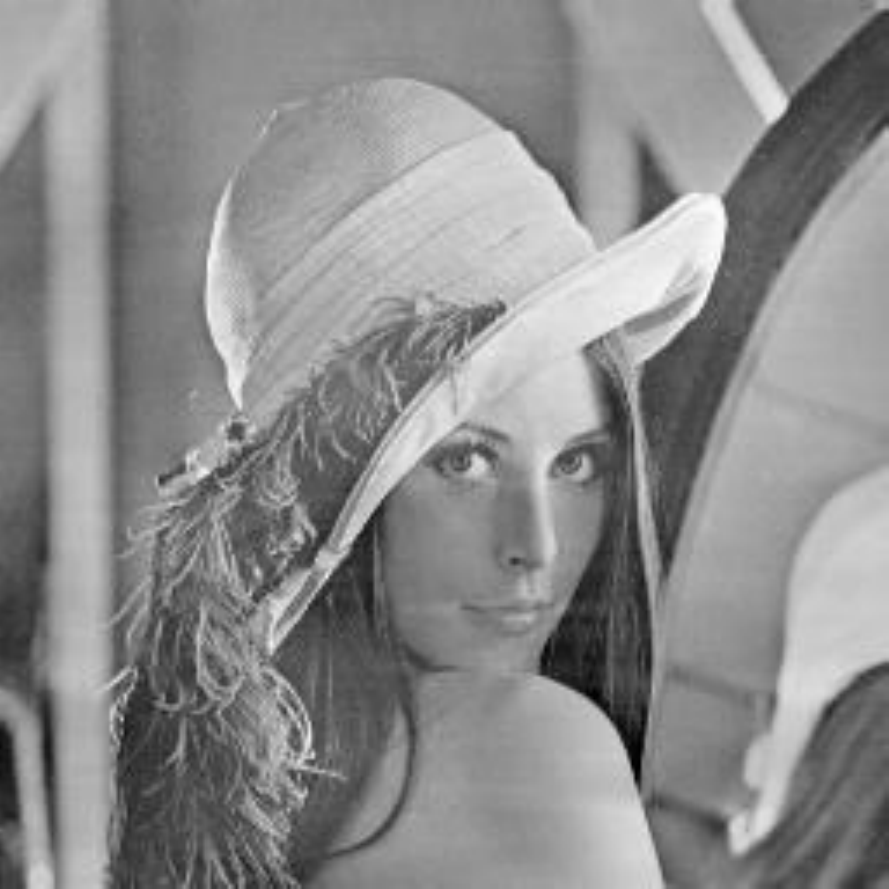}&
  \includegraphics[width=0.15\textwidth]{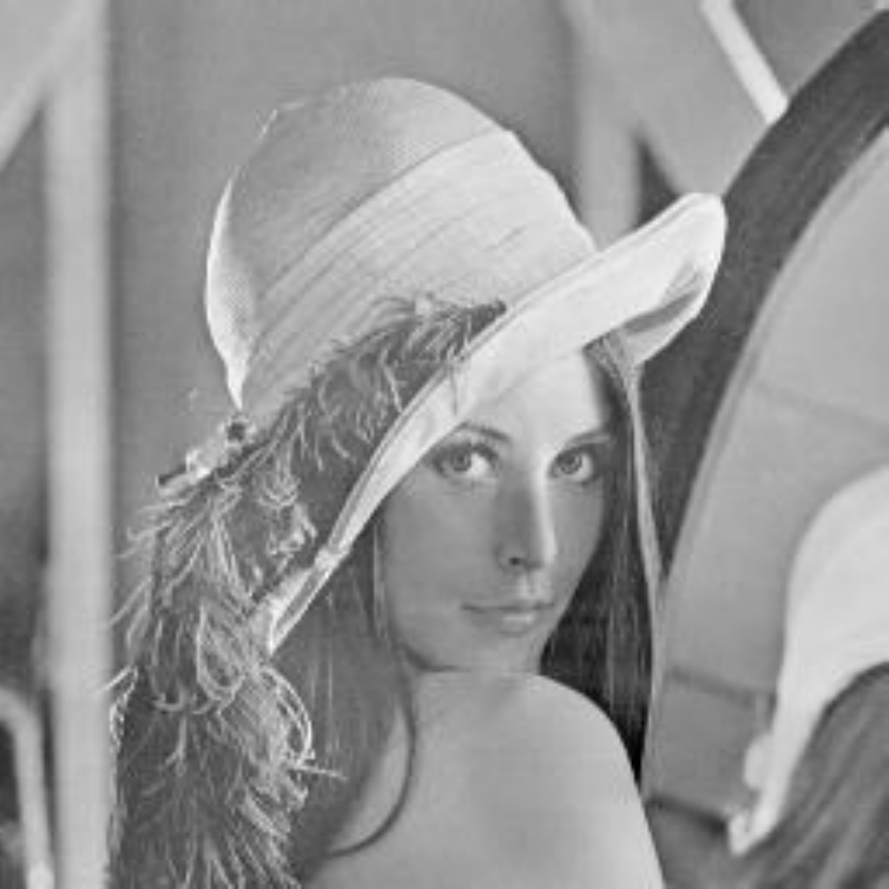}\\
  \hline
  & 6.02dB & 17.65dB & 17.53dB \\
  \textbf{0.05}&
  \includegraphics[width=0.15\textwidth]{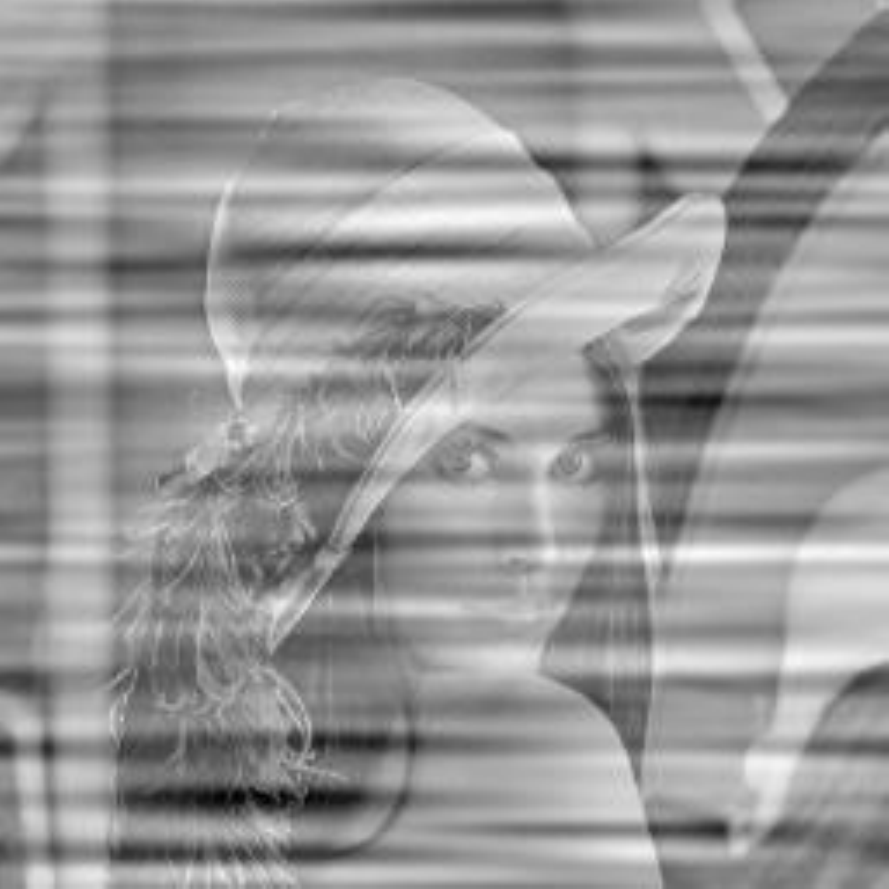}&
  \includegraphics[width=0.15\textwidth]{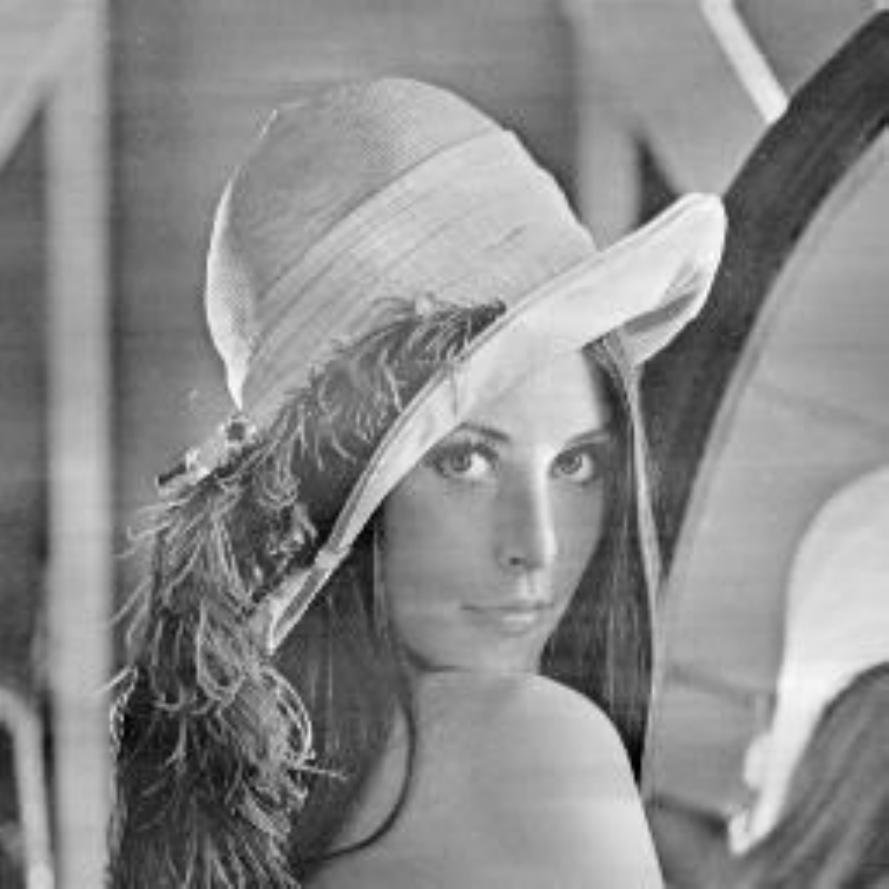}&
  \includegraphics[width=0.15\textwidth]{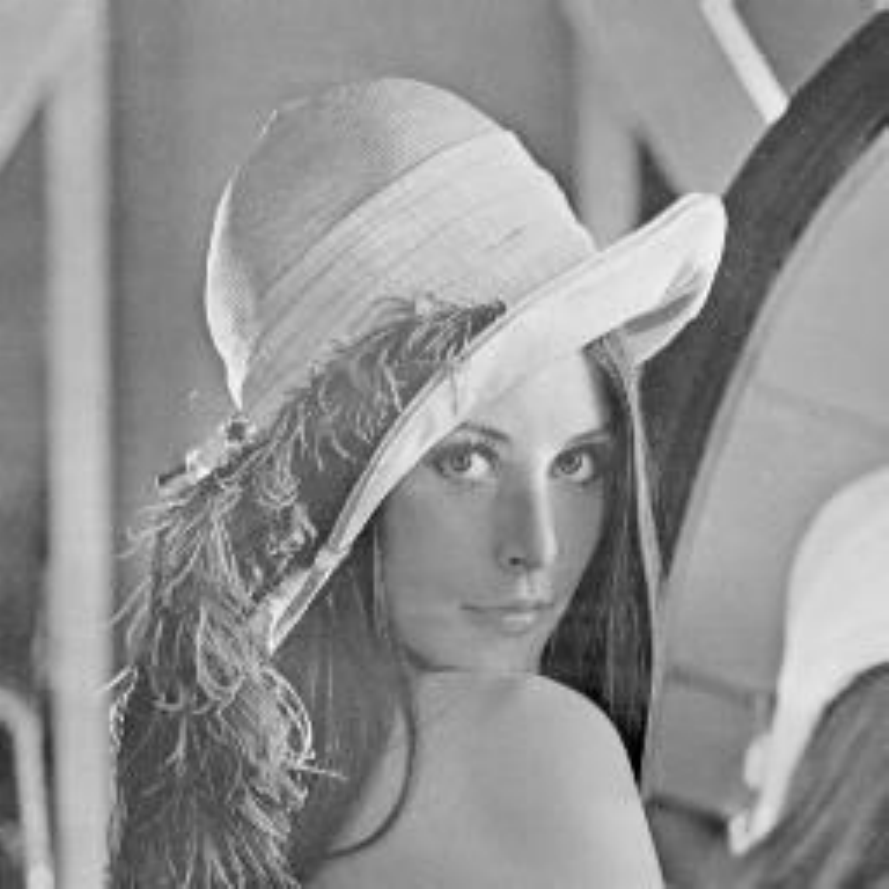}\\
  \hline
  & 6.02dB & 18.61dB & 18.24dB \\
  \textbf{0.1}&
  \includegraphics[width=0.15\textwidth]{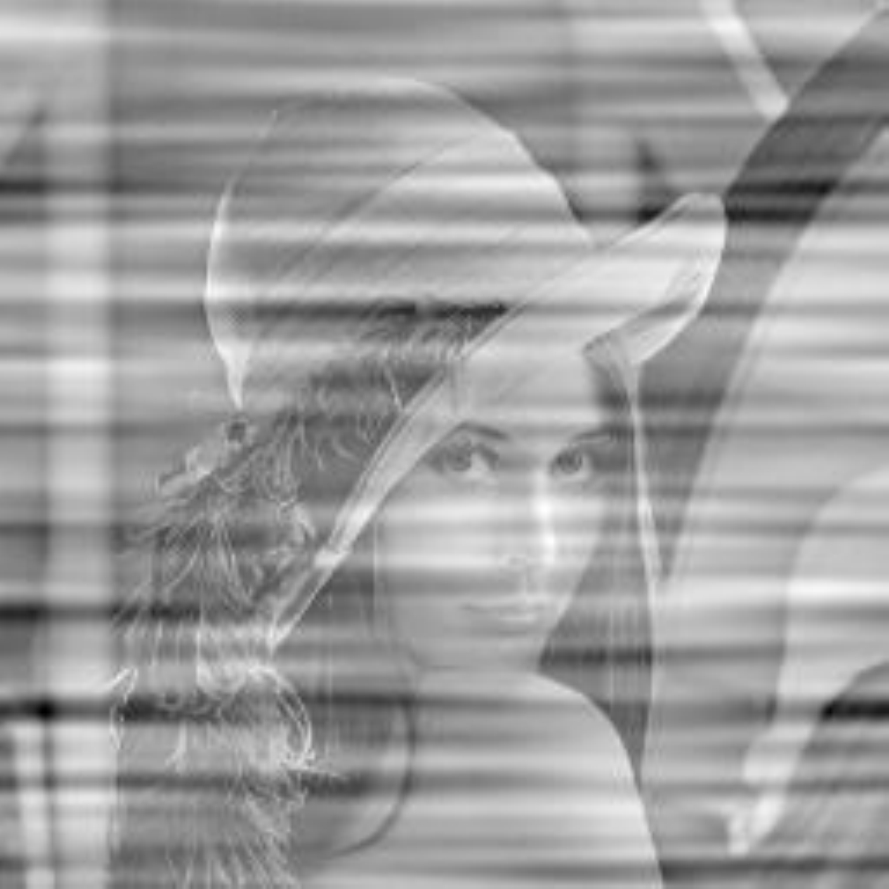}&
  \includegraphics[width=0.15\textwidth]{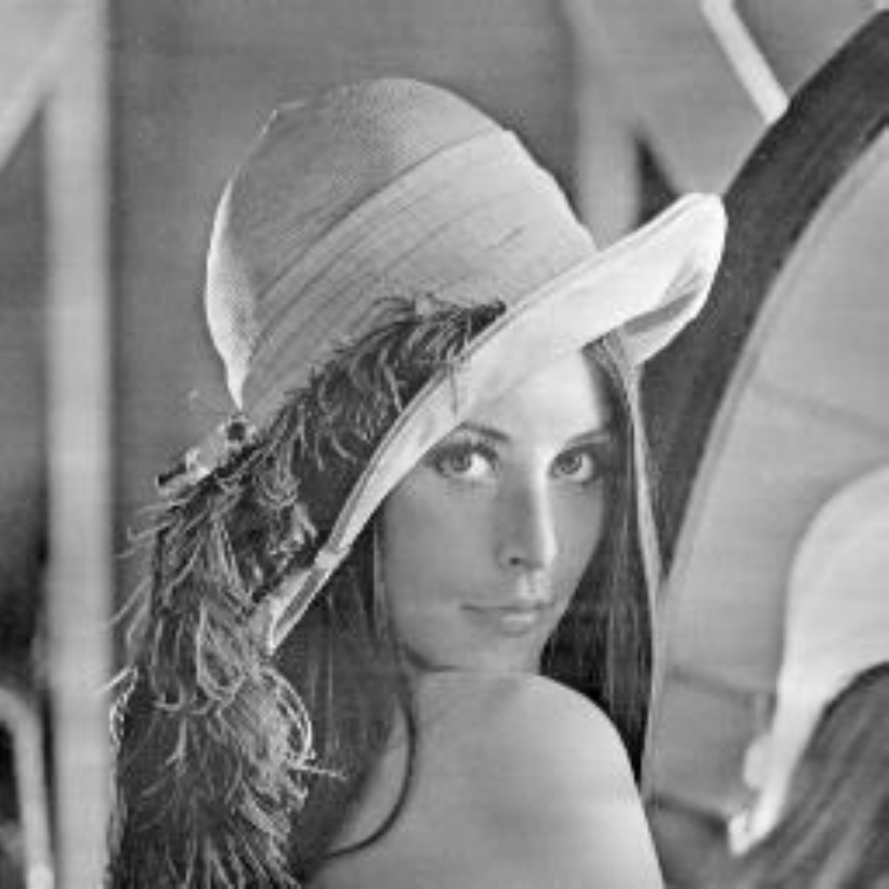}&
  \includegraphics[width=0.15\textwidth]{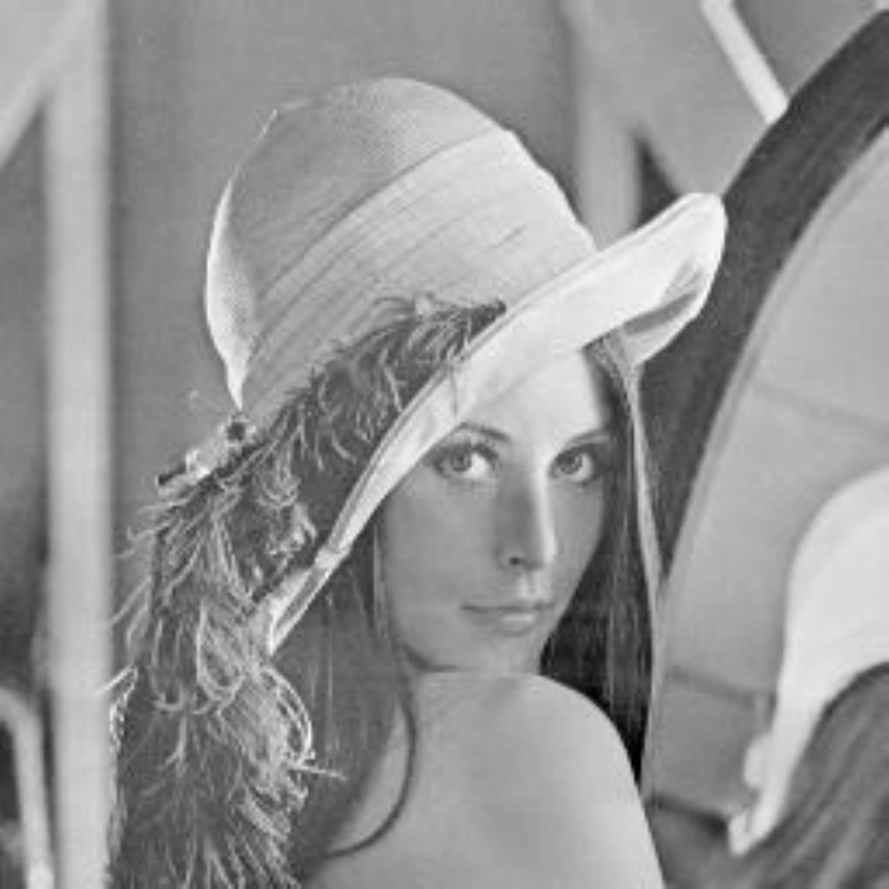}\\
  \hline
  & 6.02dB & 14.29dB & 15.41dB \\
  \textbf{0.5}&
  \includegraphics[width=0.15\textwidth]{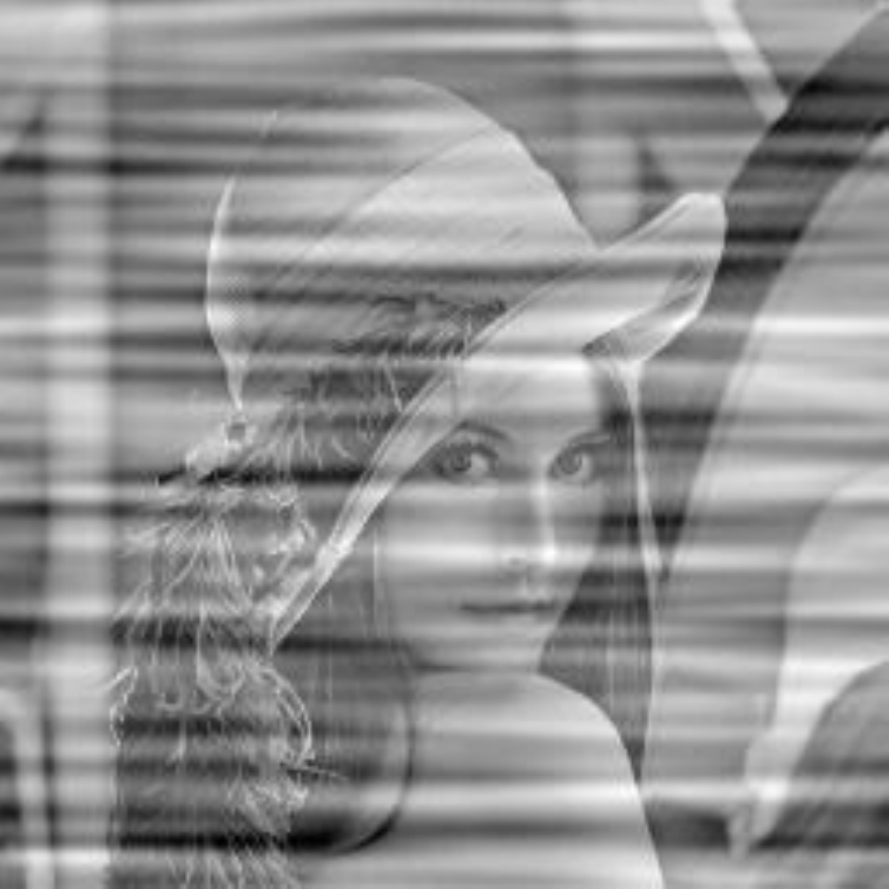}&
  \includegraphics[width=0.15\textwidth]{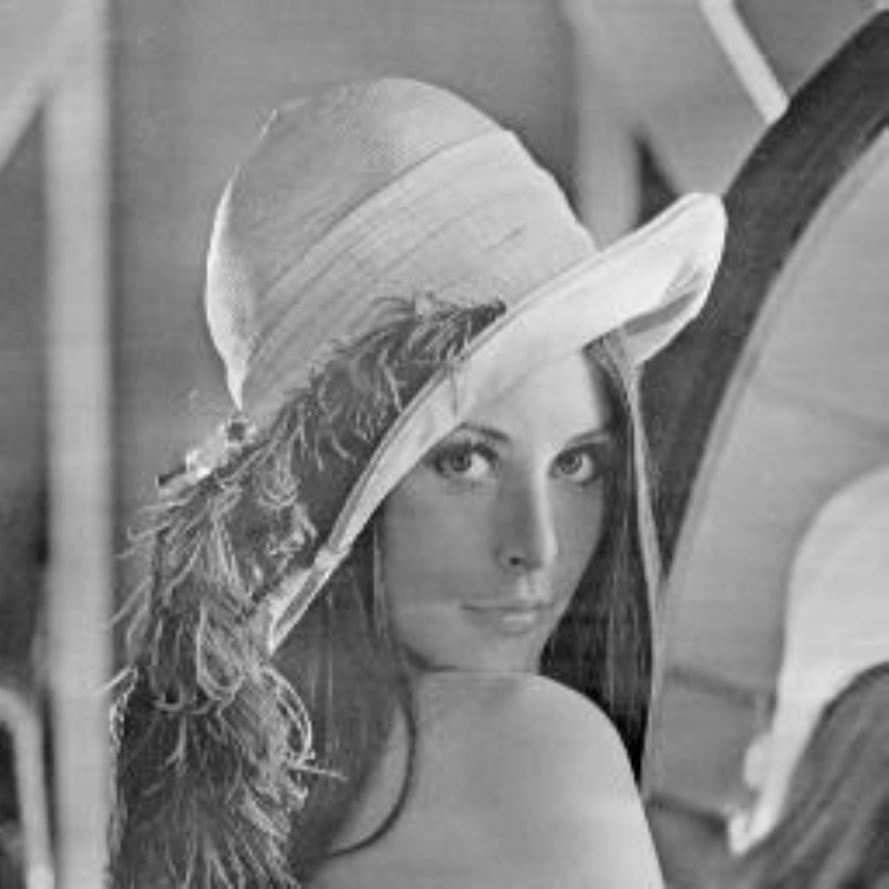}&
  \includegraphics[width=0.15\textwidth]{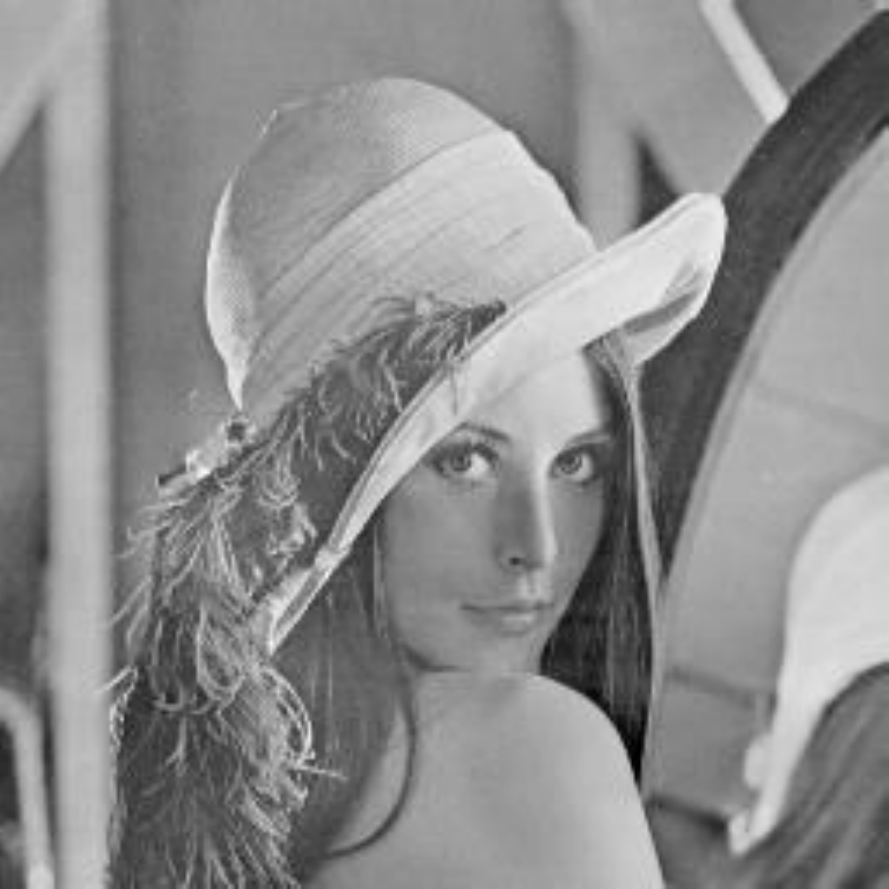}\\
  \hline
  &6.02dB & 17.67dB & 18.15dB \\
  \textbf{1}&
  \includegraphics[width=0.15\textwidth]{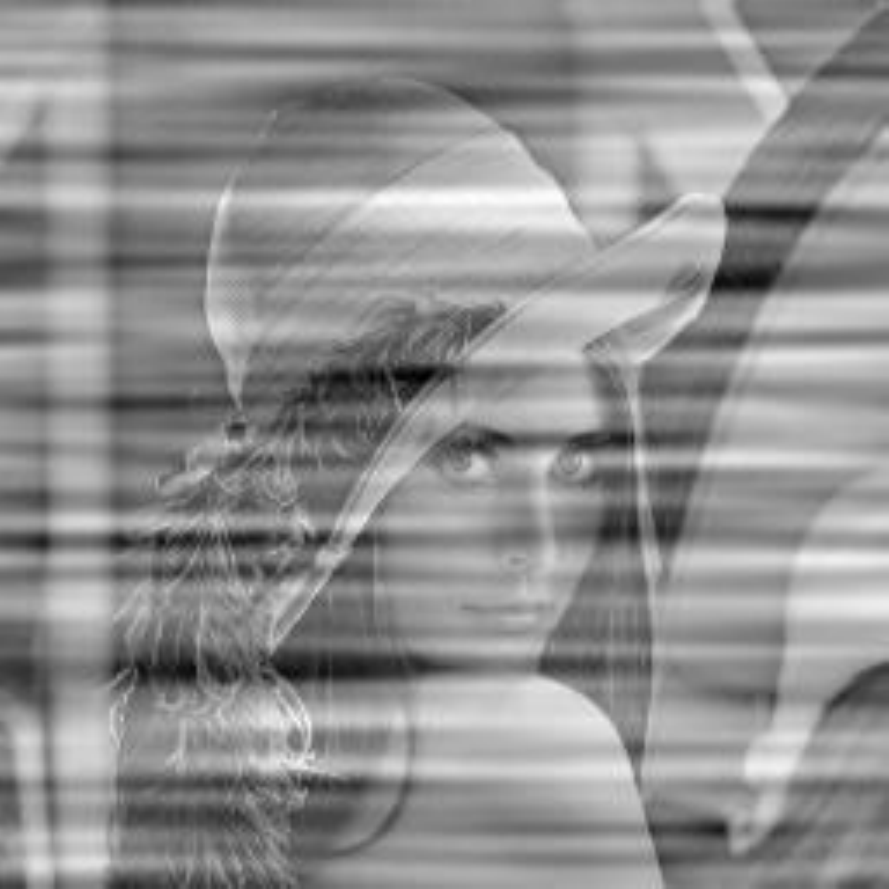}&
  \includegraphics[width=0.15\textwidth]{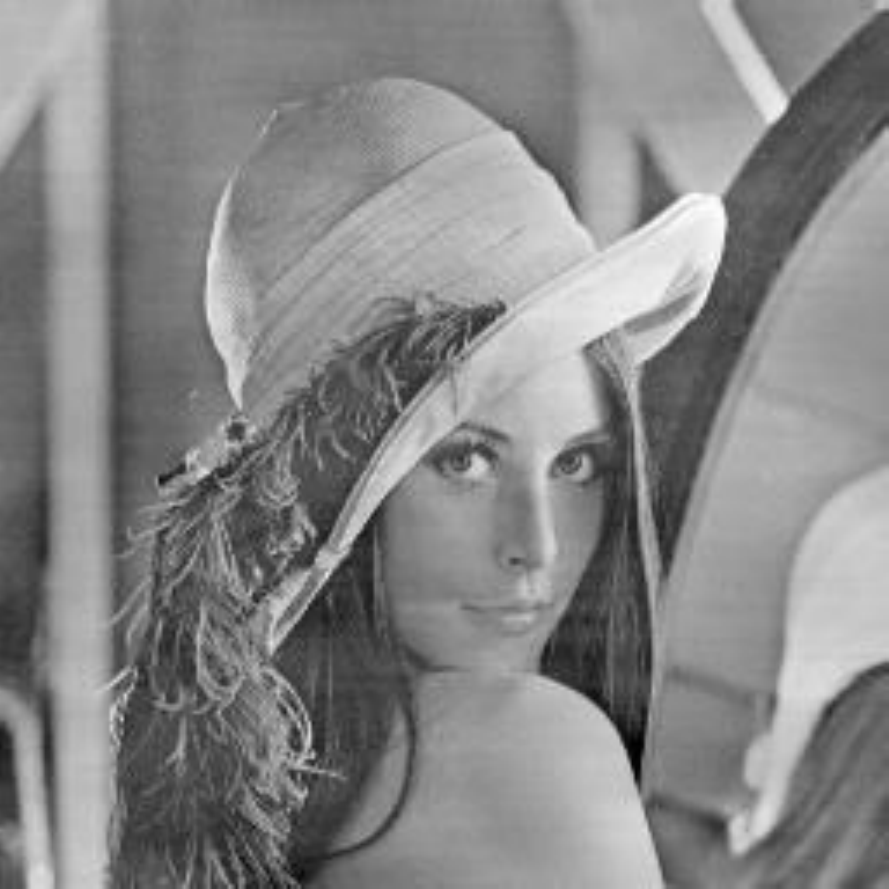}&
  \includegraphics[width=0.15\textwidth]{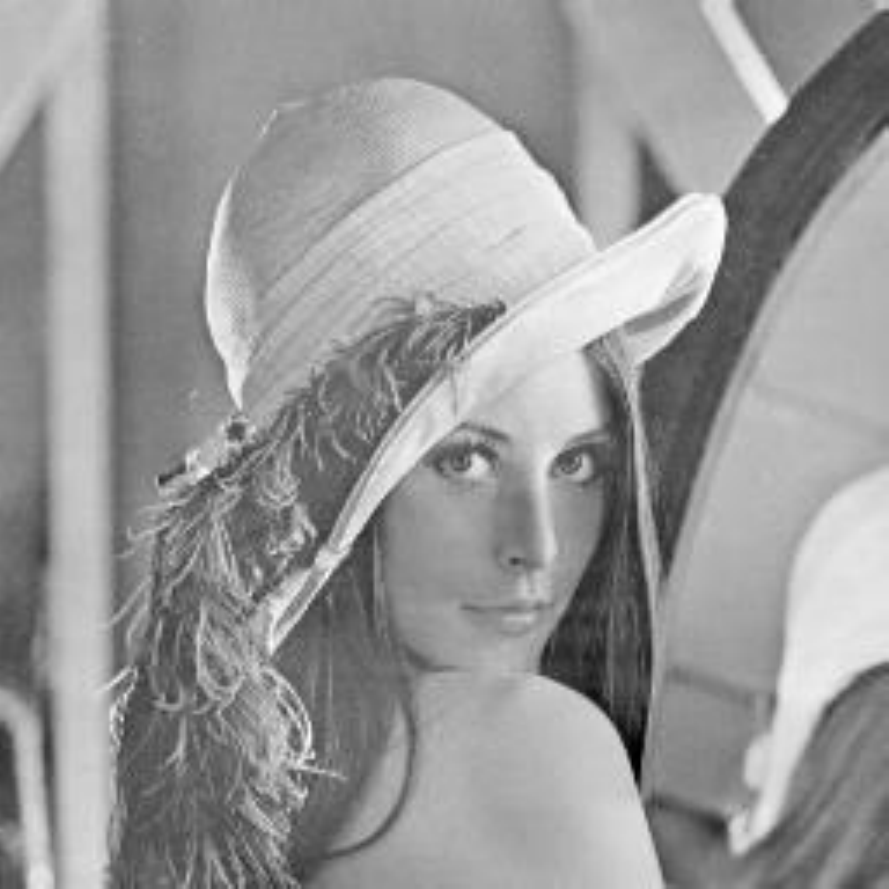}\\
  \hline
  \end{tabular}
\caption{Denoising results with the resolution of an $TV-l^1$ or $TV-l^2$ problem. From top to bottom: increasing value of $\gamma$. 
Left: noisy images. Center: denoised using an $l^1$ prior. Right: denoised using an $l^2$ prior.}\label{fig:denoisingstationary}
\end{figure}

%% file: Sec_ParameterEstimation.tex
\section{Primal-dual estimation in the $l^2$-case}
\label{sec:parameterselection}

Motivated by the results presented in the previous section, we focus on the $l^1-l^2$ problem \eqref{eq:pb}.
Since the mapping $\llambda \mapsto \sum_{i=1}^m\frac{\alpha_i}{2} \|\lambda_i\|_2^2$ is stricly convex, this problem admits a unique minimizer.

In this section, we aim at proposing an automatic estimation of an adequate value of $\alpha=(\alpha_1,\cdots,\alpha_m)$. 
A natural choice for the regularization parameter $\alpha$  (also known as Morozov's discrepancy principle \cite{morozov1966solution}) is to ensure that
\begin{equation} 
\label{eq:goal}
\| \psi_i \star \lambda_i(\alpha) \|= \|b_i\|
\end{equation}
for a given norm $\|\cdot\|$. In practice, $\|b_i\|$ is usually unknown, but the user usually has an idea of the noise level and can set $\|b_i\|\simeq \eta_i\|u_0\|$ where $\eta_i\in]0,1[$ denotes a noise fraction. 

In the rest of this section, we provide estimates for $\|b(\alpha)\|_2$, in the case $m=1$ in paragraph \ref{ssec:m1} and in the general case in paragraph \ref{ssec:mfilters}. When the parameters $(\alpha_i)_{i\in \{1,\hdots,m\}}$ are given, the filter with $m$ filters is equivalent to a related problem with $1$ filter. The link is detailed in paragraph \ref{ssec:m1m}. Finally paragraph \ref{ssec:algorithm} shows how the proposed results can be used in a practical algorithm. The proofs are provided in the appendix.

\subsection{Results for the case $m=1$ filter}
\label{ssec:m1}

We first state our results in the particular case of $m=1$ filter in order to clarify the exposition. 
We obtain several bounds on the $l^2$-norm of the noise $\|b(\alpha)\|_2$ that are valid for different values of $\alpha$. 
The following theorem stated for $m=1$ filter is a particular case of the results presented in paragraph \ref{ssec:mfilters}.
\begin{theorem}
\label{thm:central1}
Let $\alpha>0$ and denote $h_{k}=\psi\star \tilde \psi \star \tilde d_k$ for $k\in \{1,\hdots,d\}$.
Then 
\begin{equation}
\label{eq:upbound1}
\|b(\alpha)\|_2\leq \frac{\sqrt{n}}{\alpha}\max_{k\in \{1,\hdots,d\}}\| \hat h_k \|_\infty.
\end{equation}
If we further assume that $\hat \psi$ does not vanish there exists a value $\bar \alpha>0$ such that $\forall \alpha\in (0,\bar \alpha], \ b(\alpha) = u_0-u_0^{mean}$. 
\end{theorem}

This theorem states that the norm of $b$ is bounded by a decaying function of $\alpha$. Moreover $\lim_{\alpha\rightarrow 0^+} \|b(\alpha)\|_2=\|u_0-u_0^{mean}\|_2$, and for sufficiently small values of $\alpha$ the solution is independent of $\alpha$ and known in closed form. 
Note that $\alpha \mapsto \|b(\alpha)\|_2$ is not necessarily monotonically decreasing. The quantity $\|u_0-u_0^{mean}\|_2$ which is an upper bound in a neighborhood of $0$ is not necessarily an upper bound for all $\alpha >0$. In our numerical tests, we never encountered a situation where $\|b(\alpha)\|_2>\|u_0-u_0^{mean}\|_2$. In the following, we make the abuse to refer to $\displaystyle \min(\frac{\sqrt{n}}{\alpha}\max_{k\in \{1,\hdots,d\}}\| \hat h_k \|_\infty,\|u_0-u_0^{mean}\|_2)$ as an ``upper bound''.
As will be observed in the numerical experiments in section \ref{sec:numericalXP}, the bound 
$\displaystyle\|b(\alpha)\|_2\leq \frac{\sqrt{n}}{\alpha} \max_{k\in \{1,\hdots,d\}} \| \hat h_k \|_\infty$ provided in Theorem \ref{thm:central1} is quite accurate and sufficient for supervised parameter selection.
The following proposition provides a lower bound with the same asymptotic decay rate in $\frac{1}{\alpha}$ for $\|b(\alpha)\|_2$. 

\begin{proposition}
\label{eq:propminor}
Assume that $\hat \psi$ does not vanish.
Let $b(\alpha)=\psi\star \lambda(\alpha)$ where $\lambda(\alpha)$ is the solution of \eqref{eq:pbl2}.
Let $P_1$ denote the orthogonal projector on $\Ran(\Psi^T\nabla^T)$ and $b_1=P_1(\Psi^{-1}u_0)$. Then if $\alpha$ is sufficiently large,
\begin{equation*}
\|b(\alpha)\|_2\geq \frac{1}{\alpha} \frac{1}{ \|\Psi^{-1}\|_{2\rightarrow 2} }\frac{\|b_1\|_2}{\|A^+b_1\|_\infty}.
\end{equation*}
\end{proposition}

\subsection{Results for the general case $m\geq 1$ filters}
\label{ssec:mfilters}

In this paragraph, we state results that generalize Theorem \ref{thm:central1} to the case of $m\geq 1$ filters.

%(*suggestion de r\'eorganisation : mettre le majorant en 1/$\alpha$ dans le 2er th et les 2 majorations par des constantes dans le 2e th.*)
\begin{theorem}
\label{thm:central}
Let $\alpha=(\alpha_1,\cdots,\alpha_m)$ denote positive weights.
%$\hh_i= \begin{pmatrix} h_{i,1} \\ \vdots \\ h_{i,d}\end{pmatrix}$ with
Let $h_{i,k}=\psi_i\star \tilde \psi_i \star \tilde d_k$ for $k\in \{1,\hdots,d\}$.
Then 
\begin{equation}
\label{eq:upbound}
\|b_i(\alpha)\|_2\leq \frac{\sqrt{n}}{\alpha_i} \max_{i\in\{1,\hdots,m\}} \max_{k\in \{1,\hdots,d\}}\| \hat h_{i,k} \|_\infty.
\end{equation}
\end{theorem}

\begin{theorem}
\label{prop:solutionpetitalpha}
Denote $\Psi=(\Psi_1,\Psi_2,\hdots, \Psi_m)\in R^{n\times nm}$ and assume that $\Psi^T\Psi$ has full rank (this is equivalent to the fact that $\forall \xi$, $\exists i \in \{1,\hdots,m\}$, $\hat \psi_i(\xi)\neq 0$).
Let $\hat \llambda^0(\alpha)=(\hat \lambda_1^0,\hdots,\hat \lambda_m^0)$ be defined by:
\begin{equation}
\label{eq:deflambda0_1}
\hat \lambda_i^0(\alpha)(\xi)= 
\left\{
\begin{array}{ll}
0 & \textrm{if \ }  \xi=0, \\
\frac{\bar{\hat{\psi_i}}(\xi) \hat{u_0}(\xi)}{\alpha_i \sum_{j=1}^m \frac{|\hat \psi_j(\xi)|^2}{\alpha_j}} & \textrm{otherwise}.
\end{array}\right.
\end{equation}
Then there exists a value $\bar \alpha>0$ such that for all $\alpha\in ]0,\bar \alpha]^m$ the solution $\llambda(\alpha)$ of problem \eqref{eq:pb} is:
\begin{equation}
\llambda(\alpha)=\llambda^0(\alpha).
\end{equation}
%Moreover  for all $\alpha$,
%\begin{equation}
%\label{eq:upbound3}
%\|b_i(\alpha)\|_2\leq \|\hat \psi_i\|_\infty \|\hat u_0 \oslash \hat \psi_i\|_2
%\end{equation}
%with the convention that $\left(\hat u_0 \oslash \hat \psi_i\right)(\xi)=0$ if $\hat \psi_i(\xi)=0$.
\end{theorem}

Theorems \ref{thm:central} and \ref{prop:solutionpetitalpha} generalize Theorem \ref{thm:central1}. 
In practice, we observed that the ratio $$\frac{\frac{\sqrt{n}}{\alpha_i} \max_{i\in\{1,\hdots,m\}} \max_{k\in \{1,\hdots,d\}}\| \hat h_{i,k} \|_\infty}{\|b_i(\alpha)\|_2}$$ does not exceed limited values of the order of $5$ (see the bottom row of Figure \ref{fig:comparisons}). This gives an idea of the sharpness of \eqref{eq:upbound}. 
The bound \eqref{eq:upbound} can thus be used to provide the user warm start parameters $\alpha_i$. 
This idea is detailed in the algorithm presented in section \ref{ssec:algorithm}.

\subsection{Equivalence with a single filter model}
\label{ssec:m1m}

In section \ref{sec:marginalselection}, we showed that the following image formation model is rich enough for many applications of interest:
\begin{equation}
 u_0 = u + \sum_{i=1}^m \lambda_i\star \psi_i
\end{equation}
where $\lambda_i$ is the realization of a \textit{Gaussian random vector} of distribution $\mathcal{N}(0,\sigma_i^2 I)$. 
Let $b=\sum_{i=1}^m \lambda_i\star \psi_i$.
An important observation is that the previous model is equivalent to the following:
\begin{equation}
 u_0 = u + \lambda\star \psi,
\end{equation}
where $\lambda$ is the realization of a Gaussian random vector $\mathcal{N}(0,\sigma^2 I)$ and $\sigma$ and $\psi$ satisfy:
\begin{equation}
\label{eq:consistency}
 \sigma^2 |\hat \psi(\chi)|^2 = \sum_{i=1}^m \sigma_i^2 |\hat \psi_i(\chi)|^2, \ \forall \chi.
\end{equation}
This condition ensures that both noises have the same covariance matrix $\mathbb{E}(BB^T)$ where $B$ is defined in \eqref{eq:defB}.
In what follows, we set $\alpha = \frac{1}{\sigma^2}$ and  $\alpha_i = \frac{1}{\sigma_i^2}$.

The above remark has a pleasant consequence: problems \eqref{eq:probmfiltres} and \eqref{eq:prob1filtre} below are equivalent from a Bayesian point of view if only the noise component $b=\sum_{i=1}^m \lambda_i\star \psi_i$ and the denoised image $u$ are sought for. 
\begin{equation}
\min_{\lambda\in \R^{n\times m}} \|\nabla (u_0- \sum_{i=1}^m \lambda_i\star \psi_i) \|_1 + \sum_{i=1}^m \frac{\alpha_i}{2} \|\lambda_i\|_2^2.
\label{eq:probmfiltres}
\end{equation}
\begin{equation}
\min_{\lambda\in \R^{n}} \|\nabla (u_0- \lambda \star \psi) \|_1 + \frac{\alpha}{2} \|\lambda\|_2^2.
\label{eq:prob1filtre}
\end{equation}
Hence the optimization can be performed on $\R^n$ instead of $\R^{n\times m}$.
The following result states that this simplification is also justified form a deterministic point of view.

\begin{theorem}
\label{eq:corollambda}
Let $\lambda_i(\alpha)$ denote the minimizer of \eqref{eq:probmfiltres} and $\lambda(\alpha)$ denote the minimizer of \eqref{eq:prob1filtre}.
Let $b_i(\alpha)=\lambda_i(\alpha) \star \psi_i$ and $b(\alpha)=\lambda(\alpha) \star \psi$.
If condition \eqref{eq:consistency} is satisfied, the following equality holds:
\begin{equation}
\label{eq:equalitynoise}
\sum_{i=1}^m b_i(\alpha) = b(\alpha).
\end{equation}
Moreover, the noise components $b_i(\alpha)$ can be retrieved from $b(\alpha)$ using the following formula:
\begin{equation}
\label{eq:lambdaifromlambda}
\hat \lambda_i(\xi)=
\left\{
\begin{array}{ll}
\frac{\bar{\hat{\psi_i}}(\xi) \hat{b}(\xi)}{\alpha_i \sum_{j=1}^m \frac{|\hat \psi_j(\xi)|^2}{\alpha_j}} & \textrm{if \ } \sum_{j=1}^m |\hat \psi_j(\xi)|^2 \neq 0 \\
0 & \textrm{otherwise}. 
\end{array}
\right.
\end{equation}
\end{theorem}

In practice, this theorem  allows to divide the computing times and memory requirements by a factor approximately equal to $m$.

\subsection{Algorithm}
\label{ssec:algorithm}

The following algorithm summarizes how the results presented in this paper allow to design an effective supervised parameter estimation.

\begin{algorithm}[H]
%  \SetLine
  \KwIn{ $u_0\in \R^n$: noisy image. \\
      $(\psi_i)_{i\in \{1,\hdots, m\}} \in \R^{n\times m}$: a set of filters. \\
      $(\eta_1,\hdots, \eta_m)\in [0,1]^m$: noise levels.}
  \KwOut{$u\in \R^n $: denoised image \\
	$(b_i)_{i\in \{1,\hdots, m\}} \in \R^{n\times m}$: noise components (satisfying $\|b_i\|_2\simeq \eta_i \|u_0\|_2$).}
  \Begin{
      Compute $\alpha_i= \frac{\sqrt{n}\|\hat \hh_i\|_\infty}{ \|u_0\|_2 \eta_i}$ (see Proposition \ref{prop:valueofnorm1D2}). \\
      Compute $\hat \psi= \sqrt{\sum_{i=1}^m \frac{\|\hat \psi_i\|^2}{\alpha_i}}$. \\
      Find  $\displaystyle \lambda\in \argmin_{\lambda\in \R^{n}} \|\nabla (u_0- \lambda \star \psi) \|_1 + \frac{1}{2} \|\lambda\|_2^2$ (see \cite{vsnr}).\\
      Compute $u=u_0- \lambda\star \psi$. \\
      Compute $b=\lambda\star \psi$. \\
      Compute $b_i = \lambda_i \star \psi_i$ using Theorem \ref{eq:corollambda}.
      }
\caption{Effective supervised algorithm.}
\end{algorithm}

\subsection{Numerical experiments} 
\label{sec:numericalXP}

The objective of this section is to validate Theorem \ref{thm:central} experimentally and to check that the upper bound in the right-hand side of equation \eqref{eq:upbound} is not too coarse. 
We compute the minimizers of \eqref{eq:pb} using an iterative algorithm for various filters, various images and various values of $\alpha$. 
Then we compare the value $\|b(\alpha)\|_2$ with $\min(\frac{\sqrt{n}\| \hat \hh_i \|_\infty}{\alpha_i},\|u_0-u_0^{mean}\|_2)$. 
As stated in paragraph \ref{ssec:m1}, this quantity is not strictly speaking an upper-bound but we could not find examples of practical interest where $\displaystyle \|b(\alpha)\|_2\geq \min(\frac{\sqrt{n}\| \hat \hh_i \|_\infty}{\alpha_i},\|u_0-u_0^{mean}\|_2)$. As can be seen in the fourth and fifth row of Figure \ref{fig:comparisons}, the upper-bound and the true value follow a similar curve. The fifth row shows the ratio between these values. For the considered filters, the upper bound deviates at most from a factor $4.5$ from the true value. This shows that the upper-bound \eqref{eq:upbound} can provide a good hint on how to choose a correct value of the regularization parameter. The user can then refine this bound easily to get a visually satisfactory result.

\begin{figure}[!htbp]
  \centering
 \begin{tabular}{|c|c|c|}
  \includegraphics[width=0.3\textwidth]{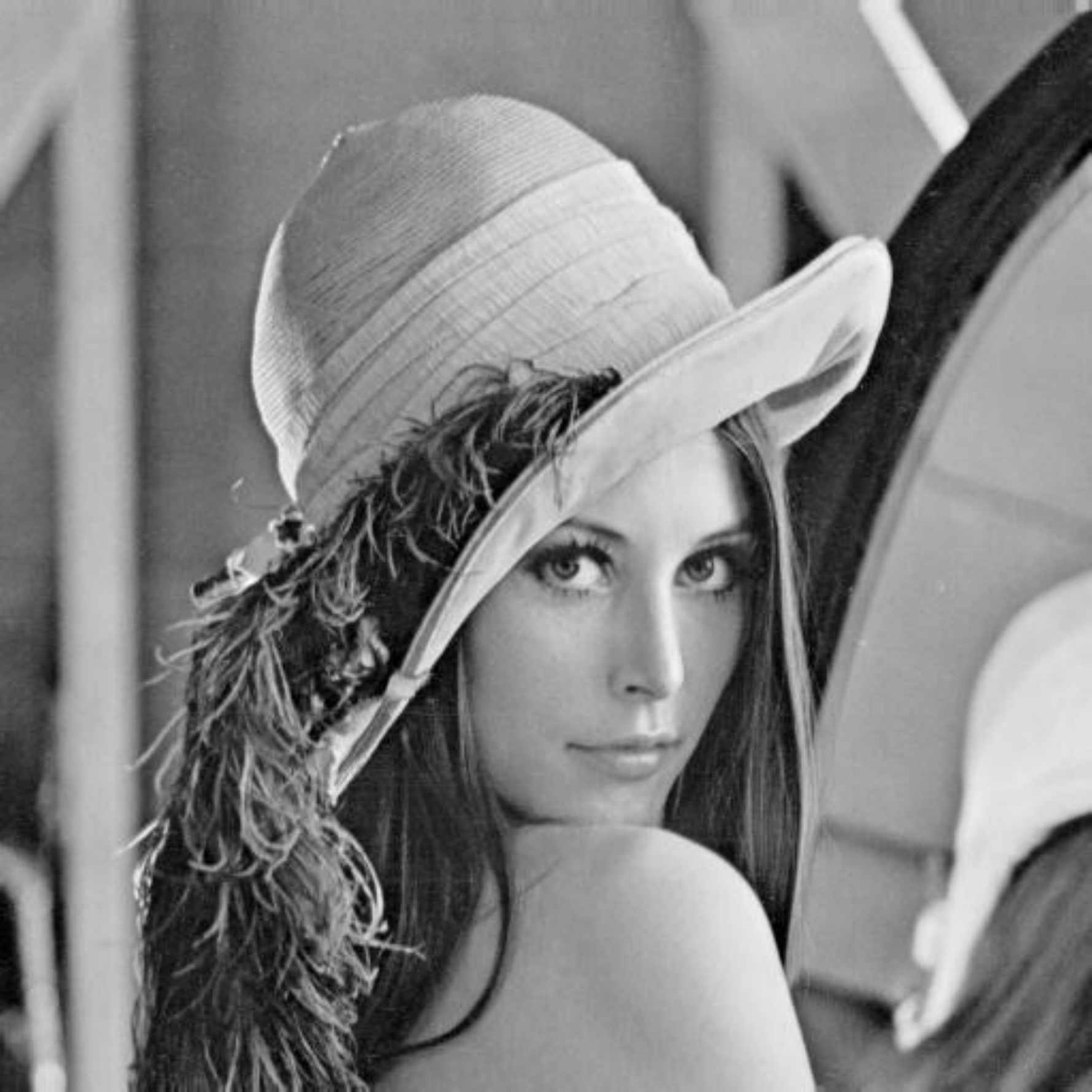}&
  \includegraphics[width=0.3\textwidth]{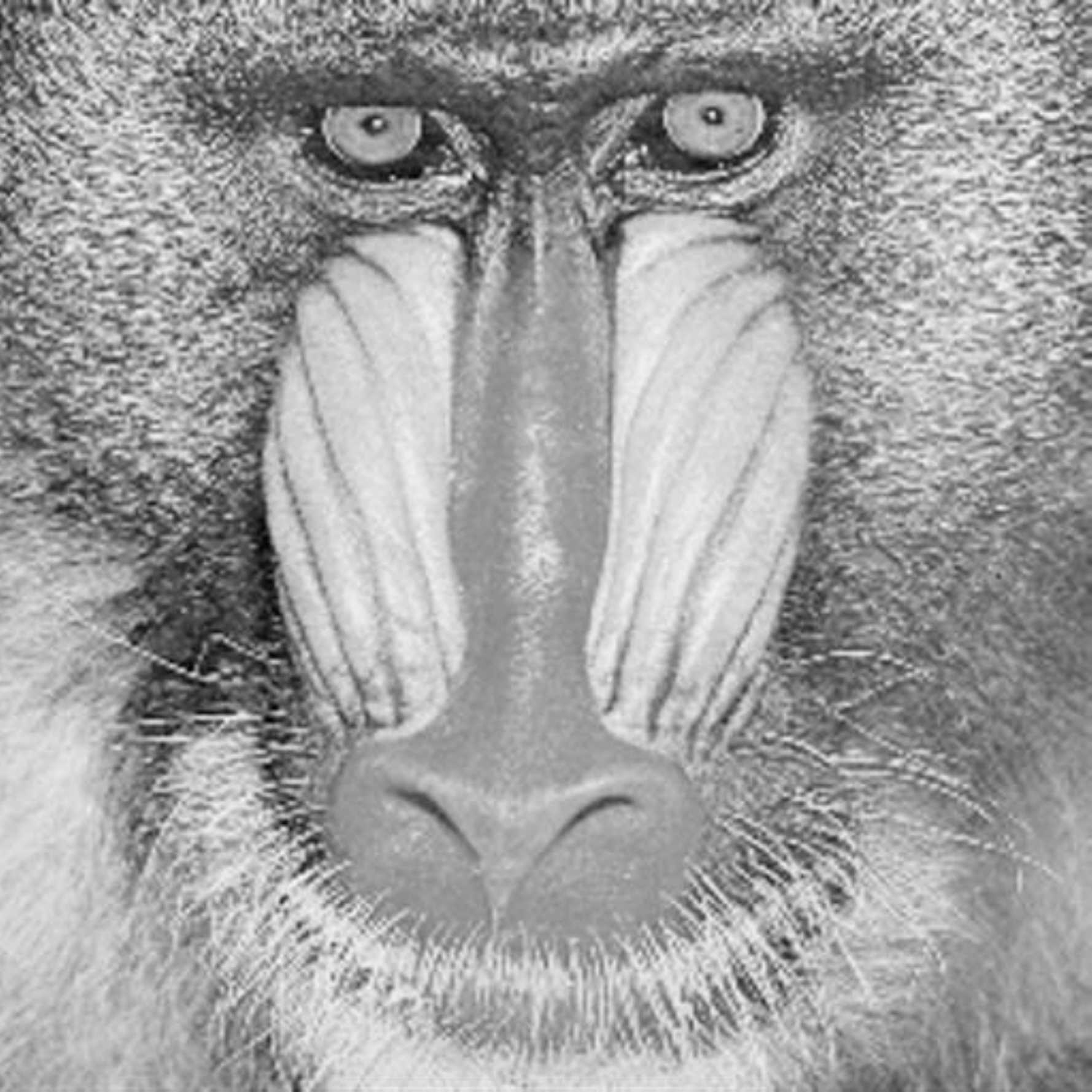}&
  \includegraphics[width=0.3\textwidth]{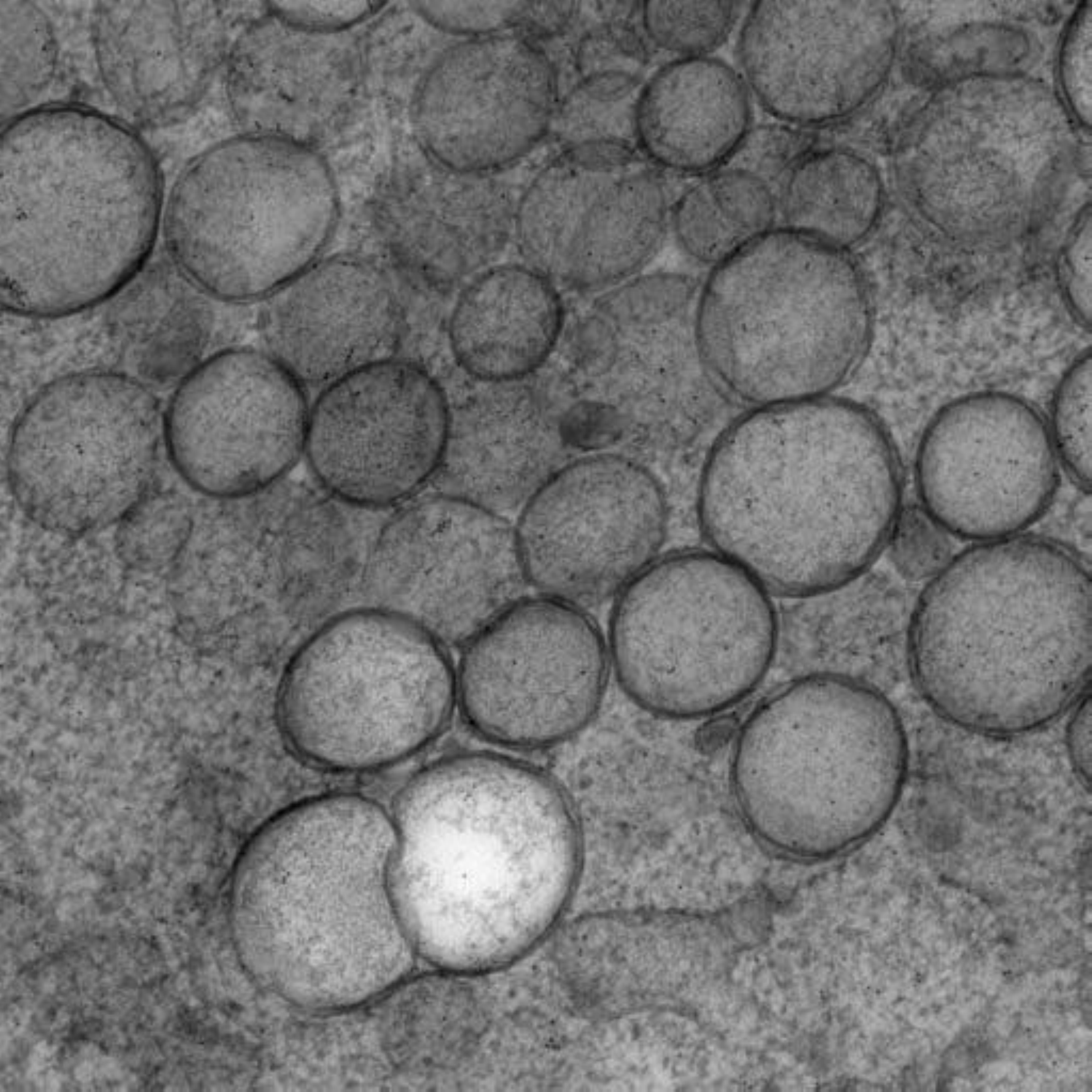}\\
  \hline
  \includegraphics[width=0.3\textwidth]{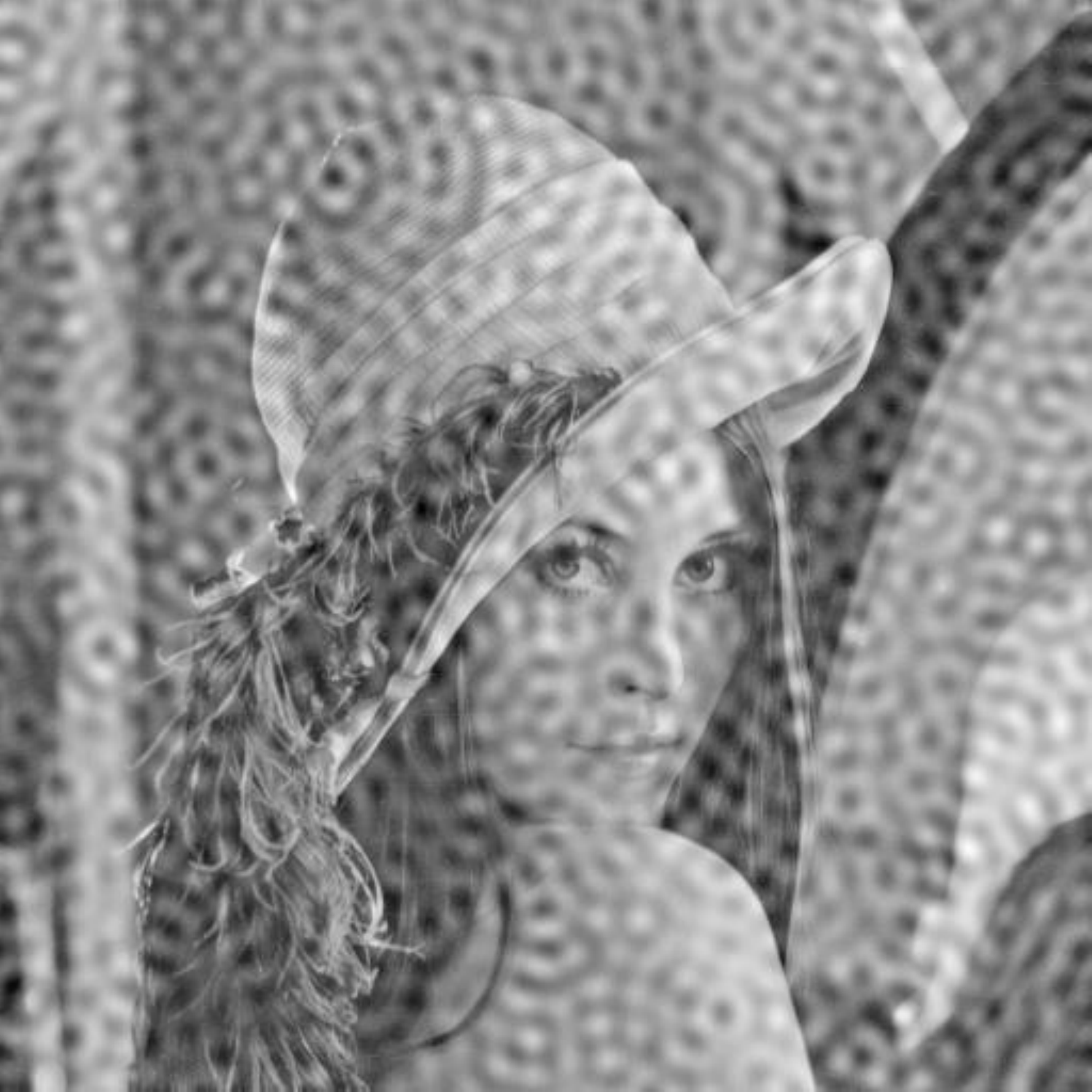}&
  \includegraphics[width=0.3\textwidth]{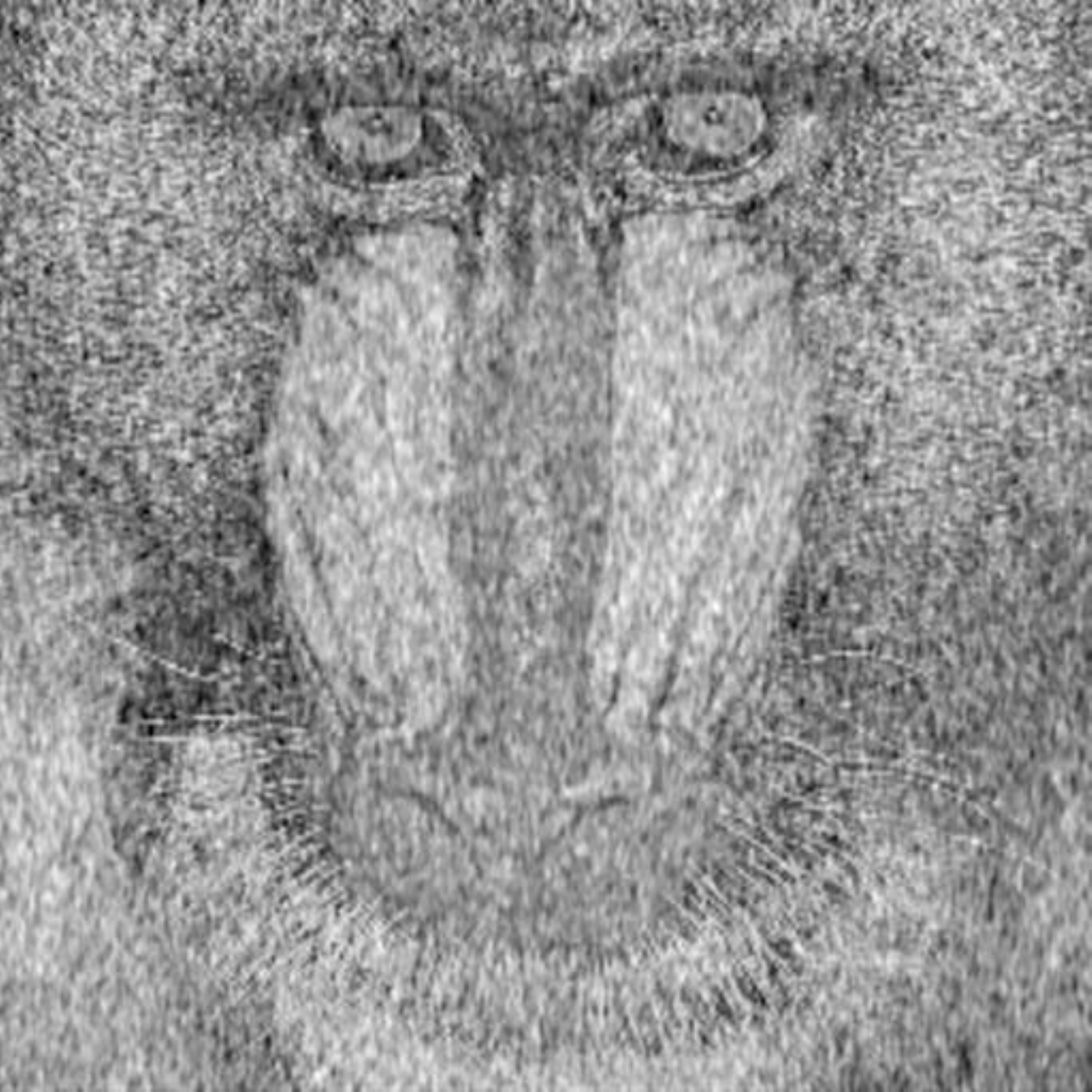}&
  \includegraphics[width=0.3\textwidth]{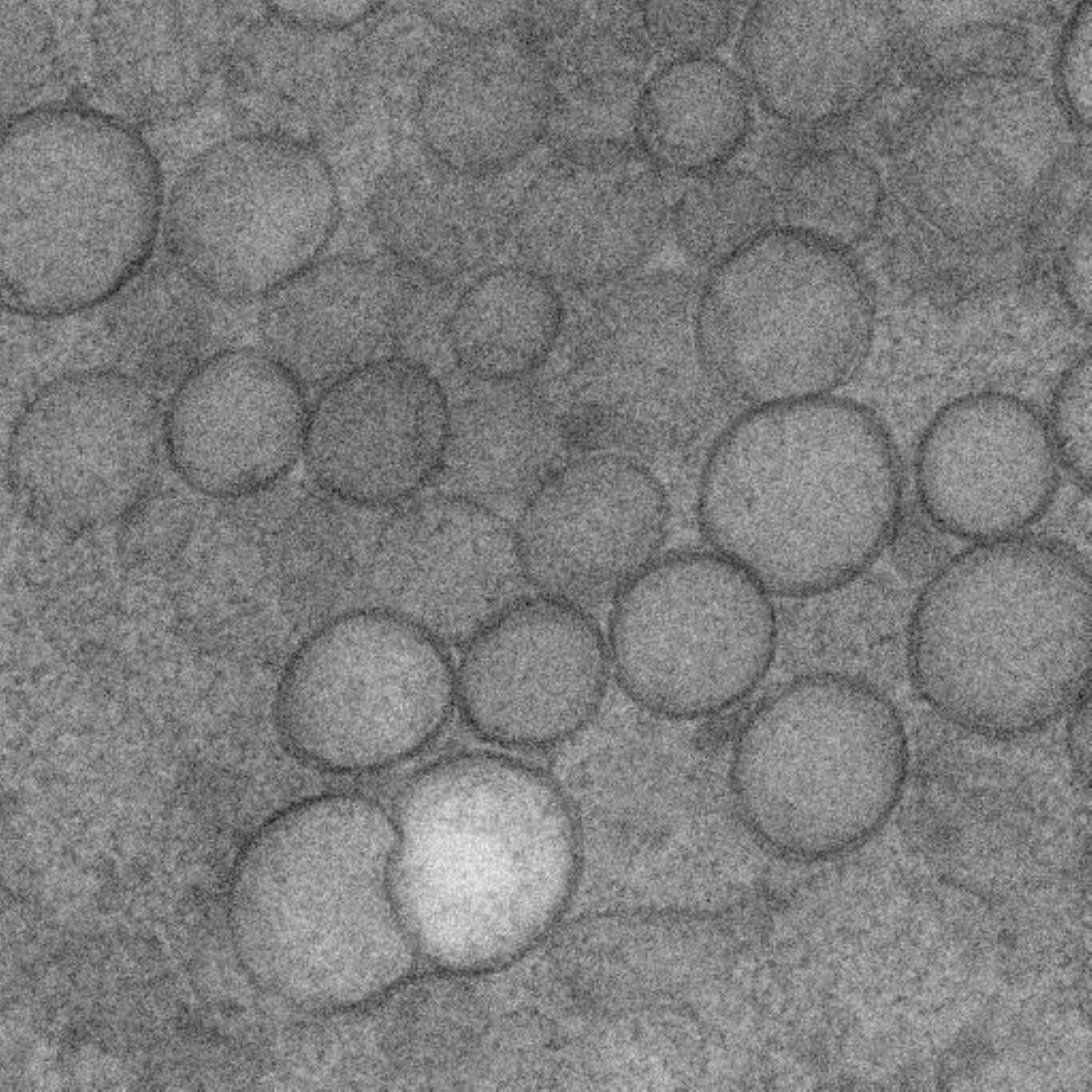}\\
  \hline
  \includegraphics[width=0.3\textwidth]{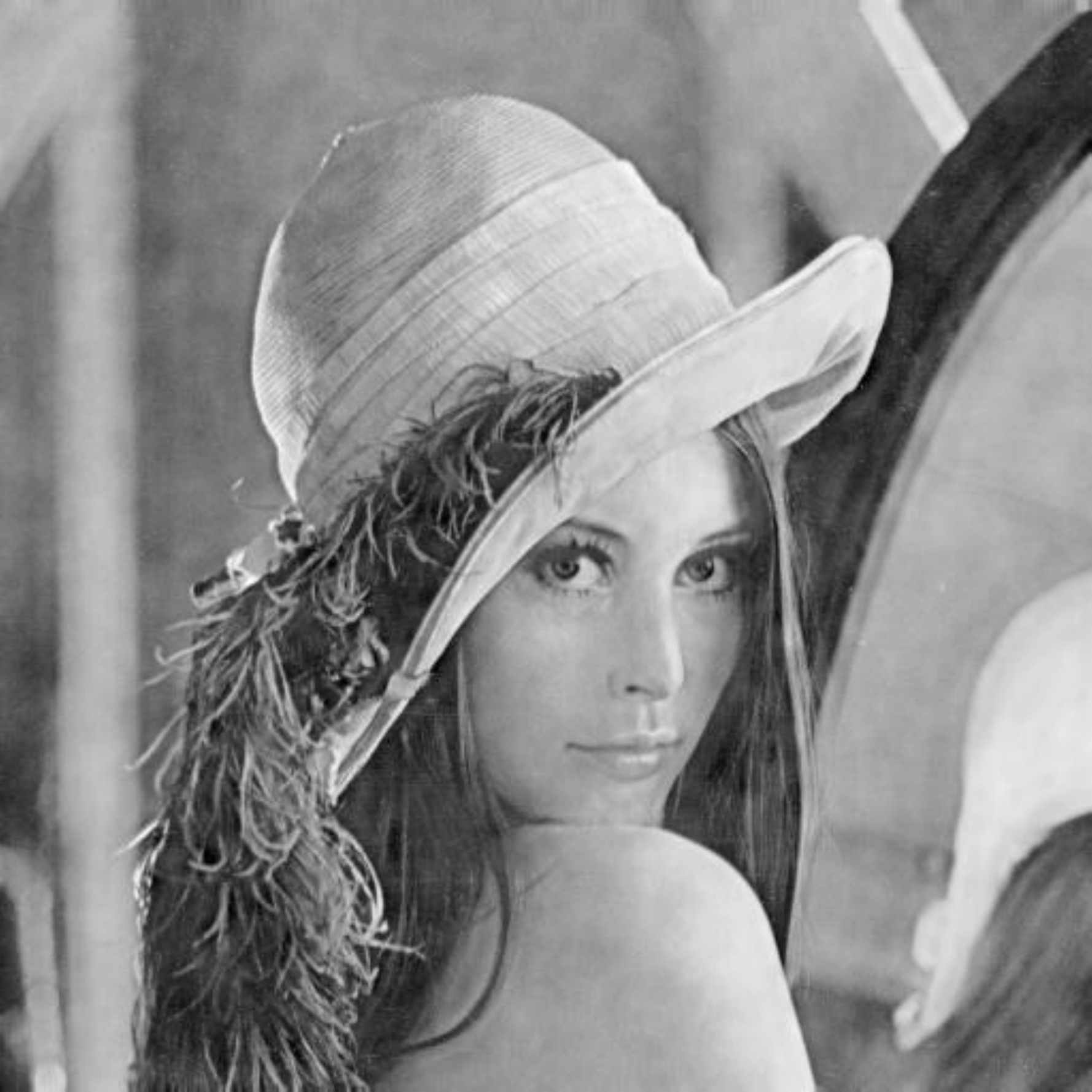}&
  \includegraphics[width=0.3\textwidth]{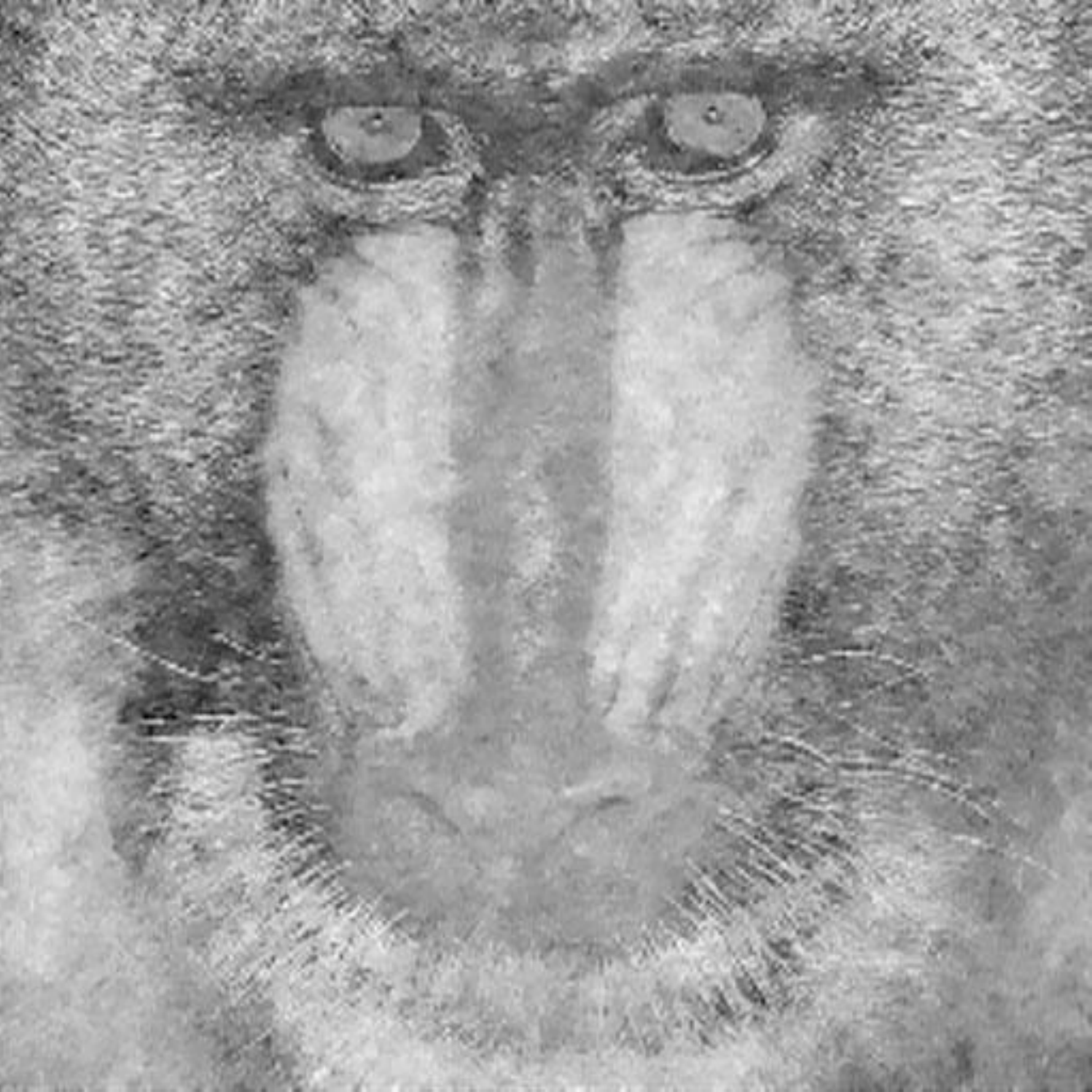}&
  \includegraphics[width=0.3\textwidth]{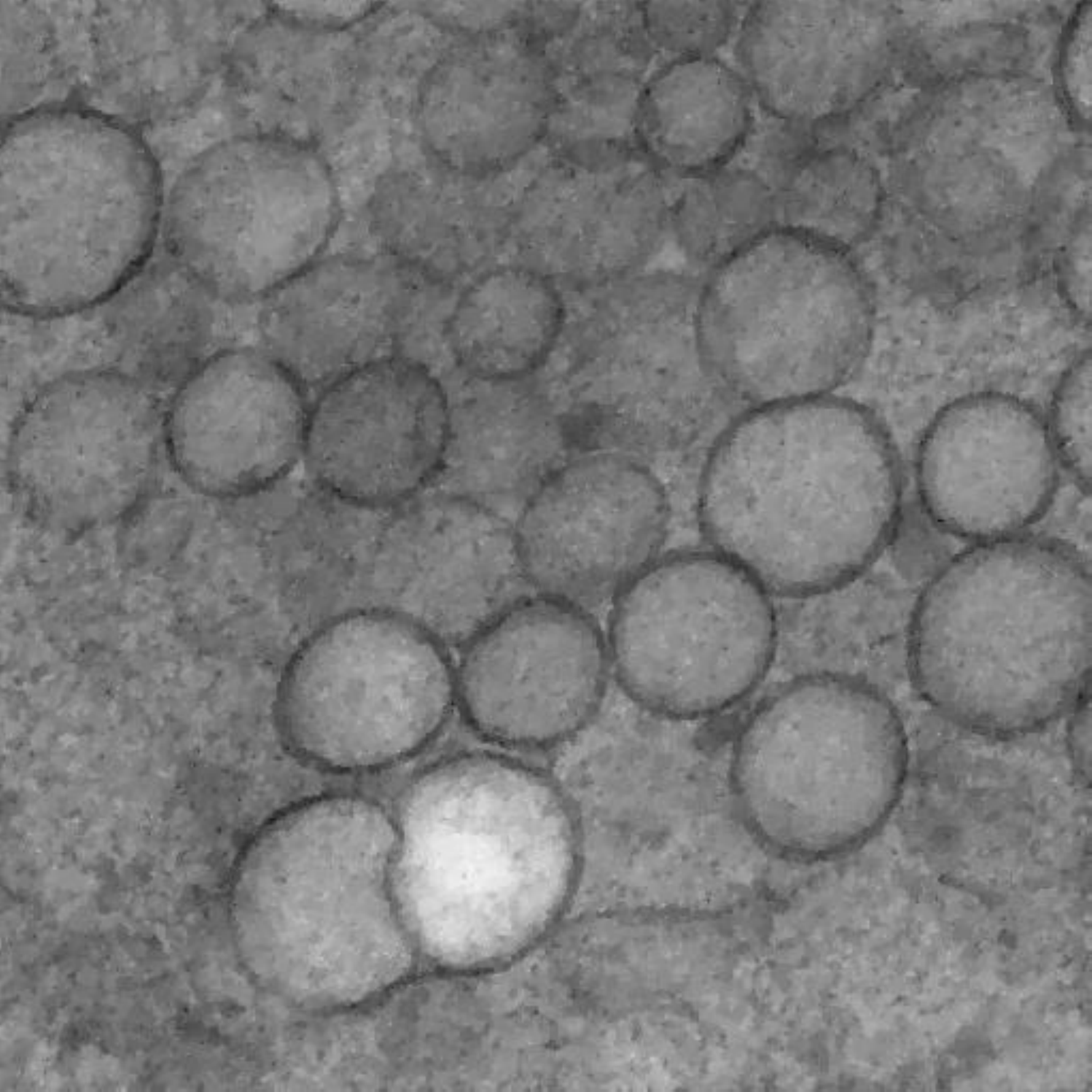}\\
  \hline
  \includegraphics[width=0.32\textwidth]{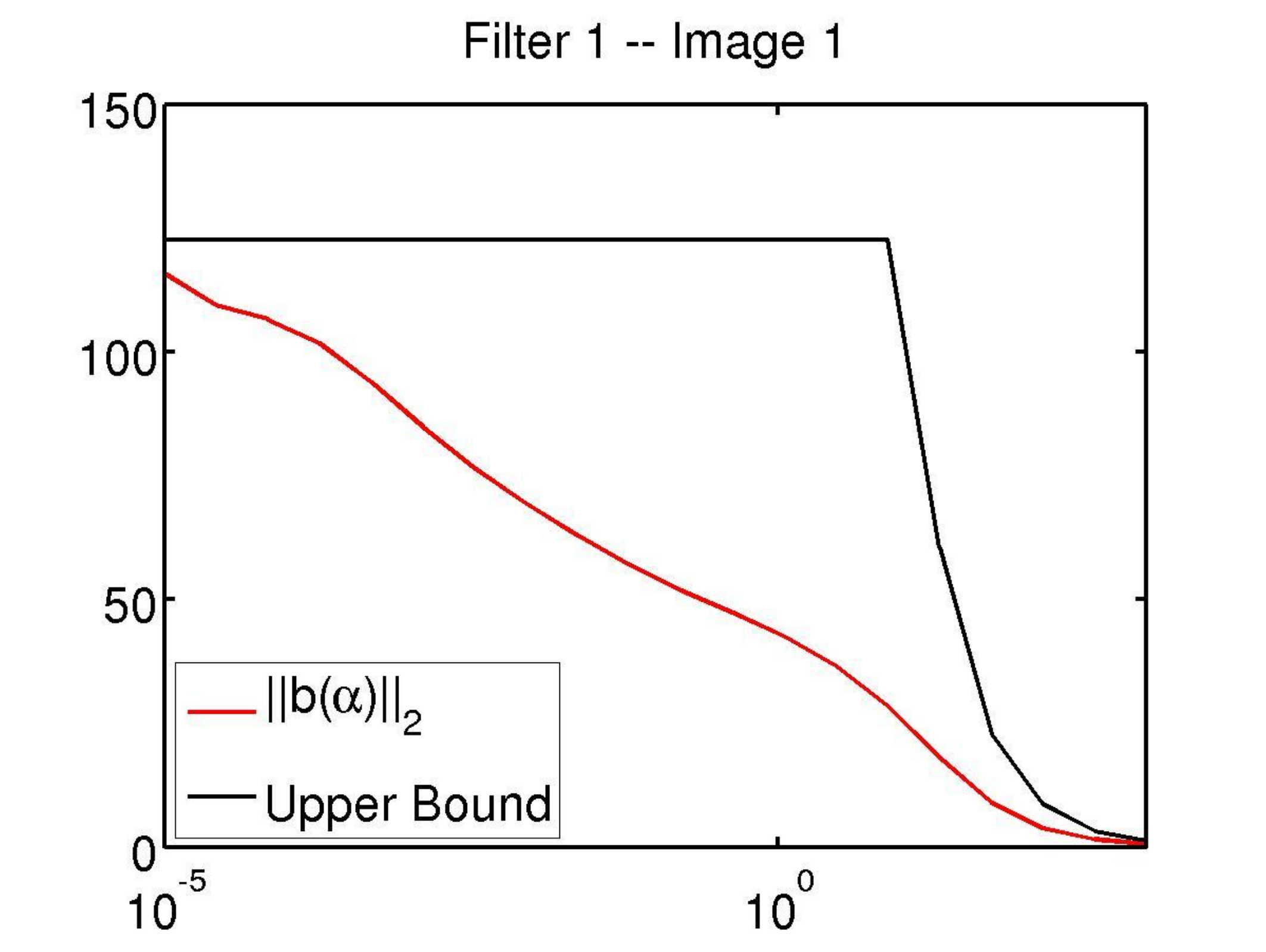}&
  \includegraphics[width=0.32\textwidth]{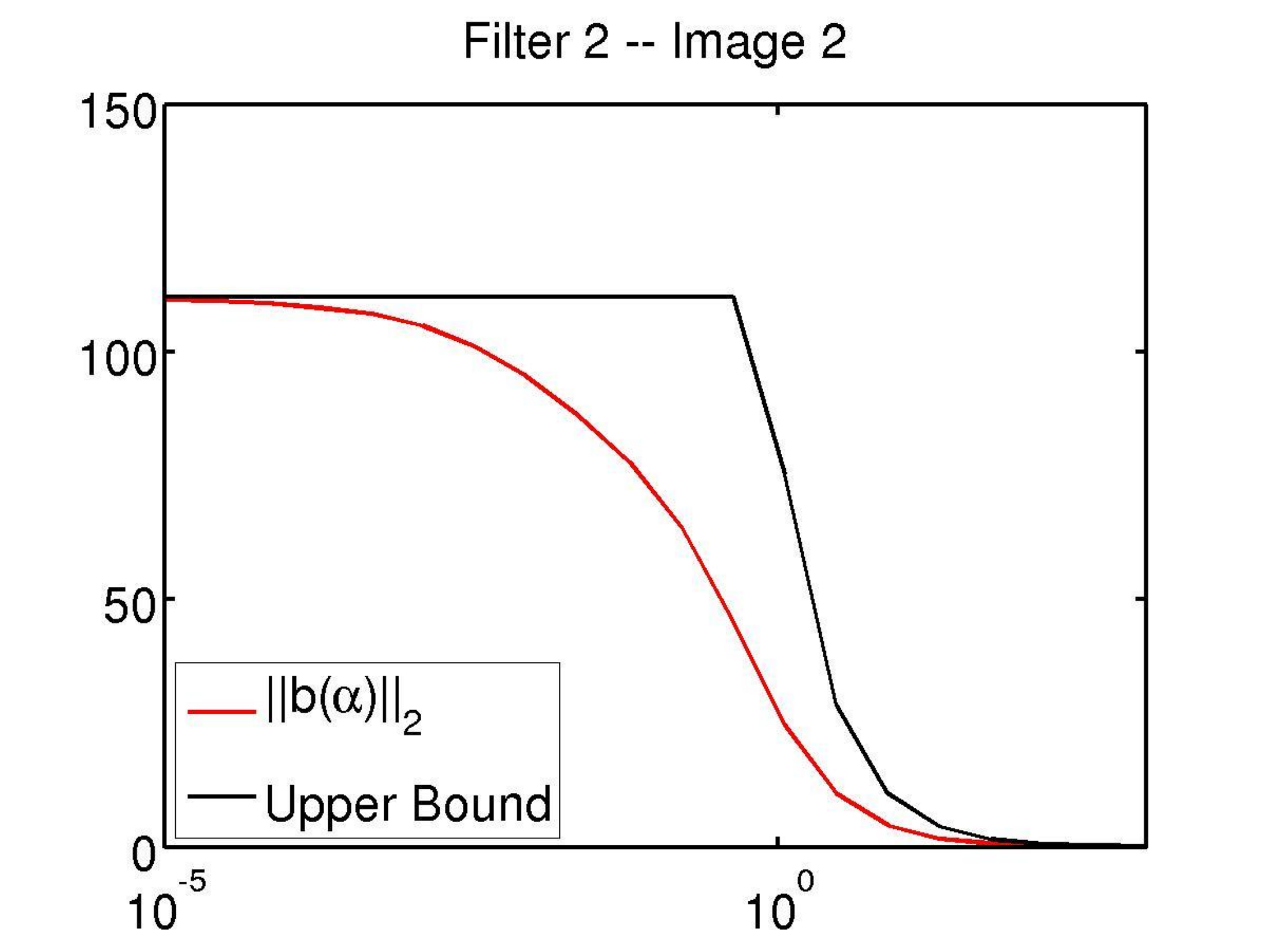}&
  \includegraphics[width=0.32\textwidth]{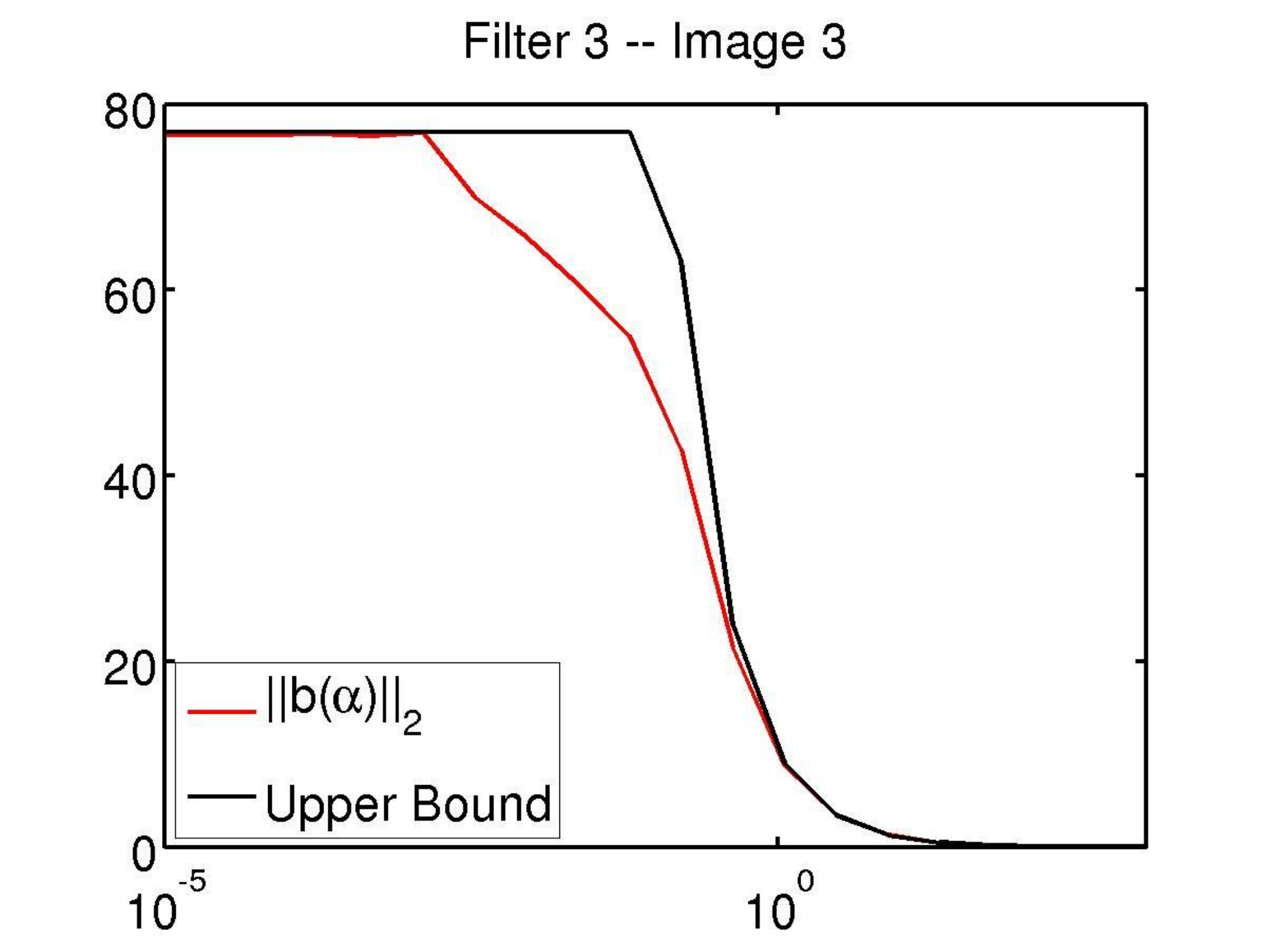}\\
  \hline
  \includegraphics[width=0.32\textwidth]{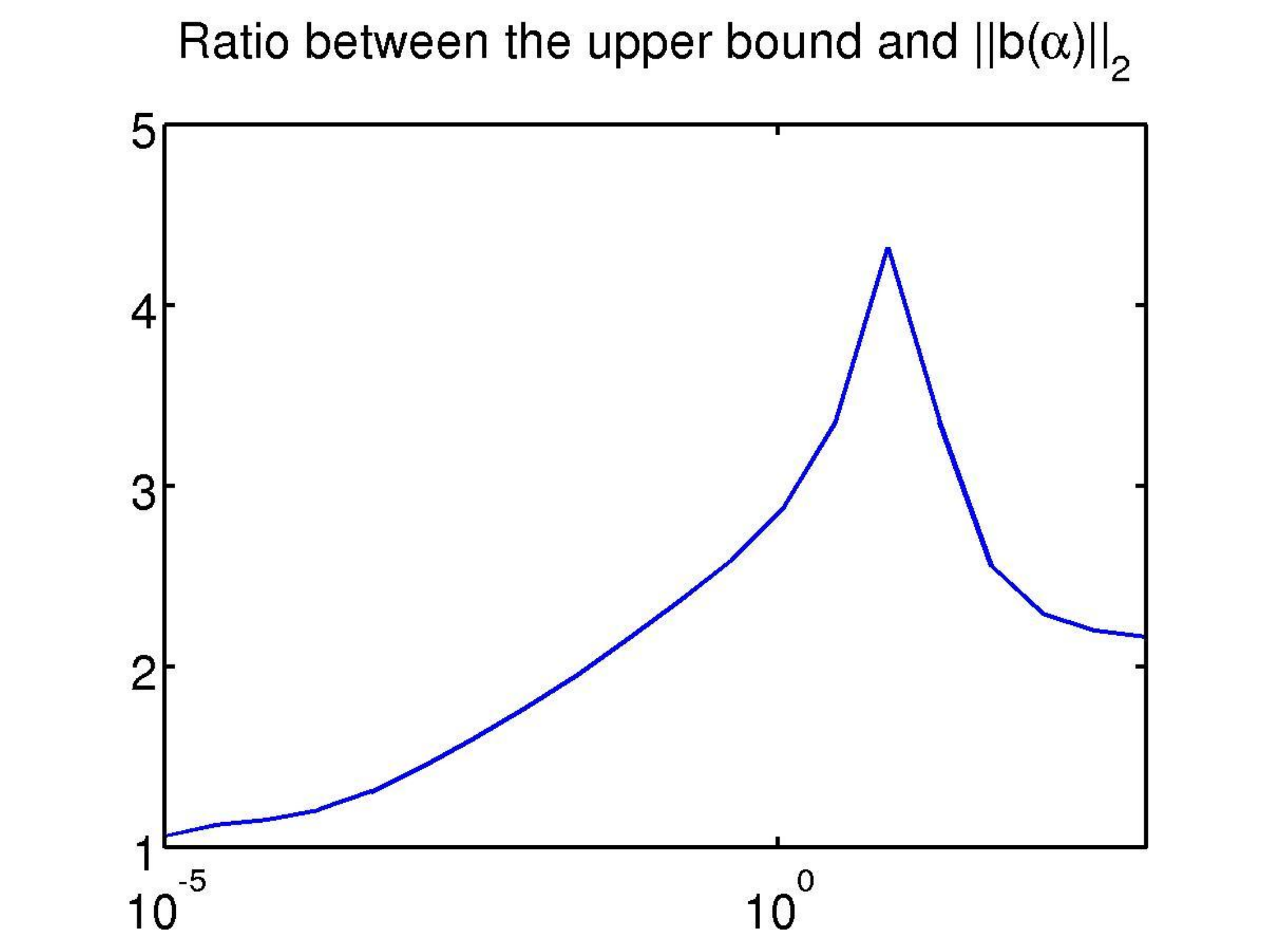}&
  \includegraphics[width=0.32\textwidth]{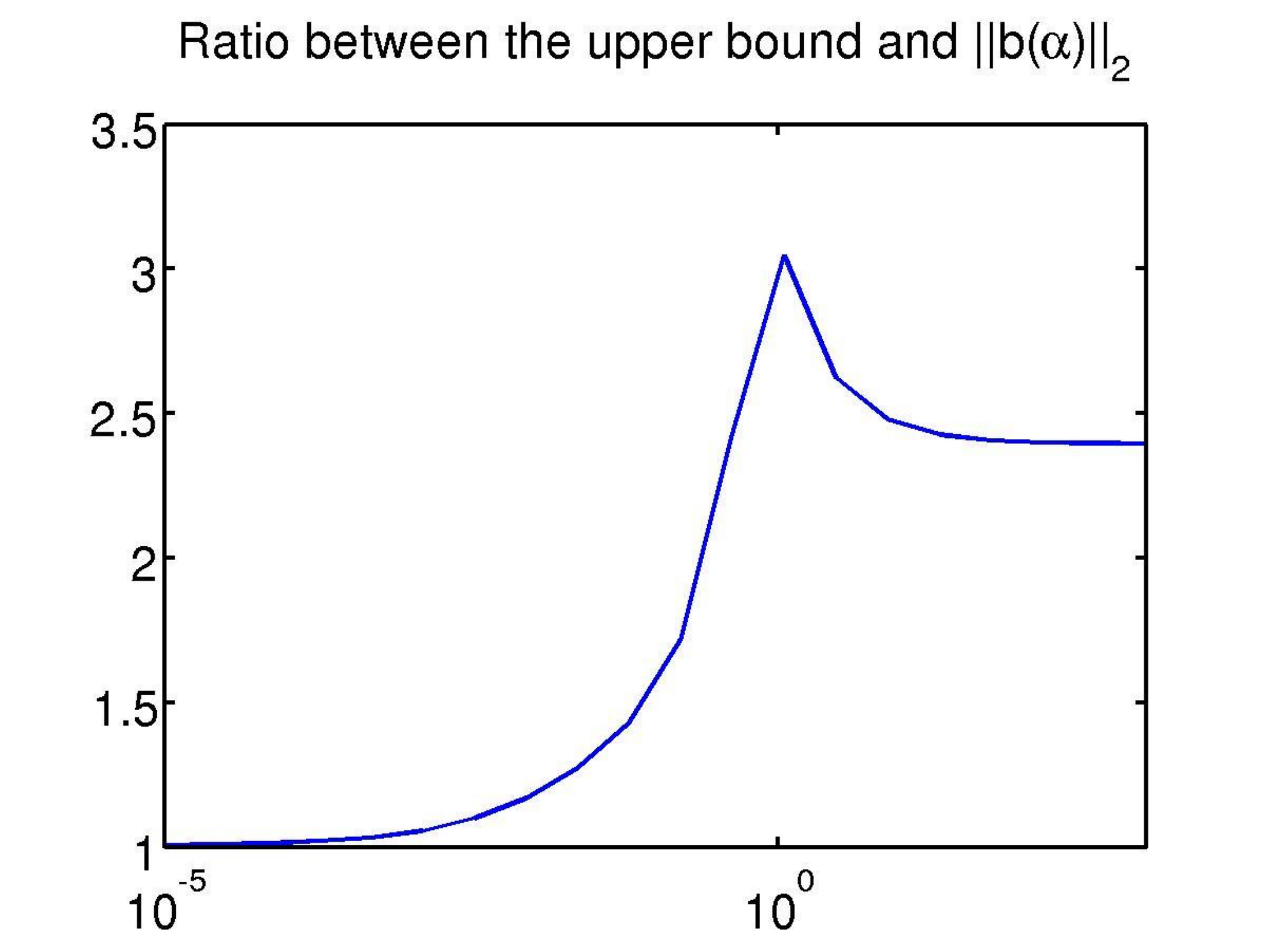}&
  \includegraphics[width=0.32\textwidth]{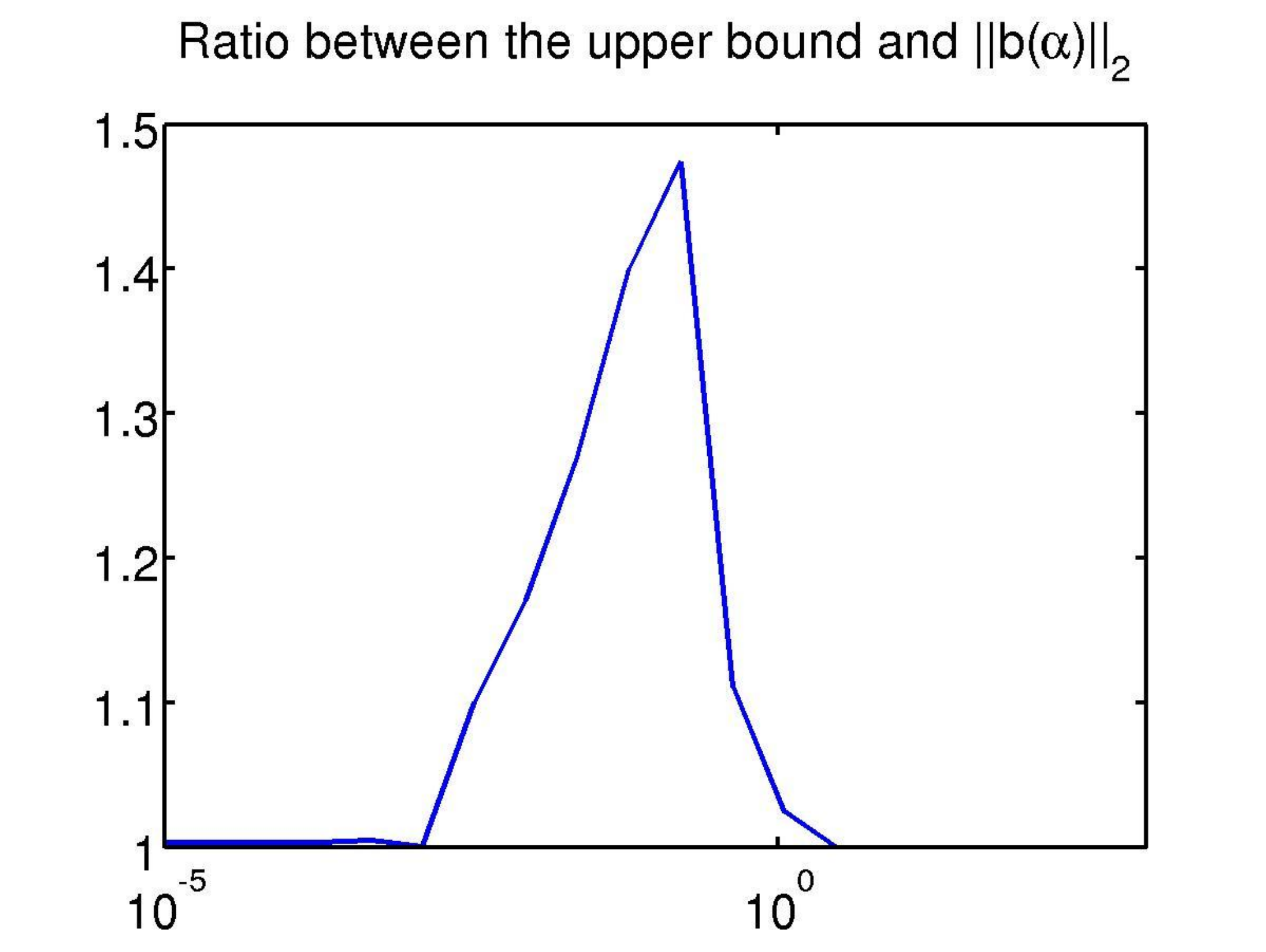}\\
  \hline
  \end{tabular}
\caption{Comparison of the upper bound in equation \eqref{eq:upbound} with $\|b(\alpha)\|_2$. First row: original image. 2nd row: noisy image. 3rd row: denoised using the proposed algorithm. 4th row: comparison of the upper bound \eqref{eq:upbound} with $\|b(\alpha)\|_2$. Last row: ratio between the upper bound and the true value of $\|b(\alpha)\|_2$.\label{fig:comparisons}}
\end{figure}

%% file: Sec_Appendix.tex
\section{Appendix}
\label{sec:appendix}

In this section we provide detailed proofs of the  results presented in section \ref{sec:parameterselection}.

\subsection{Proof of Theorem \ref{thm:central}}

Theorem \ref{thm:central} %(*seulement la premiere majoration*) 
is a direct consequence of Lemma \ref{prop:upperbound}
 and proposition \ref{prop:valueofnorm1D2} below.
%The following lemma is central to justify our estimation procedure:
\begin{lemma}
\label{prop:upperbound}
 Let $\|\cdot\|_N$ denote a norm on $\R^{n}$. The following inequality holds:
\begin{equation}
\|\psi_i \star \lambda_i(\alpha)\|_N  \leq \frac{1}{\alpha_i} \|\Psi_i \Psi_i^T \nabla^T\|_{\iinfty \rightarrow N}.\label{eq:estimatenorm2}
\end{equation}
\end{lemma}

\begin{proof}
Problem \eqref{eq:pb} can be recast as the following saddle-point problem:
\begin{equation*}
\min_{\llambda\in \R^{n\times m}} \max_{\qq\in \R^{n\times d}, \|\qq\|_{\boldsymbol \infty}\leq 1} \langle \nabla (u_0-\sum_{i=1}^m\lambda_i\star \psi_i) ,\qq\rangle + \sum_{i=1}^m\frac{\alpha_i}{2} \|\lambda_i\|_2^2. 
\end{equation*}
The dual problem obtained using Fenchel-Rockafellar duality \cite{Rockafellar} reads:
\begin{equation}
\label{eq:dual}
\max_{\qq\in \R^{n \times d}, \|\qq\|_{\boldsymbol \infty} \leq 1}  \min_{\llambda\in \R^{n\times m}}  \langle \nabla (u_0- \sum_{i=1}^m\lambda_i\star \psi_i) ,\qq\rangle + \sum_{i=1}^m\frac{\alpha_i}{2} \|\lambda_i\|_2^2. 
\end{equation}

Let $\qq(\alpha)$ denote the solution of the dual problem \eqref{eq:dual}. 
The primal-dual optimality conditions are:
\begin{equation}
\label{eq:primale}
 \lambda_i(\alpha) = -\frac{\Psi_i^T\nabla^T\qq(\alpha)}{\alpha_i}
\end{equation}
and
\begin{equation}
\label{eq:duale}
 \qq(\alpha)=\frac{\nabla (\sum_{i=1}^m \psi_i \star \lambda_i(\alpha)-u_0)}{|\nabla(\sum_{i=1}^m \psi_i \star \lambda_i(\alpha) - u_0)|}.
\end{equation}
The last equality holds only formally since $\nabla(\sum_{i=1}^m \psi_i \star \lambda_i(\alpha) - u_0)$  may vanish at some locations. 
It means that $\qq$ represents the normal to the level curves of the denoised image $u_0-\sum_{i=1}^m\psi_i \star \lambda_i $.

Using \eqref{eq:primale}, we obtain $\psi_i\star \lambda_i(\alpha)= -\frac{1}{\alpha_i} \Psi_i \Psi_i^T \nabla^T \qq(\alpha)$.
Moreover, $\|\qq(\alpha)\|_\iinfty\leq 1$. 
It then suffices to use the norm operator definition \eqref{eq:opnorm} to obtain inequality \eqref{eq:estimatenorm2}.
\end{proof}

In order to use inequality \eqref{eq:estimatenorm2} for practical purposes, one needs to estimate upper bounds for $\|\cdot\|_{\iinfty, N}$. 
 Unfortunately, it is known to be a hard mathematical problem as shown in \cite{rohn2000computing,hendrickx2010matrix}.
The special case $N=2$, which corresponds to a Gaussian noise assumption, can be treated analytically:

\begin{proposition}
\label{prop:valueofnorm1D2}
Let $\hh_i= \begin{pmatrix} h_{i,1} \\ \vdots \\ h_{i,d}\end{pmatrix}$ with $h_{i,k}=\psi_i\star \tilde \psi_i \star \tilde d_k$. Then:
\begin{align}
\|\Psi_i \Psi_i^T \nabla^T\|_{\iinfty \rightarrow 2} &= \sqrt{n} \|\hat \hh_i\|_\infty \nonumber\\
&=\sqrt{n} \max_{k\in \{1, \dots, d\}} \|\hat h_{i,k}\|_\infty. \label{eq:ineqconv1D2}
\end{align}
\end{proposition}
\begin{proof}
First remark that:
\begin{align*}
\|\Psi_i \Psi_i^T \nabla^T\|_{\iinfty,2}&= \max_{\|\qq\|_\iinfty\leq 1} \|\sum_{k=1}^d h_{i,k}\star q_k\|_2 \\
&\leq \max_{\|\qq\|_2\leq \sqrt{n}} \|\sum_{k=1}^d h_{i,k}\star q_k\|_2 \\
& \leq \sqrt{n} \max_{\sum_{k=1}^d\|\hat q_k\|_2^2 \leq 1} \|\sum_{k=1}^d \hat h_{i,k} \odot \hat q_k\|_2 \\
& = \sqrt{n}\|\hat \hh_i\|_\infty.
\end{align*}
In order to obtain the reverse inequality, let us define 
\begin{equation*}
\mathcal{Q}_k=\{\qq \in \R^{n\times d}, q_k \in \{f_1,\cdots, f_n \} \ \mathrm{ and } \ q_i=0, \ i\in \{1, \cdots, d\} \backslash \{k\}\}
\end{equation*} 
and the Fourier transform of this set which is
\begin{equation*}
\hat{\mathcal{Q}}_k=\{\hat \qq \in \C^{n\times d}, \hat q_k \in \{n e_1,\cdots, ne_n \} \  \mathrm{ and } \ \hat q_i=0,\  i\in \{ 1, \cdots, d \} \backslash \{k\} \}.
\end{equation*} 
Let us denote $\displaystyle \mathcal{Q}=\cup_{k=1}^d Q_k$ and $\displaystyle \hat{\mathcal{Q}}=\cup_{k=1}^d \hat Q_k$ . 
%Note that if $\qq \in \mathcal{Q}$, $\hat \qq \in \hat{\mathcal{Q}}$. 
Thus we obtain:
\begin{align*}
 \|\Psi_i \Psi_i^T \nabla^T\|_{\iinfty,2}&= \max_{\|\qq\|_\iinfty\leq 1} \|\sum_{k=1}^d h_{i,k}\star q_k\|_2 \\
 &\geq \max_{ \qq \in \mathcal{Q} } \|\sum_{k=1}^d h_{i,k}\star q_k\|_2 \\
 & = \max_{ \hat \qq \in \hat{\mathcal{Q}}} \frac{\|\sum_{k=1}^d \hat h_{i,k}\odot \hat q_k\|_2}{\sqrt{n}} \\
 & = \sqrt{n}\|\hat \hh_i\|_\infty
\end{align*}
which ends the proof. %Remark that $\displaystyle \argmax_{\|u\|_\infty\leq 1} \|h\star u\|_2 \in \{f_1,\cdots, f_n\}$.
\end{proof}

\subsection{Proof of Theorem \ref{prop:solutionpetitalpha}}

%The upper bound \eqref{eq:estimatenorm2} tends to $+\infty$ as $\alpha$ goes to $0$. 
%It is thus important to refine the bound for small values of $\alpha$.
%We obtain in the following an alternative upper bound that does not depend on $\alpha$.
Denote $\Psi=(\Psi_1,\Psi_2,\hdots, \Psi_m)\in R^{n\times nm}$ and assume that $\Psi^T\Psi$ has full rank.
This condition ensures the existence of $\llambda$ satisfying $\sum_{i=1}^m \lambda_i\star \psi_i=u_0-u_0^{\textrm{mean}}$,
where $u_0^{\textrm{mean}}$ denotes the mean of $u_0$. 

\begin{proposition}
\label{prop:caracterizelambda0}
Let $\llambda^0(\alpha)$ denote the solution of the following problem
\begin{equation}
\llambda^0(\alpha)=\argmmin{\sum_{i=1}^m \frac{\alpha_i}{2} \|\lambda_i\|_2^2}{\llambda \in\R^{n\times m} \\ \sum_{i=1}^m \lambda_i\star \psi_i=u_0-u_0^{{\rm mean}}}.
\end{equation}

Then the vector $\hat \llambda^0(\alpha)=(\hat \lambda_1^0,\hdots,\hat \lambda_m^0)$ is equal to:
\begin{equation}
\label{eq:deflambda0}
\hat \lambda_i^0(\xi)= 
\left\{
\begin{array}{ll}
0 & \textrm{if \ } \xi=0 \\
\frac{ \bar{\hat{\psi_i}}(\xi) \hat{u_0}(\xi)}{\alpha_i \sum_{j=1}^m \frac{|\hat \psi_j(\xi)|^2}{\alpha_j}} & \textrm{otherwise}.
\end{array}
\right.
\end{equation}
\end{proposition}
\begin{proof}
First notice that the full rank hypothesis on $\Psi^T\Psi$ is equivalent to assuming that $\forall \xi,\ \exists i\in\{1,\hdots,m\},\ \hat\psi_i(\xi) \neq 0$ since $\Psi_i=\FF^{-1}\diag(\hat \psi_i)\FF$. Then:
\begin{align*}
\llambda^0(\alpha)&=\argmmin{\sum_{i=1}^m \frac{\alpha_i}{2} \|\lambda_i\|_2^2}{\llambda \in\R^{n\times m} \\ \sum_{i=1}^m \lambda_i\star \psi_i=u_0-u_0^{\textrm{mean}}} \\
&=\argmmin{\sum_{i=1}^m \frac{\alpha_i}{2} \|\hat \lambda_i\|_2^2}{\llambda \in\R^{n\times m} \\ \sum_{i=1}^m \hat \lambda_i\odot \hat \psi_i= \widehat{u_0-u_0^{\textrm{mean}}}}.
\end{align*}
This problem can be decomposed as $n$ independent optimization problems of size $m$.
If  $\xi= 0$, it remains to observe that $\widehat{u_0-u_0^{\textrm{mean}}}(0)=0$ since $u_0-u_0^{\textrm{mean}}$ has zero mean.
For $\xi\neq 0$, this amounts to solve the $m$ dimensional quadratic problem:
\begin{equation}
\argmin_{\hat \llambda(\xi) \in \C^m} \sum_{i=1}^m\frac{\alpha_i}{2} |\hat \lambda_i(\xi)|_2^2 \qquad \textrm{such that} \quad \sum_{i=1}^m \hat \psi_i(\xi) \hat \lambda_i(\xi)=\hat u_0(\xi).
\end{equation}
It is straightforward to derive the solution \eqref{eq:deflambda0} analytically.
\end{proof}

\begin{lemma}\label{lem:maj}
If $\hat{\psi_i}(\xi)=0$ then
$\hat{\lambda_i^0}(\alpha)(\xi)=0$
and if $\hat{\psi_i}(\xi)\neq 0$ then
$|\hat{\lambda_i^0}(\alpha)(\xi)|\le \left|\dfrac{\hat{u_0}(\xi)}{\hat{\psi_i}(\xi)}\right|$.
Therefore, for every $\alpha$, $\|\lambda_i^0(\alpha)\|_2 \le \|\hat{u_0}\oslash\hat{\psi_i}\|_2$ (with the convention to replace by 0 the terms where the denominator vanishes).
\end{lemma}
\begin{proof}
 It is a direct consequence of Equation \eqref{eq:deflambda0}.
\end{proof}

\begin{proof} of Theorem \ref{prop:solutionpetitalpha}
Let $F_\alpha(\llambda)=G(\llambda)+\sum_{i=1}^m\frac{\alpha_i}{2}\|\lambda_i\|_2^2$ with $G(\llambda)=\|(\nabla (\Psi\llambda - u_0) \|_{\boldsymbol 1}$. 
The objective is to prove that $\partial F_\alpha(\llambda^0(\alpha))\ni 0$ for sufficiently small $\alpha$.
Denote $C=\{\beta \one_{\R^n}, \beta \in \R\}$ the space of constant images.
Since $\Ker(\nabla)=C$ and $\Psi\llambda^0(\alpha)-u_0\in C$, $\nabla (\Psi\llambda^0 - u_0)=0$.
Standard results of convex analysis lead to
\begin{align*}
\partial G(\llambda^0(\alpha))&=\Psi^T\nabla^T \partial_{\|\cdot\|_{\boldsymbol{1}}}(0) \\
&= \Psi^T\nabla^T Q
\end{align*}
where $Q=\{\qq\in \R^{n\times d}, \|\qq(\xx)\|_2\leq 1, \ \forall \xx\in \Omega\}$ is the unit ball associated to the dual norm $\|\cdot\|_{\boldsymbol{1}}^*$.
Since $\Ran(\nabla^T)=\Ker(\nabla)^\perp$ we deduce $\Ran (\nabla^T) = C^{\perp}$ is the set of images with zero mean. Therefore, since $Q$ has non-empty interior, there exists $\gamma >0$ such that $\nabla^T Q\supset B(0,\gamma)\cap C^{\perp}$ where $B(0,\gamma)$ denotes a Euclidean ball of radius $\gamma$.
Therefore
\begin{align*}
(\partial F_\alpha(\llambda^0(\alpha)))_i &= (\partial G(\llambda^0(\alpha)))_i + \alpha_i \lambda^0_i(\alpha)\\
&= (\Psi^T \nabla^T Q)_i + \alpha_i \lambda^0_i(\alpha) \\
&\supset (\Psi^T (B(0,\gamma)\cap C^{\perp}))_i + \alpha_i \lambda^0_i(\alpha).
\end{align*}
Note that
\begin{equation*}
\Psi^T (B(0,\gamma)\cap C^{\perp})= \Psi_1^T (B(0,\gamma)\cap C^{\perp}) \times \hdots \times \Psi_m^T (B(0,\gamma)\cap C^{\perp}).
\end{equation*}
%The set $B(0,\gamma)\cap C^{\perp}$ is the space of vectors with zero mean and Euclidean norm bounded by $\gamma$. 
Since convolution operators preserve $C^\perp$  we obtain:
\begin{equation*}
\Psi_i^T (B(0,\gamma)\cap C^{\perp})\supset B(0,\gamma_i) \cap \Ran(\Psi_i) \cap C^\perp \quad \textrm{for some} \quad \gamma_i>0.
\end{equation*}
Now, it remains to remark that proposition \ref{prop:caracterizelambda0} ensures
\begin{equation*}
\llambda^0(\alpha)\in (\Ran(\Psi_1) \cap C^\perp) \times \hdots \times (\Ran(\Psi_m) \cap C^\perp).
\end{equation*} 
Therefore for $\alpha_i\leq \frac{\gamma_i}{\|\hat{u_0}\oslash\hat{\psi_i}\|_2}$ 
\begin{align*}
(\partial F_\alpha(\llambda^0))_i \supset (B(0,\gamma_i) \cap \Ran(\Psi_i) \cap C^\perp) + \alpha_i \lambda_i^0(\alpha) \ni 0.
\end{align*}
In view of Lemma \ref{lem:maj} it suffices to set $\displaystyle \bar \alpha = \min_{i\in \{1,\hdots, m\}}\frac{\gamma_i}{\|\hat{u_0}\oslash\hat{\psi_i}\|_2}$ to conclude the proof.% of the first part of Theorem \ref{prop:solutionpetitalpha}. For the second part of Theorem \ref{prop:solutionpetitalpha} (*j'ai tout a coup un doute - je vois avec 1 filtre mais pas avec $m$*)
\end{proof}

%\begin{lemma} In the case $m=1$, the function $\alpha\mapsto\|\lambda(\alpha)\|_2$ is decreasing with respect to $\alpha$.
%\end{lemma}
%(*je le mets ici pour l'instant, on pourra changer de place*)
%\begin{proof}
%Let $\alpha_1<\alpha_2$ and denote $\lambda_1=\lambda(\alpha_1), \lambda_2=\lambda(\alpha_2)$. With the same notations as above, for $k=1,2$ we have
%$$\lambda_k=\Argmin F_{\alpha_k}(\lambda)=\Argmin G(\lambda)+ \dfrac{\alpha_k}{2}||\lambda||^2.$$
%The optimality conditions read
%$$\alpha_k\lambda_k\in -\partial G(\lambda_k),$$
%and in particular $\alpha_1(\lambda_2-\lambda_1)+(\alpha_2-\alpha_1)\lambda_2\in\partial G(\lambda_1)-\partial G(\lambda_2)$. Since $G$ is convex, $\partial G$ is a monotone operator, hence
%$$(\alpha_1(\lambda_2-\lambda_1)+(\alpha_2-\alpha_1)\lambda_2).(\lambda_1-\lambda_2)\ge0.$$
%This is equivalent to 
%$$\alpha_1(\lambda_2-\lambda_1).(\lambda_1-\lambda_2)+(\alpha_2-\alpha_1)\lambda_2.(\lambda_1-\lambda_2)\ge0.$$
%Since $\alpha_2-\alpha_1>0$ this implies that $\lambda_2.(\lambda_1-\lambda_2)\ge0$, hence $||\lambda_1||\ge||\lambda_2||$.
%\end{proof}

\subsection{Proof of proposition \ref{eq:propminor}}
%\subsection{Evaluation of $\|\cdot\|_{\iinfty, N}$ operator norms}

%\subsection{How sharp is the upper-bound ?}

We now concentrate on problem \eqref{eq:prob1filtre} in the case of $m=1$ filter and provide a lower-bound on $\|b(\alpha)\|_2$. We assume that $\Psi$ in invertible, meaning that $\hat \psi$ does not vanish. %This hypothesis is satisfied for the practical denoising applications we consider in this paper.

%[J'\'ecris en vrac pour le moment]

The dual problem of 
\begin{equation}
\min_{\lambda\in \R^{n}} \|\nabla (u_0- \lambda \star \psi) \|_1 + \frac{\alpha}{2} \|\lambda\|_2^2
\label{eq:pp}
\end{equation}
is 
\begin{equation}
\max_{\qq\in \R^{n\times d}, \|\qq\|_\iinfty\leq 1} \langle \nabla u_0,\qq\rangle - \frac{1}{2\alpha} \|\Psi^T\nabla^T \qq\|_2^2.
\label{eq:pd}
\end{equation}
The solution $\lambda(\alpha)$ of \eqref{eq:pp} can be deduced from the solution $\qq(\alpha)$ of \eqref{eq:pd} by using the primal-dual relationship 
\begin{equation}
\lambda(\alpha)= -\frac{1}{\alpha}\Psi^T\nabla^T \qq(\alpha).
\end{equation}

%Our aim is to bound $\|\lambda(\alpha)\|_2$ and $\|b(\alpha)\|_2$ from below. 
%
%\begin{proposition}
%Assume that $\hat \psi$ does not vanish.
%Let $\lambda(\alpha)$ denote the solution of \eqref{eq:pp} and $b(\alpha)=\psi\star \lambda(\alpha)$.
%Let $P_1$ denote the orthogonal projector on $\Ran(\Psi^T\nabla^T)$ and $b_1=P_1(\Psi^{-1}u_0)$.
%Then 
%\begin{equation*}
%\|\lambda(\alpha)\|_2\geq \frac{1}{\alpha}\frac{\|b_1\|_2}{\|A^+b_1\|_\infty}
%\end{equation*}
%and 
%\begin{equation*}
%\|b(\alpha)\|_2\geq \frac{1}{\alpha}\frac{\|\Psi b_1\|_2}{\|A^+b_1\|_\infty}  [CHECK].
%\end{equation*}
%Note that the quantities can be evaluated numerically and provide an idea of the sharpness of the upper bound.
%\end{proposition}
%
%\begin{proof}
%To achieve this goal w
We can write:
\begin{align}
&\argmax_{\qq\in \R^{n\times d}, \|\qq\|_\iinfty\leq 1} \langle \nabla u_0,\qq\rangle - \frac{1}{2\alpha} \|\Psi^T\nabla^T \qq\|_2^2 \\
&= \argmin_{\qq\in \R^{n\times d}, \|\qq\|_\iinfty\leq 1} \frac{1}{2}\|\Psi^T\nabla^T \qq - \alpha \Psi^{-1}u_0\|_2^2. \label{eq:pdd}
\end{align}

Let $P_1$ denote the orthogonal projector on $\Ran(\Psi^T\nabla^T)$ and $P_2$ denote the orthogonal projector on $\Ran(\Psi^T\nabla^T)^\perp$.
Using these operators, we can write $\alpha \Psi^{-1}u_0= \alpha b_1 + \alpha b_2$ where $b_1=P_1\Psi^{-1}u_0$ and $b_2=P_2\Psi^{-1}u_0$. Problem \eqref{eq:pdd} becomes:
\begin{equation*}
\qq(\alpha) = \argmin_{\qq\in \R^{n\times d}, \|\qq\|_\iinfty\leq 1} \frac{1}{2}\|\Psi^T\nabla^T \qq - \alpha b_1\|_2^2.
\end{equation*}

Let us denote $A=\Psi^T\nabla^T$ and $\qq'(\alpha)=\frac{A^+b_1}{\|A^+b_1\|_\iinfty}$. Since $b_1\in \Ran(A)$, $A\qq'(\alpha)=\frac{b_1}{\|A^+b_1\|_\iinfty}$.
Thus as long as $\|A^+ b_1\|_\iinfty\geq \frac{1}{\alpha}$:
\begin{align*}
| \|A\qq(\alpha)\| - \alpha \|b_1\|_2| &\leq \|A\qq(\alpha) - \alpha b_1\| \\ 
&= \min_{\|\qq\|_\iinfty \leq 1} \|A\qq - \alpha b_1\|_2 \\
&\leq \|A\qq'(\alpha) - \alpha b_1\|_2 \\
&= \|\frac{b_1}{\|A^+ b_1\|_\iinfty} - \alpha b_1\|_2 \\
&= \left(\alpha - \frac{1}{\|A^+b_1\|_\iinfty}\right) \|b_1\|_2.
\end{align*}

Since $A\qq(\alpha)$ is a projection of $\alpha b_1$ on a convex set that contains the origin, $\alpha \|b_1\|_2\geq \|A\qq(\alpha)\|_2$ and we get:
\begin{equation*}
\alpha\|b_1\|_2 - \|A\qq(\alpha)\|_2 \leq \left(\alpha - \frac{1}{\|A^+b_1\|_\iinfty}\right) \|b_1\|_2
\end{equation*}
which is equivalent to 
\begin{equation*}
\|A\qq(\alpha)\|_2\geq \frac{b_1}{\|A^+b_1\|_\infty}.
\end{equation*}

Since $\lambda(\alpha)= -\frac{1}{\alpha}A \qq(\alpha)$ we get:
\begin{equation*}
\|\lambda(\alpha)\|_2\geq \frac{1}{\alpha}\frac{\|b_1\|_2}{\|A^+b_1\|_\infty}
\end{equation*}
and since $\lambda=\Psi^{-1}b$ we obtain
\begin{equation*}
||b(\alpha)||_2\ge\frac{1}{\alpha}\|\Psi^{-1}\|_{2\rightarrow 2}^{-1}\frac{\|b_1\|_2}{\|A^+b_1\|_\infty}.
\end{equation*}
%\end{proof}

\subsection{Proof of Theorem \ref{eq:corollambda}}

Theorem \ref{eq:corollambda} is a simple consequence of a more general result described below.

Let $F$ be a convex lower semi-continuous (l.s.c.) function and $\Psi~:~\R^n~\to~\R^n$, $\Psi_i:\R^n\to\R^n,i=1\ldots m$ denote linear operators.
Define:
\begin{equation}\label{eq:P1}\tag{$P_1$}
(\overline{\lambda_i})_{1\le i\le m}=\argmin_{(\lambda_i)_{1\le i\le m}\in(\R^n)^m}F\left(\sum_{i=1}^m\Psi_i\lambda_i\right)+\dfrac{1}{2}\sum_{i=1}^m\|\lambda_i\|^2,
\end{equation}
and
\begin{equation}\label{eq:P2}\tag{$P_2$}
\overline{\lambda}=\argmin_{\lambda\in\R^n}F(\Psi\lambda)+\dfrac{1}{2}\|\lambda\|^2.
\end{equation}
\begin{proposition}
\label{prop:mdonne1}
If the operators $\Psi$ and $(\Psi_i)_{i=1\ldots m}$ satisfy the relation
$$\Psi\Psi^*=\sum_{i=1}^m\Psi_i\Psi_i^*,$$
then the solutions $(\overline{\lambda_i})_{1\le i\le m}$ of \eqref{eq:P1} and $\overline{\lambda}$ of \eqref{eq:P2} are related by
$$\Psi\overline{\lambda}=\sum_{i=1}^m\Psi_i\overline{\lambda_i}.$$
\end{proposition}
\begin{proof}
We define ${\bf \Psi}:\R^{n\times m} \to\R^n$ by ${\bf \Psi}=(\Psi_1, \Psi_2,\ldots,\Psi_m)$, so that for $\llambda~=~\left(\begin{array}{c}\lambda_1\\ \lambda_2\\\cdots\\\lambda_m \end{array}\right)$, ${\bf \Psi}\llambda=\sum_{i=1}^m\Psi_i\lambda_i$.
The optimality condition of \ref{eq:P1} reads:
$$\PPsi^*\partial F(\PPsi \bar \llambda)+\sum_{i=1}^m \bar \lambda_i\ni 0,$$
and the optimality condition of \ref{eq:P2} reads
$$\Psi^*\partial F(\Psi \bar \lambda)+\bar \lambda\ni 0.$$
The minization problem \ref{eq:P2} admits a unique minimizer denoted $\overline{\lambda}$. By hypothesis $\PPsi\PPsi^*=\Psi\Psi^*$, hence $\Ran \PPsi=\Ran \Psi$ and there exists $\llambda_0$ such that
$$\Psi\lambdabar=\PPsi\llambda_0.$$
The optimality condition of \ref{eq:P2} implies that
$$0\in \Psi\Psi^*\partial F(\Psi \lambdabar)+\Psi\lambdabar,$$
hence
$$0\in \Psi\Psi^*\partial F(\PPsi \llambda_0)+\PPsi\llambda_0=\PPsi(\PPsi^*\partial F(\PPsi\llambda_0)+\llambda_0).$$
This proves that every vector $\llambda'\in\PPsi^*\partial F(\PPsi\llambda_0)+\llambda_0$ belongs to $\Ker \PPsi$. If we choose such a $\llambda'$ and set $\llambda=\llambda_0-\llambda'$ we have
$$\PPsi^*\partial F(\PPsi\llambda)+\llambda=\PPsi^*\partial F(\PPsi\llambda_0)+\llambda_0-\llambda'\ni0.$$
This implies that $\llambda$ is the minimizer of \ref{eq:P1} and we have
$$\PPsi\llambda=\PPsi(\llambda_0-\llambda')=\PPsi\llambda_0=\Psi\lambdabar,$$
which ends the proof.
\end{proof}

Let us now turn to the proof of Theorem \ref{eq:corollambda}.
\begin{proof}
To obtain \eqref{eq:equalitynoise}, it suffices to make the change of variable $\lambda_i'=\frac{\lambda_i}{\sqrt{\alpha_i}}$ in problem \eqref{eq:P1} and to apply proposition \eqref{prop:mdonne1} together with condition \eqref{eq:consistency}.
To obtain \eqref{eq:lambdaifromlambda}, it remains to observe that since $\sum_{k=1}^m b_i(\alpha) = b(\alpha)$, the determination of $\lambda_i$ boils down to the following quadratic problem:
\begin{align*}
 (\lambda_i(\alpha))_{i\in \{1,\hdots, m\}} & = \argmin_{\sum_{i=1}^m \lambda_i\star \psi_i = b(\alpha) } \sum_{i=1}^m \frac{\alpha_i}{2}\|\lambda_i\|_2^2 \\
& = \argmin_{\sum_{i=1}^m \hat \lambda_i \odot \hat \psi_i = \hat b(\alpha) } \sum_{i=1}^m \frac{\alpha_i}{2}\|\hat \lambda_i\|_2^2.
\end{align*}
The solution of this problem can be obtained analytically by deriving its optimality conditions. It leads to equation \eqref{eq:lambdaifromlambda}.
\end{proof}